\documentclass[11pt]{article}

\usepackage{xr-hyper}
\usepackage{amsfonts,amsmath,amssymb,amsthm,graphicx}
\usepackage{changepage}
\usepackage{enumerate}
\usepackage{scalerel}
\usepackage{accents}
\usepackage[margin=1in]{geometry}
\usepackage{float}
\usepackage[caption = false]{subfig}
\usepackage{enumitem}
\usepackage{natbib}
\usepackage[utf8]{inputenc} \usepackage[T1]{fontenc} 
\usepackage{csquotes}
\usepackage{comment}
\usepackage{bbm}
\usepackage{lmodern}
\usepackage{array}
\usepackage{csquotes}

\usepackage{tcolorbox}

\tcbuselibrary{listings,breakable,skins}
\newtcolorbox{promptbox}[1]{
  enhanced,
  breakable,
  colback=white,
  colframe=black,
  boxrule=0.8pt,
  arc=1.5pt,
  left=8pt,right=8pt,top=8pt,bottom=8pt,
  fonttitle=\bfseries,
  title=#1
}

\usepackage[colorlinks=true, linkcolor= blue!70!black, citecolor=blue!70!black,urlcolor=black,breaklinks=true]{hyperref}

\usepackage{thm-restate}

\usepackage{cleveref}
\crefname{claim}{claim}{claims}

\usepackage{tikz}
\usepackage{pgfplots}
\pgfplotsset{compat=1.18}
\usetikzlibrary{patterns}
\usepgfplotslibrary{fillbetween}
\usetikzlibrary{intersections}
\usetikzlibrary{arrows.meta,positioning}
\usepackage{fix-cm}
\usepackage[ruled,vlined,linesnumbered]{algorithm2e}
\Crefname{algocf}{Algorithm}{Algorithms}
\allowdisplaybreaks
\usetikzlibrary{arrows, decorations.markings}
\tikzstyle{vecArrow} = [thick, decoration={markings,mark=at position
   1 with {\arrow[semithick]{open triangle 60}}},
   double distance=1.4pt, shorten >= 5.5pt,
   preaction = {decorate},
   postaction = {draw,line width=1.4pt, white,shorten >= 4.5pt}]
\tikzstyle{innerWhite} = [semithick, white,line width=1.4pt, shorten >= 4.5pt]

\allowdisplaybreaks
\theoremstyle{plain}
\newtheorem{theorem}{Theorem}[section]
\newtheorem{lemma}[theorem]{Lemma}

\newtheorem{fact}[theorem]{Fact}
\newtheorem{proposition}[theorem]{Proposition}
\newtheorem{corollary}[theorem]{Corollary}

\theoremstyle{plain}
\newtheorem{definition}{Definition}[section] \newtheorem{example}[definition]{Example}

\theoremstyle{plain}

\usepackage{xfrac}

\newcommand{\squishlist}{
\begin{list}{{$\bullet$}}
{\setlength{\itemsep}{3pt}      \setlength{\parsep}{1pt}
\setlength{\topsep}{1pt}       \setlength{\partopsep}{0pt}
\setlength{\leftmargin}{1em} \setlength{\labelwidth}{1em}
\setlength{\labelsep}{0.5em} } }
\newcommand{\squishend}{  \end{list}}

\newcommand{\xhdr}[1]{\vspace{2mm}\noindent{\bf {#1.}\ }}

\newcommand{\Payoff}[2][]{\text{\bf Payoff}\ifthenelse{\not\equal{}{#1}}{_{#1}}{}\!\left[{\def\givenn{\middle|}#2}\right]}

\newcommand{\Reg}[2][]{\text{\bf Reg}\ifthenelse{\not\equal{}{#1}}{_{#1}}{}\!\left[{\def\givenn{\middle|}#2}\right]}

\newcommand{\action}{a}
\newcommand{\actionSpace}{\mathcal{A}}

\newcommand{\outcomeSpace}{\mathcal{Y}}

\newcommand{\agentU}{u}
\newcommand{\tildeagentU}{\tilde{\agentU}}
\newcommand{\predictor}{\mu}
\newcommand{\predictornew}{\nu}

\newcommand{\truePredic}{\kappa}

\newcommand{\prediction}{p}

\newcommand{\bestr}{a^*}

\newcommand{\NI}{\textsc{NI}}

\newcommand{\indirectU}{U}
\newcommand{\agentUClass}{\mathcal{U}}
\newcommand{\indirectUDeriv}{\delta_\indirectU}
\newcommand{\ECE}[2][]{\textsc{ECE}\ifthenelse{\not\equal{}{#1}}{_{#1}}{}\!\left[{\def\givenn{\middle|}#2}\right]}

\newcommand{\dCE}[2][]{\text{\cc{ dCE}}\ifthenelse{\not\equal{}{#1}}{_{#1}}{}\!\left[{\def\givenn{\middle|}#2}\right]}
\newcommand{\ddCE}[2][]{\text{\cc{ ddCE}}\ifthenelse{\not\equal{}{#1}}{_{#1}}{}\!\left[{\def\givenn{\middle|}#2}\right]}
\newcommand{\smCE}[2][]{\text{\cc{ smCE}}\ifthenelse{\not\equal{}{#1}}{_{#1}}{}\!\left[{\def\givenn{\middle|}#2}\right]}

\newcommand{\OBJ}[2][]{\text{\bf OBJ}\ifthenelse{\not\equal{}{#1}}{_{#1}}{}\!\left[{\def\givenn{\middle|}#2}\right]}

\newcommand{\OPT}[2][]{\text{\bf OPT}\ifthenelse{\not\equal{}{#1}}{_{#1}}{}\!\left[{\def\givenn{\middle|}#2}\right]}

\newcommand{\Primal}[2][]{\cc{Primal}\ifthenelse{\not\equal{}{#1}}{_{#1}}{}\!\left[{\def\givenn{\middle|}#2}\right]}

\newcommand{\earthDist}[2][]{\text{\cc{ EMD}}\ifthenelse{\not\equal{}{#1}}{_{#1}}{}\!\left[{\def\givenn{\middle|}#2}\right]}

\newcommand{\REMD}[2][]{\textsc{REMD}\ifthenelse{\not\equal{}{#1}}{_{#1}}{}\!\left[{\def\givenn{\middle|}#2}\right]}
\newcommand{\earthDistFlow}[2][]{\textsc{REMD$^{\textsc{MisC}}$}\ifthenelse{\not\equal{}{#1}}{_{#1}}{}\!\left[{\def\givenn{\middle|}#2}\right]}
\newcommand{\REMDNew}[2][]{\ensuremath{\overline{\mathrm{REMD}}^{\mathrm{MisC}}\ifthenelse{\equal{#1}{}}{}{_{#1}}\!\left[\def\givenn{\middle|}#2\right]}}
\newcommand{\empricaldFlow}[2][]{\ensuremath{\widehat{\mathrm{REMD}}^{\mathrm{MisC}}\ifthenelse{\equal{#1}{}}{}{_{#1}}\!\left[\def\givenn{\middle|}#2\right]}}

\newcommand{\empricaldInforGap}[2][]{\ensuremath{\widehat{\textsc{InfoGap}}\ifthenelse{\equal{#1}{}}{}{_{#1}}\!\left[\def\givenn{\middle|}#2\right]}}

\newcommand{\Marg}{\textsc{REMD}}

\newcommand{\Bayes}{\textsc{True}}

\newcommand{\upperdFlowC}[2][]{\overline{\textsc{REMD}}^{\mathrm{MisC}}\ifthenelse{\not\equal{}{#1}}{_{#1}}{}\!\left[{\def\givenn{\middle|}#2}\right]}
\newcommand{\feasibleflowSet}{\overline{\Pi}}

\newcommand{\EMD}[2][]{\textsc{EMD}\ifthenelse{\not\equal{}{#1}}{_{#1}}{}\!\left[{\def\givenn{\middle|}#2}\right]}
\newcommand{\CDL}[2][]{\textsc{CDL}\ifthenelse{\not\equal{}{#1}}{_{#1}}{}\!\left[{\def\givenn{\middle|}#2}\right]}
\newcommand{\UCal}[2][]{\textsc{UCal}\ifthenelse{\not\equal{}{#1}}{_{#1}}{}\!\left[{\def\givenn{\middle|}#2}\right]}

\newcommand{\inforGap}[2][]{\textsc{InfoGap}\ifthenelse{\not\equal{}{#1}}{_{#1}}{}\!\left[{\def\givenn{\middle|}#2}\right]}

\newcommand{\agentExpPay}[2][]{\textsc{Payoff}\ifthenelse{\not\equal{}{#1}}{_{#1}}{}\!\left[{\def\givenn{\middle|}#2}\right]}

\newcommand{\Dist}[2][]{\textsc{InfoMeasure}\ifthenelse{\not\equal{}{#1}}{_{#1}}{}\!\left[{\def\givenn{\middle|}#2}\right]}

\newcommand{\outcome}{y}

\newcommand{\reals}{\mathbb{R}}

\newcommand{\Bern}{\cc{Bern}}
\newcommand{\context}{x}

\newcommand{\CDF}{F_\predictor}
\newcommand{\CDFNew}{F_\predictornew}
\newcommand{\PDF}{f_\predictor}
\newcommand{\PDFNew}{f_\predictornew}

\newcommand{\SCDF}{S_\predictor}
\newcommand{\SCDFNew}{S_\predictornew}

\newcommand{\scdf}{S}

\newcommand{\join}{f_{\text{join}}}
\newcommand{\miscali}{\pi}
\newcommand{\couplingSpace}{\Pi}
\newcommand{\predictionNew}{q}

\newcommand{\dist}{f_1}
\newcommand{\distNew}{f_2}
\newcommand{\distSCDF}{S_1}
\newcommand{\distSCDFNew}{S_2}
\newcommand{\distCDF}{F_1}
\newcommand{\distCDFNew}{F_2}
\newcommand{\joinSCDF}{S_{\text{join}}}

\newcommand{\auxfunc}{V}

\newcommand{\abs}[1]{\left| #1 \right|}

\newcommand{\ProGOne}[1]{\text{\hyperref[eq:primal perfect]{$\textsc{P}_1[#1]$}}}

\newcommand{\ProGTwo}[1]{\text{\hyperref[eq:primal perfect intermiddle]{$\textsc{P}_2[#1]$}}}

\newcommand{\ProGOneMisC}[1]{\text{\hyperref[eq:primal miscali]{$\textsc{P}^\textsc{MisC}_1[#1]$}}}

\newcommand{\Dirac}{\textsc{Dirac}}
\newcommand{\completeConstOne}{c}

\newcommand{\soundConstOne}{s}

\newcommand{\newday}{t}
\newcommand{\absoIncr}{\eta}
\newcommand{\relaIncr}{\tau}

\newcommand{\frequency}{f}
\newcommand{\brierScore}[2][]{\textsc{BS}\ifthenelse{\not\equal{}{#1}}{_{#1}}{}\!\left[{\def\givenn{\middle|}#2}\right]}

\newcommand{\ChatGPT}{\sf ChatGPT}
\newcommand{\DeepSeek}{\sf DeepSeek}
\newcommand{\Gemini}{\sf Gemini}
\newcommand{\Qwen}{{\sf Qwen}}

\newcommand{\record}{x}

\newcommand{\IGtol}{s}
\newcommand{\infogapTolerance}{IG-tol}
\newcommand{\generalizedSCDF}{CA-SCDF}
\newcommand{\ProGeneralizedSCDF}{Calibration-adjusted SCDF}

\newcommand{\postProPred}{\hat{p}}
\newcommand{\InfoGap}{\textsc{InfoGap}}

\newcommand{\temperature}{T}
\newcommand{\plattCoef}{A}
\newcommand{\plattInterc}{B}
\newcommand{\logit}{z}
\newcommand{\binsNum}{B}
\newcommand{\scalingFunc}{g}
\newcommand{\scalingFuncFamily}{\mathcal{G}}
\newcommand{\hyperParam}{\alpha}
\newcommand{\calibData}{\mathcal{D}}

\newcommand{\bin}{b}
\newcommand{\binningModel}{M}
\newcommand{\binomPara}{\theta}
\newcommand{\predictPara}{\theta}
\newcommand{\sigmoid}[1]{\sigma({#1})}
\newcommand{\softmax}{\mathrm{softmax}}

\newcommand{\nll}[2][]{\textsc{NLL}\ifthenelse{\not\equal{}{#1}}{_{#1}}{}\!\left[{\def\givenn{\middle|}#2}\right]}

\newcommand{\postParam}{\theta}
\newcommand{\proPostParam}{\Theta}
\newcommand{\postParamEstim}[1]{\hat{\theta}\!\left(#1\right)}

	\DeclareMathOperator{\argmax}{argmax}
	\DeclareMathOperator{\argmin}{argmin}

\newcommand{\condition}{\,\mid\,}

\newcommand{\givenn}{\mid}
\newcommand{\prob}[2][]{\text{Pr}\ifthenelse{\not\equal{}{#1}}{_{#1}}{}\!\left[#2\right]}
\newcommand{\expect}[2][]{\mathbb{E}\ifthenelse{\not\equal{}{#1}}{_{#1}}{}\!\left[{\def\givenn{\middle|}#2}\right]}

\newcommand{\tparen}{\big}
\newcommand{\tprob}[2][]{\text{Pr}\ifthenelse{\not\equal{}{#1}}{_{#1}}{}\tparen[{\def\given{\tparen|}#2}\tparen]}
\newcommand{\texpect}[2][]{\mathbb{E}\ifthenelse{\not\equal{}{#1}}{_{#1}}{}\tparen[{\def\given{\tparen|}#2}\tparen]}

\newcommand{\sprob}[2][]{\text{Pr}\ifthenelse{\not\equal{}{#1}}{_{#1}}{}[#2]}
\newcommand{\sexpect}[2][]{\mathbb{E}\ifthenelse{\not\equal{}{#1}}{_{#1}}{}[#2]}

\newcommand{\dd}{{\mathrm d}}

\newcommand{\indicator}[1]{{\mathbbm{1}\left\{ #1 \right\}}}

\newcommand{\plus}[1]{{\left( #1 \right)^+}}

\newcommand{\cc}[1]{\ensuremath{\mathsf{#1}}}

\newcommand{\primed}{^{\dagger}}

\newcommand{\supp}{\cc{supp}}

\title{
Is This Predictor More Informative than Another? \\ A Decision-Theoretical Comparison}

\author{Yiding Feng\thanks{Hong Kong University of Science and Technology, Email: {\tt ydfeng@ust.hk}}  
\quad 
Liuhan Qian\thanks{Hong Kong University of Science and Technology, Email: {\tt lqianab@connect.ust.hk}}
\quad 
Wei Tang\thanks{The Chinese University of Hong Kong, Email: {\tt wtang2359@gmail.com}}
}
\date{}

\begin{document}

\maketitle
\begin{abstract}

In many real-world applications, a model provider provides probabilistic forecasts to downstream decision-makers who use them to make decisions under diverse payoff objectives. 
The provider may have access to multiple predictive models, each potentially miscalibrated, and must choose which model to deploy in order to maximize the usefulness of predictions for downstream decisions. A central challenge arises: how can the provider meaningfully compare two predictors when neither is guaranteed to be well-calibrated, and when the relevant decision tasks may differ across users and contexts?

To answer this, our first contribution introduces the notion of the \emph{informativeness gap} between any two predictors, defined as the maximum normalized payoff advantage one predictor offers over the other across all decision-making tasks. Our framework strictly generalizes several existing notions: it subsumes U-Calibration \citep{KLST-23} and Calibration Decision Loss \citep{HW-24}, which compare a miscalibrated predictor to its calibrated counterpart, and it recovers Blackwell informativeness \citep{B-51,B-53} as a special case when both predictors are perfectly calibrated. Our second contribution is a dual characterization of the informativeness gap, which gives rise to a natural informativeness measure that can be viewed as a relaxed variant of the earth mover's distance (EMD) between two prediction distributions. We show that this measure satisfies natural desiderata: it is complete and sound, and it can be estimated sample-efficiently in the prediction-only access setting. Along the way, we also obtain novel combinatorial structural results when applying this measure to perfectly calibrated predictors.
We complement our theory with experiments on LLM-based forecasters in real-world weather and Bitcoin prediction tasks, showing that the informativeness gap offers a more decision-relevant alternative to traditional metrics like the Brier score/ECE, and provides a principled lens for evaluating how ad hoc calibration post-processing affects downstream decision usefulness.

 \end{abstract}

\newpage

\section{Introduction}
\label{sec:intro}
\newcolumntype{C}[1]{>{\centering\arraybackslash}m{#1}}
\newcommand{\poly}{\text{poly}}
\newenvironment{mythm}[1][Theorem]{\par\medskip \noindent \textbf{#1.} \itshape }{\par\medskip }

\newcommand{\inforGapMea}{\text{infoGap}}

Over the last decade, the machine learning predictors have grown remarkably powerful, especially with the rapid advancements in large-scale models such as large language models (LLMs).
These predictors have demonstrated strong performance across a wide range of domains, providing high-quality predictions that are increasingly used by downstream decision-makers to inform their decisions.
However, there are usually  two key challenges that often hinder decision-makers from fully leveraging these predictions: 
(1) 
the underlying mechanisms used to generate predictions are frequently proprietary and opaque to external users;
(2) due to limitations in training data or computational constraints in the training process, these predictors may exhibit biases and fail to accurately reflect the empirical frequencies of outcomes.
To mitigate these challenges, one natural solution is to ensure that the predictions are calibrated.

A calibrated predictor regulates
that predicted probabilities align with the true (conditional) probability of the outcome \citep{D-82,FV-98}. For
example, predictions of ``80\% likelihood'' materialize approximately 80\% outcome realizations of the time. 
It is well established that agents who naively best respond to perfectly calibrated predictions incur no regret \citep{FV-98,FH-21}.
With this desired property, a variety of calibration error metrics have been proposed to quantify how much a predictor deviates from perfect calibration -- such as the Expected Calibration Error (ECE) \citep{FV-97}, the smooth calibration error  \citep{FH-18}, and the distance to calibration \citep{BGHN-23b}.
Remarkably, some ``decision-theoretic'' measures are proposed to directly quantify the decision-makers' regret when she best-responds to the predictor's forecasts.
\cite{KLST-23} introduced ``U-Calibration'' (UCal), a measure that captures the maximum payoff loss incurred by an agent who naively best responds to a miscalibrated predictor, compared to responding to a best calibrated predictor that provides a fixed prediction. 
Similarly, \cite{HW-24} proposed the ``Calibration Decision Loss'' (CDL) which quantifies the maximum payoff gap between naively responding to a miscalibrated predictor and responding to the true empirical frequencies of the outcomes associated with each prediction.

Both UCal and CDL therefore assess how far a  predictor falls short of its own calibrated counterpart -- they tell us ``how much better this predictor would be if it were perfectly calibrated.'' 
In practice, however, a model provider (e.g., a platform, API service, or forecasting vendor) may have access to multiple distinct predictors, each trained separately and potentially miscalibrated in different ways, and must decide which predictor to deploy to its downstream decision-makers. This deployment choice is often made without knowing the specific payoff objectives of every downstream user, and may need to serve many heterogeneous decision tasks simultaneously. 
Because neither predictor is necessarily the calibrated version of the other, existing metrics do not directly answer the questions like
``Which predictor should the provider deploy so as to maximize downstream decision-makers' payoff?'' ``When two candidate predictors are both miscalibrated, which one is more useful for downstream decisions?'' ``Is deploying a perfectly calibrated predictor always better than deploying a miscalibrated one?''
The answers to these questions are far from obvious. To illustrate this subtlety, we begin with a simple weather-forecasting example.
\begin{example}
\label{ex:comparisons}
Suppose the long‐run probability of rain is 50\%. We compare following predictors:
\begin{itemize}
    \item The predictor $\predictornew$ always forecasts a $50\%$ chance of rain.
    \item 
The predictor $\predictor_1$ forecasts four possible predictions $\{0\%, 49\%, 51\%, 100\%\}$, with a conditional prediction distribution constructed as
\begin{table}[H]
    \centering
\renewcommand{\arraystretch}{2.1}\begin{tabular}{l|C{1.5cm}|C{1.5cm}|C{1.5cm}|C{1.5cm}}
    \hline
        \rule{0pt}{1ex}
        & \bfseries 0\% 
        & \bfseries $\mathbf{49\%}$ 
        & \bfseries $\mathbf{51\%}$ 
        & \bfseries 100\% \\ \hline
    \bfseries When it rains             
        & $0$ 
        & $0.0051$ 
        & $0.0049$ 
        & 0.99 \\ \hline
    \bfseries When it does not rain          & 0.99 
        & $0.0049$ 
        & $0.0051$ 
        & $0$ \\ 
        \hline
    \end{tabular}
    \end{table}

    \item 
    The predictor $\predictor_2$ forecasts two possible predictions $\{1\%, 99\%\}$, with a conditional prediction distribution constructed as
    \begin{table}[H]
    \centering
\renewcommand{\arraystretch}{1.5}\begin{tabular}{l|C{1.5cm}|C{1.5cm}}
    \hline
        \rule{0pt}{2ex}
        & \bfseries $\mathbf{1\%}$ 
        & \bfseries $\mathbf{99\%}$ \\ \hline
    \rule{0pt}{2ex}
    \bfseries When it rains             
        & $0$ 
        & $1$ \\ \hline
    \rule{0pt}{2ex}
    \bfseries When it does not rain          
        & $1$ 
        & $0$ \\ 
        \hline
    \end{tabular}
    \end{table}
\end{itemize}
\end{example}
In this example, we can see that although the predictor $\predictornew$ is perfectly calibrated but it contains no information about the realized outcome beyond the base rate (empirical rainy frequency).
In contrast, although predictor $\predictor_1$ miscalibrates slightly on the two middle predictions -- it says $49\%$ (resp.\ $51\%$) but the true rain rate is $51\%$ (resp.\ $49\%$) -- but is otherwise nearly perfectly accurate: it predicts 100\% almost surely when it rains, and 0\% almost surely when it does not.
Because of this, any decision‐maker who acts on predictor $\predictor_1$'s predictions can make choices that more closely reflect the actual outcome, compared to relying solely on the base rate.
In fact, and perhaps unsurprisingly, one can show that for every decision problem -- no matter the payoff structure -- the expected payoff under miscalibrated predictor $\predictor_1$ is never lower (and often strictly higher) than under perfectly calibrated predictor $\predictornew$. 
This example demonstrates that a miscalibrated but more informative predictor can dominate a perfectly calibrated yet uninformative one.

However, this dominance is not guaranteed in general. If we slightly alter predictor $\predictor_1$ to be predictor $\predictor_2$ such that it miscalibrates on the two extreme predictions -- i.e., predicting $99\%$ when it actually rains, and predicting $1\%$ when it actually does not rain -- then its advantage no longer holds.\footnote{In fact, one can show that the ECEs of predictors $\predictor_1$ and $\predictor_2$ have the same magnitude.}
In this case, in some decision problems, the predictor $\predictor_2$ leads to strictly lower expected payoff than predictor $\predictornew$, despite still being more ``informative'' in some sense.

The above example highlights that not all miscalibrations are equal in terms of their impact on decision-making. This raises a natural question: Which predictor is more useful for decision-making problems? Notably, this question is not fully answered even when both predictors are perfectly calibrated.
In particular, by viewing a predictor as an information structure, the Blackwell informativeness order offers a partial answer to this question \citep{B-51,B-53}. Intuitively, the Blackwell informativeness and its induced Blackwell order captures whether one perfectly calibrated predictor is always more useful than another in every decision problem, thereby inducing a partial order over the space of perfectly calibrated predictors.
However, not every pair of perfectly calibrated predictors is comparable under the Blackwell order, let alone pairs of possibly miscalibrated predictors. This work aims to study the following questions:

\begin{displayquote}
    \emph{Can we compare {any two (possibly miscalibrated)} predictors based on how ``useful'' they are to the decision-making problems? \\
    If they are not equally useful, can we quantify their gap, and is there a natural measure that satisfies common sense desiderata for characterizing this difference?}
\end{displayquote}

\subsection{Main Results}
In this paper, we provide principled answers to both of the motivating questions.
In line with most prior work in the calibration literature, we focus on predictors for binary outcomes.

\xhdr{Informativeness gap} As our first conceptual contribution, we introduce and study a new notion called the \emph{informativeness gap}, denoted by $\inforGap{\cdot, \cdot}$.
Given any two (possibly miscalibrated) predictors $\predictor$ and $\predictornew$, the informativeness gap $\inforGap{\predictor, \predictornew}$
quantifies the maximum payoff advantage that predictor $\predictor$ offers over predictor $\predictornew$ across all {\em normalized} decision-making tasks. 
Here, the normalized decision-making tasks refer to those in which, for any action, the decision-maker's payoff difference over different outcomes are bounded by one (see the formal definition in \Cref{def:informativeness gap}).\footnote{Alternatively, one could consider a \emph{multiplicative} informativeness gap, defined as the maximum ratio between the payoffs under two predictors across all decision problems. However, we focus on and study the \emph{additive} version, as it aligns more naturally with the decision-theoretic calibration literature, which primarily emphasizes additive regret.} 
\begin{align*}
    \inforGap{\predictor, \predictornew} \triangleq 
    \underset{\substack{\textrm{payoff-normalized}\\\textrm{decision problem}}}{\sup}
    \agentExpPay{\predictor} - \agentExpPay{\predictornew}~,
\end{align*}
where $\agentExpPay{\predictor}$ denotes the expected utility in a particular decision problem obtained by the decision-maker who naively best-responds to the predictions produced by predictor.
This definition provides an operational analogue to Blackwell informativeness, which interprets the relative usefulness of two predictors through the decision-maker's maximum payoff difference on all possible decision problems.

Our proposed informativeness gap $\inforGap{\predictor, \predictornew}$ also subsumes several existing decision-theoretic notions of calibration. For instance, U-Calibration corresponds to the special case where predictor $\predictor$ always outputs the true base rate, while CDL arises when $\predictor$ outputs the true conditional empirical frequencies induced by predictor $\predictornew$.
When both $\predictor$ and $\predictornew$ are perfectly calibrated, the Blackwell order offers a partial characterization of $\inforGap{\predictor, \predictornew}$. In particular, we have $\inforGap{\predictor, \predictornew} = 0$ if and only if $\predictornew$ is more Blackwell informative than $\predictor$. In other words, the Blackwell order provides an ordinal, partial comparison--indicating which predictor is more informative--without quantifying how much more informative one is than the other.
However, the Blackwell framework does not extend to miscalibrated predictors, nor does it yield a cardinal measure of informativeness. The main contribution of this paper is to go beyond these limitations by extending the informativeness gap $\inforGap{\predictor, \predictornew}$ to \emph{any pair} of predictors--including miscalibrated ones--while providing a meaningful \emph{cardinal} measure of their relative usefulness.

\xhdr{Characterizing $\inforGap{\cdot,\cdot}$ between perfectly calibrated predictors}
We begin by analyzing the informativeness gap $\inforGap{\predictor,\predictornew}$ in the special but theoretically fundamental case where both predictors $\predictor$ and $\predictornew$ are perfectly calibrated. Our main result provides a characterization of this gap, and interestingly, reveals that it closely resembles a relaxed variant of the well-known earth mover's distance (EMD), also known as Wasserstein distance, between probability distributions. Motivated by this connection, we introduce the quantity $\REMD{\cdot, \cdot}$, defined over probability distributions. We then illustrate how to interpret it as an informativeness measure that quantifies how much more useful predictor $\predictor$ is compared to predictor $\predictornew$ in decision-making tasks.

\begin{mythm}[\Cref{def:REMD} and \Cref{thm:strong-duality-perfectly-calibrated} (Informal)]
For any two distributions $\dist,\distNew$, define the \emph{relaxed earth mover's distance} $\REMD{\dist,\distNew}$ as
\begin{align*}
    \REMD{\dist,\distNew} \triangleq 
    \inf\nolimits_{\miscali\in \couplingSpace(\dist, \distNew) } \int_0^1 \abs{\int_0^1 \miscali(\prediction, \predictionNew)\cdot (\prediction - \predictionNew) \, \dd \predictionNew }  \, \dd \prediction
\end{align*}
and $\couplingSpace(\dist, \distNew)$ denotes the set of all couplings between two distributions.
For any two perfectly calibrated predictors $\predictor, \predictornew$, their informativeness gap $\inforGap{\predictor, \predictornew}$ satisfies 
\begin{align*}
    \inforGap{\predictor, \predictornew} = \REMD{\PDF, \PDFNew}
\end{align*}
where $\PDF, \PDFNew \in \Delta([0, 1])$ denote the distributions over predictions generated by the two predictors $\predictor,\predictornew$, respectively.
\end{mythm}
While the informativeness gap $\inforGap{\cdot,\cdot}$ is defined over predictors, our relaxed earth mover's distance $\REMD{\cdot,\cdot}$ is defined over distributions. As we discuss below, $\REMD{\cdot,\cdot}$ admits interesting structural characterizations and may be of independent interest---even beyond the context of calibration or forecasting.\footnote{Throughout the paper, we use $\PDF$ and $\PDFNew$ to denote the prediction distributions associated with predictors $\predictor$ and $\predictornew$, respectively, and $\dist$ and $\distNew$ to denote general distributions.}

By definition, we always have $\REMD{\dist, \distNew} \ge 0$ for all distributions $\dist, \distNew$. When $\REMD{\PDF, \PDFNew} = 0$ for prediction distributions $\PDF$ and $\PDFNew$, our results guarantee that, for all decision problems, the decision-maker obtains weakly higher utility under the perfectly calibrated predictor $\predictornew$ than under $\predictor$. In this case, $\predictornew$ is more Blackwell informative than $\predictor$.
Conversely, when $\predictor$ is more Blackwell informative than $\predictornew$, the value $\REMD{\PDF, \PDFNew}$ quantifies the extent to which $\predictor$ is more informative than $\predictornew$, in terms of the maximal utility gap achievable across all payoff-normalized decision problems.

Since $\REMD{\PDF, \PDFNew}$ can be viewed as a relaxation of the classical earth mover's distance (EMD), its value is always weakly smaller than that of the EMD. Given their structural similarity, one might wonder whether the EMD could serve as a proxy for the informativeness gap. Unfortunately, this is not the case: it is not difficult to construct examples where the multiplicative gap between the $\REMD{\PDF,\PDFNew}$ and the $\EMD{\PDF,\PDFNew}$ of two prediction distributions $\PDF,\PDFNew$ is arbitrarily large (see \Cref{ex: unbounded EMD and relaxed EMD}).

To establish \Cref{thm:strong-duality-perfectly-calibrated}, we first formulate $\inforGap{\predictor, \predictornew}$ as an infinite-dimensional linear program (LP), and then show that it admits a strong duality with the measure $\REMD{\PDF, \PDFNew}$.
Along the way, we also derive alternative and insightful representations of $\REMD{\cdot,\cdot}$, which may be of independent interest.
\begin{mythm}[\Cref{prop:idf} and \Cref{prop:join equivalency} (Informal)]
For two distributions $\dist, \distNew$ with the same mean and supports contained in $[0, 1]$, $\REMD{\dist, \distNew}$ admits the following two representations:
\begin{itemize}
    \item 
    {\em SCDFs representation}: $\REMD{\dist, \distNew} = 2\cdot \max_{t\in [0, 1]}\distSCDF(t) - \distSCDFNew(t)$ where $\distSCDF, \distSCDFNew$ denote the super-cumulative distribution functions (SCDFs) of distributions $\dist, \distNew$, respectively -- that is, the integrals of their cumulative distribution functions. Also see \Cref{fig:CCDF}.
    \item 
    {\em Lattice representation}: $\REMD{\dist, \distNew} = \REMD{\dist\vee \distNew, \distNew}$
    where $\dist\vee \distNew$ denotes the join of distributions $\dist$ and $\distNew$ under the Blackwell partial order. Also see \Cref{fig:join}.
\end{itemize}
\end{mythm}
While the above representations are stated for $\REMD{\cdot,\cdot}$ over distributions, by invoking the equivalence between $\inforGap{\cdot,\cdot}$ and $\REMD{\cdot,\cdot}$ (\Cref{thm:strong-duality-perfectly-calibrated}), we also obtain the analogous results for the informativeness gap $\inforGap{\cdot,\cdot}$ over predictors. Specifically, the SCDF representation shows that $\inforGap{\predictor, \predictornew}$ is exactly twice the maximum gap between the SCDFs of the two prediction distributions. This offers a simple yet precise way to quantify the informativeness gap when the predictors are perfectly calibrated.
It is known that the space of all perfectly calibrated predictors, together with the Blackwell order, forms a lattice structure \citep{KR-00,MS-06}. Our lattice representation provides a qualitative interpretation of this structure: $\inforGap{\predictor, \predictornew}$ can be viewed as the informativeness gap between predictor $\predictornew$ and the least upper bound (join) of $\predictor$ and $\predictornew$ in the lattice.

\xhdr{Characterizing $\inforGap{\cdot,\cdot}$ between miscalibrated predictors}
We now turn to the setting where both predictors may be miscalibrated. To capture the informativeness gap in the presence of miscalibration, it is essential to incorporate the true frequencies of outcomes conditional on predictions. Specifically, for any prediction $\prediction \sim \predictor$, we let $\truePredic_\predictor(\prediction) \in [0, 1]$ denote the true outcome frequency given the prediction.
Our main result introduces a strict extension of the measure $\REMD{\cdot, \cdot}$. This generalized measure, denoted by $\earthDistFlow{\dist, \distNew, \truePredic_1, \truePredic_2}$, takes as input two distributions and two corresponding outcome functions.

\begin{mythm}[\Cref{def:REMD miscali} and \Cref{thm:strong-duality-miscalibrated} (Informal)]
For any two distributions $\dist,\distNew\in\Delta([0,1])$ and two functions $\truePredic_1,\truePredic_2:[0,1]\rightarrow[0,1]$, define \emph{informativeness measure} $\earthDistFlow{\dist,\distNew,\truePredic_1,\truePredic_2}$ as
\begin{align*}
    \earthDistFlow{\dist, \distNew,\truePredic,\truePredic_2} \triangleq
    \inf\limits_{\miscali\in \feasibleflowSet(\dist, \distNew)} ~  & 
    \int_0^1 \left|\int_0^1 \miscali(\prediction, \predictionNew) (\prediction - \predictionNew)\,\dd \predictionNew 
    + (\truePredic_1(\prediction) - \prediction) \dist(\prediction)
    - (\truePredic_2(\prediction) - \prediction) \distNew(\prediction) \right|  \, \dd \prediction~,
\end{align*}
where the set $\feasibleflowSet(\predictor, \predictornew)$, referred to as the \emph{flow set} is a strict superset of the coupling set $\couplingSpace(\dist,\distNew)$, which imposes ``flow conservation'' constraint over $\miscali\in\Delta([0,1]\times[0,1])$.

For any two (possibly miscalibrated) predictors $\predictor, \predictornew$, their informativeness gap $\inforGap{\predictor, \predictornew}$ satisfies 
\begin{align*}
    \inforGap{\predictor, \predictornew} = \earthDistFlow{\PDF, \PDFNew,\truePredic_\predictor,\truePredic_\predictornew}
\end{align*}
where $\PDF,\PDFNew$ (resp.\ $\truePredic_\predictor,\truePredic_\predictornew$) are the prediction distributions (resp.\ true outcome frequencies) of predictors $\predictor,\predictornew$, respectively. Moreover, the informativeness gap $\inforGap{\predictor, \predictornew}$  can also be expressed as 
\begin{align*}
    \inforGap{\predictor, \predictornew} =
    2\cdot \max_{t\in [0, 1]} ~
    \left(\scdf_\predictor(t) + \int_0^t (\prediction - \truePredic_\predictor(\prediction))\cdot \PDF(\prediction)\, \dd \prediction \right)
    - 
    \left(\scdf_\predictornew(t) + \int_0^t (\prediction - \truePredic_\predictornew(\prediction)) \cdot \PDFNew(\prediction)\, \dd \prediction\right)
\end{align*}
\end{mythm}
To understand how $\earthDistFlow{\PDF, \PDFNew,\truePredic_\predictor,\truePredic_\predictornew}$ generalizes $\REMD{\PDF, \PDFNew}$, notice that when both predictors $\predictor, \predictornew$ are perfectly calibrated, we have $\truePredic_\predictor(\prediction) = \truePredic_\predictornew(\prediction)= \prediction$ for all $\prediction\in[0, 1]$.
In this case, the objective in $\earthDistFlow{\PDF, \PDFNew,\truePredic_\predictor,\truePredic_\predictornew}$ reduces exactly to that of $\REMD{\PDF, \PDFNew}$.
Moreover, the flow set $\feasibleflowSet(\PDF, \PDFNew)$ (see \Cref{sec:misc} for the formal definition) generalizes the standard coupling set $\couplingSpace(\PDF, \PDFNew)$, and thus it represents a relaxed constraint.

Interestingly, by \Cref{thm:strong-duality-perfectly-calibrated} and \Cref{thm:strong-duality-miscalibrated}, 
we can see that when both predictors are perfectly calibrated, the two measures equal to each other: $\earthDistFlow{\PDF, \PDFNew, \truePredic_\predictor,\truePredic_\predictornew} = \REMD{\PDF, \PDFNew}$. 
In other words, optimizing over the relaxed flow set is also sufficient to characterize the informativeness gap  in the perfectly calibrated setting.

Given the close connection between the flow set $\feasibleflowSet(\PDF, \PDFNew)$ and the coupling set $\couplingSpace(\PDF, \PDFNew)$, one might wonder whether simply optimizing over $\couplingSpace(\PDF, \PDFNew)$ in $\earthDistFlow{\PDF, \PDFNew, \truePredic_{\predictor}, \truePredic_{\predictornew}}$ would suffice -- or at least yield a good approximation -- to the informativeness gap $\inforGap{\predictor, \predictornew}$. 
However, this is not the case: there exist predictors $\predictor, \predictornew$ for which the gap between optimizing over $\couplingSpace(\PDF, \PDFNew)$ in $\earthDistFlow{\PDF, \PDFNew,\truePredic_\predictor,\truePredic_\predictornew}$ and $\inforGap{\predictor, \predictornew}$ is arbitrarily large (see \Cref{prop:bad approx of marg}). To prove this result, we also extend the SCDF representation in \Cref{prop:idf} to miscalibrated predictors. The resulting expression, which appears as the final term in the theorem statement, is given by
\begin{align*}
    \scdf_\predictor(t) + \int_0^t (\prediction - \truePredic_\predictor(\prediction)) 
        \cdot \PDF(\prediction)\, \dd\prediction
\end{align*}
and is referred to as the \emph{calibration-adjusted SCDF (CA-SCDF)}. As a sanity check, when the predictor is perfectly calibrated, the CA-SCDF reduces to the SCDF of the prediction distribution.

\xhdr{Desiderata of our informativeness gap and informativeness measure}
\cite{BGHN-23} proposed the following desiderata for an ideal calibration measure that should satisfy: 
(1) the \emph{prediction-only accessibility}: the measure can be evaluated by only querying the pair of prediction and its associated realized outcome;
(2) the \emph{consistency}: the measure is robust \emph{completeness} (correct predictions have low error) and robust \emph{soundness} (incorrect predictions have high error);
and (3) the \emph{sample efficiency}: the measure can be computed within accuracy $\varepsilon$ in time $\poly(\sfrac{1}{\varepsilon})$ with using $\poly(\sfrac{1}{\varepsilon})$ samples of pair (prediction, realized outcome). 

Our proposed informativeness measure $\earthDistFlow{\cdot, \cdot,\cdot,\cdot}$ can be served as a tool for quantifying the informativeness gap between (possibly miscalibrated) predictors. 
By \Cref{thm:strong-duality-miscalibrated}, $\earthDistFlow{\PDF, \PDFNew,\truePredic_\predictor,\truePredic_\predictornew}$ is both complete and sound (in fact, it satisfies these criteria exactly), and thus it is consistent.
By our definition, our proposed informativeness gap $\inforGap{\predictor,\predictornew}$ and its equivalence representation, i.e., informativeness measure $\earthDistFlow{\PDF, \PDFNew,\truePredic_\predictor,\truePredic_\predictornew}$, are also prediction-only accessible: notice that for any predictor $\predictor$, its true conditional frequency $\truePredic_\predictor(\prediction)$ can be computed using only the sample pair (prediction, realized outcome).  
Lastly, we also present sample complexity bounds for estimating the measure $\earthDistFlow{\PDF, \PDFNew,\truePredic_\predictor,\truePredic_\predictornew}$ and thus the informativeness gap $\inforGap{\predictor,\predictornew}$.
The algorithm for this estimation utilizes our structural characterizations for $\earthDistFlow{\PDF, \PDFNew,\truePredic_\predictor,\truePredic_\predictornew}$, and has a sample complexity of $\poly(\sfrac{1}{\varepsilon})$ (see \Cref{thm:sample miscali}).

Finally, another desiderata that is usually considered in the machine learning is the \emph{continuity} of a measure: the measure should be continuous w.r.t.\ the prediction distribution. 
This raises a natural question: does there exist an informativeness measure that satisfy the above three desiderata and this additional continuity property simultaneously?
Perhaps not surprisingly, the answer is no. 
We present a general impossibility result showing that no informativeness measure can satisfy consistency (i.e., completeness and soundness) and continuity simultaneously.
\begin{mythm}[\Cref{prop:impossibility} (Informal)]
For any informativeness measure over predictors, at least one of the following must fail: it is complete, it is sound, and it is continuous. 
\end{mythm}
Finally, we note that the informativeness gap $\inforGap{\cdot,\cdot}$, informativeness measures $\REMD{\cdot,\cdot}$, and $\earthDistFlow{\cdot,\cdot,\cdot,\cdot}$ also connect to other well-known calibration measures like ECE, distance to calibration, which we elaborate on these connections in \Cref{subsec:relation to other measures}.

\xhdr{Experiments with LLM-based forecasters}
We complement our theory with experiments on LLM-based forecasters in real-world weather and Bitcoin prediction tasks.
In particular, we show that our proposed {\InfoGap} (1) provides an alternative -- and more decision-relevant -- criterion to traditional metrics like the Brier score/ECE for benchmarking LLMs' forecasting capabilities (see \Cref{subsec:benchmark});
and (2) it offers a principled way to evaluate how ad hoc post-processing methods designed for reducing model miscalibration impact decision usefulness (see \Cref{subsec:post-processing}).

To this end, we instantiate our empirical study with LLM-based forecasters - an increasingly used class of predictors whose probabilistic outputs can be miscalibrated. Specifically, we focus on four widely used LLM forecasters on two real-world prediction domains - weather and Bitcoin, and evaluate them using standard LLM benchmarking metrics (Brier score and ECE) along with our {\InfoGap}. On each dataset, we design event templates that can be instantiated repeatedly, yielding labeled instances $\{(\record_i, \outcome_i)\}_{i\in[n]}$. To ensure a fair comparison, we adopt a unified prompting protocol and require each LLM to output a valid JSON object containing probabilities constrained to $[0,1]$, enabling a like-for-like evaluation across LLM predictors (see \Cref{apx:experimental process} for implementation details).

As we discussed, our proposed {\InfoGap} is asymmetric, so it is possible that both $\inforGap{\predictor, \predictornew}$ and $\inforGap{\predictornew, \predictor}$ are large. 
A natural first question, then, is whether {\InfoGap}  is nonetheless comparable across most relevant settings (or ``in most cases'').
Building on the above setup, 
we begin by examining whether {\InfoGap} comparisons are practically meaningful in realistic settings. Empirically, we find that comparability remains prevalent even under a very small tolerance across both datasets and all LLMs considered in our experiments, providing a solid basis for subsequent model evaluation and comparison. 

Another practical question in deployment is predictor selection: which predictor can a decision maker rely on? Standard benchmarking metrics capture different aspects of predictive quality and can yield conflicting rankings or near ties, providing little useful guidance in practice. We therefore ask whether our informativeness gap can offer a complementary and decision-relevant perspective in such ambiguous cases. 

We then present our experiments which benchmark LLM forecasters through the lens of predictor selection in \Cref{subsec:benchmark}. 
We then highlight two representative case studies, one from the weather dataset and the other from the Bitcoin dataset. These cases reflect two ambiguous regions that aligns the motivation's our work: in one, Brier score and ECE give conflicting signals, while in the other, they are nearly tied. In both settings, {\InfoGap} yields a clearer directional separation between predictors, and we discuss the observations from a decision-making perspective to shed light on why this separation arises. 

Finally, in our second set of experiments, 
we evaluate six common calibration post-processing methods (e.g., binning-based, isotonic, and scaling-based approaches).
Rather than evaluating these methods solely by their improvements in ECE, 
we examine whether they also preserve, or improve, decision usefulness as measured by {\InfoGap}.
Empirically, we find that post-processing methods is not uniformly beneficial from the InfoGap perspective: several widely used procedures can degrade InfoGap for a substantial fraction of (LLM, event) tasks even when they reduce ECE.
This highlights the practical importance of assessing post-processing with {\InfoGap}, especially when the downstream objective is decision-level performance. 
Motivated by this observation, we propose an {\InfoGap}-regularized variant for objective-based post-processing methods by augmenting the standard likelihood-based objective with an {\InfoGap}  penalty. Across the methods with an explicit optimization component, this regularization systematically reduces InfoGap harm, with the clear improvements observed for the post-processing methods like temperature scaling.
(see \Cref{subsec:post-processing}).

\subsection{Related Work}

\newcommand{\tnorm}{\ell_t}

Our work contributes to a growing line of research on calibration in decision-making settings.

\xhdr{Informativeness of predictors}
The goal of this paper is to quantify and compare the informativeness gap of any given pair of predictors.
To our best knowledge, the study most closely related to ours is  \cite{GKR-19} who study how to compare the information content of multiple predictors. 
However, their main focus is understand how a {\em calibrated} predictor's information content shapes fairness outcomes, 
while our focus, by contrast, is a general characterization of the informativeness gap itself, independent of any particular fairness criterion.
Beyond the machine learning community, viewing the predictors as some information sources, the theoretical economics and statistics literature has long been concerned with comparing the value of information sources (see, e.g., \citealp{B-51,B-53,L-88, KR-00, MS-06, MPST-21,BFK-22,BFK-24}).  
These works study Blackwell experiments, which in our terminology correspond to perfectly calibrated predictors.
To the best of our knowledge, this paper is the first to extend the notion of informativeness to predictors that may be miscalibrated.

\xhdr{Decision aspects of (miscalibrated) predictors}
Beginning with \cite{D-82,FV-98}, it is well established that agents who best respond to calibrated predictions (i.e., those with low ECE) achieve diminishing swap regret.
Recent works have extended this foundation by introducing new decision-theoretic calibration measures that offer finer-grained guarantees on the regret incurred by downstream agents.

In particular, \citet{KLST-23} introduce ``U-Calibration'' -- a measure bounded by a small constant times the ECE, and show that sublinear  U-Calibration is both
necessary and sufficient for ensuring sublinear regret for all downstream agents.
\citet{RS-24,HW-24} consider a similar setting but focusing on swap regret.
In particular, \citet{RS-24} show that
by ensuring unbiasedness relative to a carefully chosen collection of events, one can achieve diminishing swap regret for arbitrary downstream agents with an improved rate compared to using calibrated predictions.
\citet{HW-24} then introduce ``Calibration Decision Loss (CDL)'', defined as the maximum improvement (over all decision tasks) in decision payoff obtainable by recalibrating the predictions.
The show that CDL is upper bounded by twice the ECE and that a vanishing CDL guarantees vanishing payoff loss from miscalibration across all decision tasks,
which removes the regret dependence on the number of actions appeared in the results of \citet{RS-24}.
\cite{QZ-25}
propose a new calibration measure called ``subsampled step calibration'', which has both decision-theoretic implications and a truthfulness property.
These works mostly focus the decision loss of a given predictor compared to a reference calibrated predictor, while our focus is to compare the decision loss of any given two (possibly miscalibrated) predictors.

Other conceptually related works include \citet{NRRX-23,HPY-23,JP-24,FT-25,TZFW-25} who focus on single decision task, and \citet{CHJ-20,CRS-24} who explore repeated principal-agent interactions in a prior-free setting.
\citet{ZKS-21} introduce ``decision calibration'' that requires the predicted distribution and true distribution to be indistinguishable to a set of downstream decision-makers. \citet{GKRSW-22} develop an omniprediction framework which aims to design a single predictor that ensures every downstream decision-maker's loss is no worse than some benchmark family of functions \citep{GHKRW-23,GJRR-24,OKK-25}. 

\xhdr{Calibration error measures}
In addition to above calibration measures that are tailored to decision-making environments, there are also various  measures proposed to capture the calibration errors.  
For example, the aforementioned ECE metric is broadly accepted in the forecasting community \citep{D-82,FV-98,GBR-07,RG-10}, within theoretical computer science (see, e.g., \citealp{QV-21,DDFM-24,CDV-24}), empirical machine learning community (see, e.g., \citealp{GK-14,R-21}). 
Beyond the standard ECE, other calibration measures have been proposed to capture different desired aspects, e.g., multi-class calibration \citep{GHR-24}, distance to calibration \citep{BGHN-23,QZ-24,ACRS-25}.

\xhdr{Benchmarking for LLM prediction} There is a growing body of work on using large language models as predictors across diverse settings, including time-series forecasting \citep[e.g.,][]{JWMC-23}, real-world event probability forecasting \citep[e.g.,][]{HZYS-24} and tool-augmented agentic event forecasting \citep[e.g.,][]{YHDZ-24}. 

As these approaches expand, there is a increasing need for a standardized  evaluation protocol. Accordingly, a number of recent studies propose benchmarking protocols and datasets to assess LLMs' predictive quality. For example, \citet{YMLG-25} introduce a return-based metric ``average return'' in Prophet Arena, a live forecasting benchmark built on real-world prediction markets. In addition to standard metrics such as Brier score and ECE, average return evaluates probabilistic forecasts by their expected trading payoff at prevailing market prices, which provides an economical measure of predictive quality. \citet{KBYJ-24} propose ForecastBench, a dynamic benchmark with a public leaderboard that evaluates LLMs' predictions via Brier scores. It continuously collects new and unresolved questions through an automated pipeline, enabling timely comparison as outcomes resolve. Another dynamically updated live benchmark is FutureX \citep{ZLCH-25}, which features an automated evaluation pipeline and task-specific metrics, including 0-1 accuracy for single-choice question and F1 score for multi-choice question. In addition to the general-purpose benchmarks above, MIRAI \citep{YHDZ-24} proposes a domain-specific benchmark for international event forecasting with a tool-augmented setup, reporting performance using metrics such as F1 score. Compared with those prior work, we directly compare the LLM's predictions from the downstream decision-making perspective using $\inforGap{\cdot,\cdot}$.

\section{Preliminaries}
\label{sec:prelim}

In this work, we are interested in the predictors for the the binary classification task. We explain their connection to decision making and their informativeness in \Cref{subsec:prelim:decision making,subsec:prelim:black well}, respectively.

\subsection{Predictions for Downstream Decision Making}
\label{subsec:prelim:decision making}

Suppose there is a binary classification task, where the goal is to predict an unknown binary outcome $\outcome \in \{0, 1\} \triangleq \outcomeSpace$. A \emph{predictor} $\predictor$ for this task encodes a distribution $\PDF\in\Delta([0,1])$ over \emph{predictions} $\prediction \in [0, 1]$. A realized prediction $\prediction$ can be interpreted as the ``predicted'' probability that the outcome $\outcome$ equals one. With a slight abuse of notation, we use $\PDF(\prediction)$ to denote the probability density (or mass) function (PDF) of the predictor $\predictor$'s prediction distribution. We also denote by $\CDF$ as the corresponding cumulative distribution function (CDF).\footnote{Previous work such as \citet{BGHN-23} models a predictor $\predictor$ as a stochastic mapping from a given feature vector $\context$ to a predicted probability $\prediction$ of the outcome $\outcome$. In contrast, we do not explicitly define the feature space or introduce $\context$ in our model. The goal of this work is to compare the informativeness of different predictors, which may originate from machine learning models trained on different datasets with potentially incomparable or heterogeneous feature configurations.
Moreover, treating the predictor as a black box---without requiring access to its input features---is often desirable in practice. This setting is known as the prediction-only access model \citep{BGHN-23}, and our formulation is naturally aligned with it. See \Cref{sec:sample complexity} for more discussions.}

We say a predictor $\predictor$ is \emph{perfectly calibrated} if, conditional on its realized prediction $\prediction$, the true probability of the outcome being one is exactly equal to $\prediction$. Formally,
\begin{align*}
\label{eq:perfect calibration definition}
\tag{perfectly calibrated}
\forall \prediction \in \supp(\PDF):\qquad
\prob{\outcome = 1 \condition \prediction} = \prediction~,
\end{align*}
where $\supp(\PDF)$ denotes the support of the distribution $\PDF$, i.e., the set of realized predictions that occur with positive probability.

In most practical machine learning scenarios, perfect calibration is rarely achievable. Predictors are often miscalibrated, meaning their predicted probabilities do not match the true conditional probabilities. For a given predictor $\predictor$, we define the \emph{true expected outcome}\footnote{Since the outcome $\outcome$ is binary, the probability that it equals one is equal to its expectation.} function conditional on the prediction $\prediction$ as
\begin{align*}
\truePredic_\predictor(\prediction) \triangleq \prob{\outcome = 1 \mid \prediction}.
\end{align*}
As a sanity check, a predictor $\predictor$ is perfectly calibrated if and only if $\truePredic_\predictor(\prediction) = \prediction$ for all $\prediction \in \supp(\PDF)$.

\xhdr{Decision making with prediction}
Suppose there is a downstream decision maker (henceforth, the agent) who uses predictions to inform his decisions. The agent's decision problem is characterized by an action set $\actionSpace$ and a utility function $\agentU: \actionSpace \times \outcomeSpace \rightarrow \reals$, which maps each action $\action \in \actionSpace$ and realized outcome $\outcome \in \outcomeSpace$ to a utility received.

The agent uses the realized prediction $\prediction$ generated by the predictor to inform his decision. Following the recent literature \citep{HPY-23,GJRR-24,KLST-23,HW-24,RS-24,QZ-25}, we define the agent's \emph{best response} $\bestr(\prediction) \in \actionSpace$ when the agent \emph{trusts} the provided prediction $\prediction$. Specifically, this is the action that maximizes the agent's expected utility assuming the outcome $\outcome$ equals one with probability $\prediction$:
\begin{align}
\label{eq:agent bestr}
    \bestr(\prediction) \in \argmax\nolimits_{\action\in\actionSpace}~
    \expect[\outcome\sim \Bern(\prediction)]{\agentU(\action, \outcome)}
    =
    \argmax\nolimits_{\action\in\actionSpace}~
    \prediction\cdot 
    \agentU(\action, 1)
    +
    (1-\prediction)\cdot 
    \agentU(\action, 0)
\end{align}
where $\Bern(\prediction)$ denotes the Bernoulli distribution with mean $\prediction$.

Let $\tildeagentU(\prediction, \outcome) \triangleq \agentU(\bestr(\prediction), \outcome)$ denote the utility the agent receives upon taking the best response to prediction $\prediction$ when the true outcome is $\outcome \in \outcomeSpace$.
Given any predictor $\predictor$, the agent's expected utility (also referred to as payoff) when always best-responding to realized predictions is given by:
\begin{align*}
    \agentExpPay{\predictor}  \triangleq \expect[\outcome]{\expect[\prediction\sim\PDF]{\tildeagentU(\prediction, \outcome)}}~.
\end{align*}
We further define the agent's interim utility when best responding to a prediction $\prediction$, while the true outcome is drawn from the Bernoulli distribution $\Bern(\predictionNew)$:
\begin{align*}
\indirectU(\predictionNew; \prediction) \triangleq \predictionNew \cdot \tildeagentU(\prediction, 1) + (1 - \predictionNew) \cdot \tildeagentU(\prediction, 0)~.
\end{align*}
Using this definition, the agent's expected utility under predictor $\predictor$ can be equivalently expressed in terms of the interim utility function:
\begin{restatable}{lemma}{lemInterimUtilExpression}
\label{lem:interim utility expression}
Given any predictor $\predictor$, the agent's expected utility $\agentExpPay{\predictor}$ when he best responds to every prediction $\prediction\sim \PDF$ can be formulated as 
\begin{align*}
    \agentExpPay{\predictor} &= \displaystyle\int_0^1 \PDF(\prediction) \cdot \indirectU(\truePredic_\predictor(\prediction); \prediction)\, \dd\prediction~.
\intertext{Consequently, if predictor $\predictor$ is perfectly calibrated, the agent's expected utility $\agentExpPay{\predictor}$ is}
    \agentExpPay{\predictor} &= 
    \int_0^1 \PDF(\prediction) \cdot \indirectU(\prediction)\, \dd\prediction~,
\end{align*}
where $\indirectU(\prediction) \triangleq \indirectU(\prediction; \prediction)$ denotes the univariate form of the agent's interim utility function.
\end{restatable}

\subsection{Blackwell Order and Informativeness Gap between Predictors}
\label{subsec:prelim:black well}
Suppose we are given two predictors, $\predictor$ and $\predictornew$, along with their respective true expected outcome functions, $\truePredic_\predictor$ and $\truePredic_\predictornew$. As we discussed in the introduction, natural and important question arises: \emph{can we compare these predictors based on how informative they are?} 

A partial answer to this question is provided by the notion of informativeness introduced by \citet{B-51,B-53}, which applies when the predictors are perfectly calibrated. In particular, a perfectly calibrated predictor can be viewed as an information structure---that is, a distribution over posterior beliefs about the binary outcome $\outcome$ induced by some signaling scheme.\footnote{A signaling scheme is a stochastic mapping from the realized outcome $\outcome$ to a signal from a given signal space. Upon observing a realized signal, a posterior belief about the outcome can be formed via Bayes' rule. Since the outcome is binary, the posterior belief is fully captured by the probability that the outcome is one. Thus, the distribution of posterior beliefs corresponds to a distribution over the probabilities of the outcome being one, with randomness induced by both the outcome and the signaling scheme. Such a distribution of posterior beliefs is mathematically equivalent to a perfectly calibrated predictor.}

We now provide the formal definition of Blackwell informativeness:

\begin{definition}[Blackwell informativeness, adopted from \citealp{B-51,B-53}]
\label{def:Blackwell order}
Fix any two perfectly calibrated predictors $\predictor, \predictornew$.
We say predictor $\predictor$ is \emph{Blackwell more informative} than predictor~$\predictornew$ (equivalently, $\predictor$ \emph{Blackwell-dominates} $\predictornew$), if for all decision problems $(\actionSpace, \agentU)$ of the agent, 
\begin{align*}
\agentExpPay[{(\actionSpace, \agentU)}]{\predictor} \ge \agentExpPay[{(\actionSpace, \agentU)}]{\predictornew}~, \quad
\text{for all } (\actionSpace, \agentU)~,
\end{align*}
where $\agentExpPay[{(\actionSpace, \agentU)}]{\predictor}$ and $\agentExpPay[{(\actionSpace, \agentU)}]{\predictornew}$ are the agent's expected utilities under predictors $\predictor, \predictornew$ in decision problem $(\actionSpace, \agentU)$, respectively.
\end{definition}
The concept of Blackwell informativeness is intuitive. In particular, this definition captures how perfectly calibrated predictors (information structures) enhance the decision maker's utility \emph{for all} decision problems. A perfectly calibrated predictor is Blackwell more informative than another if the induced utility in the former is strictly greater than the induced utility in the latter.

The definition of Blackwell informativeness above introduces a partial order over the space of perfectly calibrated predictors, known as the \emph{Blackwell order}, which has been extensively studied in the literature. Both the Blackwell informativeness and the Blackwell order exhibit well-known properties. Below, we highlight two of these properties:
\begin{lemma}[\citealp{B-51,B-53}]
\label{lem:Blackwell order:SCDF definition}
    For any two perfectly calibrated predictors $\predictor$ and $\predictornew$, predictor $\predictor$ Blackwell-dominates predictor $\predictornew$ if and only if for all $t \in [0, 1]$,
    \begin{align*}
    \SCDF(t) \geq \SCDFNew(t)~.
    \end{align*}
    Here $\SCDF$ and $\SCDFNew$ are the super-cumulative distribution functions (SCDFs) defined as   
    \begin{align*}
        \SCDF(t) \triangleq \displaystyle\int_0^t \CDF(\prediction)\,\dd \prediction
        \;\;
        \mbox{and}
        \;\;
        \SCDFNew(t) \triangleq \displaystyle\int_0^t \CDFNew(t)\,\dd \prediction~,
    \end{align*}
    where $\CDF$ and $\CDFNew$ are the CDFs of prediction distribution $\PDF$ and $\PDFNew$, respectively.
\end{lemma}

\begin{lemma}[\citealp{KR-00,MS-06}]
\label{lem:Blackwell order:lattice structure}
    The space of all perfectly calibrated predictors, along with the Blackwell order over this space, forms a lattice.
\end{lemma}
We remark on two key limitations of the Blackwell informativeness relation. First, it is only defined over the space of perfectly calibrated predictors, which can be viewed as information structures. Second, even within this restricted space, the Blackwell relation induces only a partial order, providing an ordinal comparison without quantifying how much more informative one predictor is than another.

The main focus of this work is to move beyond these limitations and develop a framework to compare the informativeness of \emph{any two} predictors, including those that may be \emph{miscalibrated}, in a \emph{cardinal} manner. To this end, we introduce the following notion of the informativeness gap.
\begin{definition}[Informativeness gap]
\label{def:informativeness gap}
Given any two (possibly miscalibrated) predictors $\predictor$ and $\predictornew$, the \emph{informativeness gap $\inforGap{\predictor, \predictornew}$} of predictor $\predictornew$ relative to predictor $\predictor$ is defined as
\begin{align}
    \label{eq:primal}
    \arraycolsep=5.4pt\def\arraystretch{1}
&\begin{array}{rlll}
    \inforGap{\predictor, \predictornew}
    \triangleq 
    \sup
    \nolimits_{(\actionSpace, \agentU)\in\agentUClass} 
    ~ 
    \displaystyle 
    \agentExpPay[{(\actionSpace, \agentU)}]{\predictor}
    - 
    \agentExpPay[{(\actionSpace, \agentU)}]{\predictornew}~,
    \end{array}
\end{align}  
where $\agentUClass$ denotes the class of all decision problems $(\actionSpace, \agentU)$ such that for every action $\action \in \actionSpace$, the utility difference is bounded:
\begin{align*}
|\agentU(\action, 1) - \agentU(\action, 0)| \le 1~.
\end{align*}
\end{definition}
Notably, the bounded utility difference in \Cref{def:informativeness gap} serves primarily as a normalization.\footnote{Previous work such as \citet{KLST-23,HW-24} consider decision problems $(\actionSpace, \agentU)$ under the assumption that utilities are bounded in $[0,1]$, i.e., $\agentU(\action, 0), \agentU(\action, 1) \in [0,1]$ for all $\action \in \actionSpace$. In contrast, our bounded utility difference condition in \Cref{def:informativeness gap} can be viewed as a relaxation of this assumption.} 
Despite this restriction, the class of decision problems remains rich enough for the informativeness gap to fully capture the original notion of Blackwell informativeness, which considers all utility functions.\footnote{Besides \Cref{prop:prelim:informativeness gap captures blackwell informativeness}, we also establish connections between the informativeness gap and other properties (\Cref{lem:Blackwell order:SCDF definition,lem:Blackwell order:lattice structure}) of Blackwell informativeness. See \Cref{prop:idf,prop:join equivalency}.} Its proof can be found in \Cref{apx:prelim:informativeness gap captures blackwell informativeness proof}.
\begin{restatable}{proposition}{propInfoGAPandBW}
\label{prop:prelim:informativeness gap captures blackwell informativeness}
    For any two perfectly calibrated predictors $\predictor$ and $\predictornew$, predictor $\predictor$ Blackwell-dominates predictor $\predictornew$ if and only if the informativeness gap of predictor $\predictor$ relative to predictor~$\predictornew$ is zero, i.e.,
    $\inforGap{\predictornew, \predictor} = 0$.
\end{restatable}
We conclude this section with the following technical lemma regarding the agent's interim utility function. Its proof can be found in \Cref{apx:Lipschitz convex indirect utility proof}.
\begin{restatable}{lemma}{lemLipConvexIndirectUtil}
    \label{lem:Lipschitz convex indirect utility}
    For every decision problem $(\actionSpace, \agentU)$, its induced  univariate interim utility function $\indirectU(\prediction)$ (see definition in \Cref{lem:interim utility expression}) is 1-Lipschitz and convex in $\prediction\in[0, 1]$ if and only if for every action $\action \in \actionSpace$, the utility difference is bounded, i.e., $|\agentU(\action, 1) - \agentU(\action, 0)| \le 1$.
\end{restatable}

\section{The Informativeness of Perfectly Calibrated Predictors}

In this section, we focus on the case where both $\predictor$ and $\predictornew$ are perfectly calibrated predictors. Our main result in this section is a dual characterization of $\inforGap{\predictor, \predictornew}$. To state our main results, we first revisit a natural $\ell_1$-distance measure between two distributions. 
\begin{definition}[EMD between distributions]
\label{defn:emd}
Given any two distributions $\dist, \distNew$ with supports $\supp(\dist)\subseteq[0, 1], \supp(\distNew)\subseteq[0, 1]$,
the \emph{earth mover's distance} $\EMD{\dist, \distNew}$ is defined as follows:
\begin{align*}
    \EMD{\dist, \distNew} \triangleq \inf\nolimits_{\miscali \in \couplingSpace(\dist, \distNew) }\int_0^1 \int_0^1 \miscali(\prediction, \predictionNew) \cdot \abs{\prediction - \predictionNew} \, \dd \predictionNew \dd \prediction~,
\end{align*}
where 
\begin{align*}
    \couplingSpace(\dist, \distNew) \triangleq \left\{\miscali\in\Delta\left([0, 1]\times[0,1]\right): \int_0^1 \miscali(\prediction, \predictionNew) \, \dd \predictionNew = \dist(\prediction), 
\int_0^1 \miscali(q, \prediction) \, \dd \predictionNew = \distNew(\prediction), \forall \prediction \in [0, 1]
\right\}
\end{align*}
is referred to as the \emph{coupling set}, and it denotes the set of all couplings for distributions $\dist, \distNew$.
\end{definition}
We note that earth mover's distance $\EMD{\predictor, \predictornew}$ is also known as the \emph{Wasserstein distance} and has appeared in previous works \citep{BGHN-23,QZ-24}, where it is referred to as the ``\emph{lower distance}'' between predictors' prediction distributions. The core of our characterization is a variant of the EMD, which we define as follows:
\begin{definition}[Relaxed EMD between distributions]
\label{def:REMD}
Given any two distributions $\dist, \distNew$ with supports $\supp(\dist)\subseteq[0, 1], \supp(\distNew)\subseteq[0, 1]$,
the \emph{relaxed earth mover's distance} $\REMD{\dist, \distNew}$ is defined as follows:
\begin{align}
    \label{opt:EMD}
    \REMD{\dist, \distNew} 
    & \triangleq 
    \inf\nolimits_{\miscali\in \couplingSpace(\dist, \distNew) } \int_0^1 \abs{\int_0^1 \miscali(\prediction, \predictionNew)\cdot (\prediction-\predictionNew) \, \dd \predictionNew 
    }  \, \dd \prediction~.
\end{align}
\end{definition}
We remark that unlike $\EMD{\cdot, \cdot}$ which can be used as a metric in the space of all distributions, our proposed $\REMD{\cdot, \cdot}$ is asymmetric as generally $\REMD{\dist, \distNew} \neq \REMD{\distNew, \dist}$ and cannot be used as a metric.
We are now ready to state our main result, which establishes the equivalence between the informativeness gap $\inforGap{\cdot, \cdot}$ and our proposed $\REMD{\cdot, \cdot}$.
\begin{theorem}
\label{thm:strong-duality-perfectly-calibrated}
For any two perfectly calibrated predictors $\predictor, \predictornew$, their informativeness gap satisfies
\begin{align*}
\inforGap{\predictor, \predictornew}= 
    \REMD{\PDF, \PDFNew}~.
\end{align*}
where $\PDF$ and $\PDFNew$ are the prediction distributions of predictors $\predictor$ and $\predictornew$, respectively.
\end{theorem}
By the triangle inequality, we always have $\REMD{\dist, \distNew} \le \EMD{\dist, \distNew}$.
One tempting thought is that whether one can approximate $\inforGap{\predictor, \predictornew}$ using the classic $\EMD{\PDF, \PDFNew}$ (e.g., whether they are polynomially related with each other). 
Perhaps not surprisingly, one can show that the multiplicative gap between $\inforGap{\predictor, \predictornew}$ and $\EMD{\PDF, \PDFNew}$ can be arbitrarily large. 
\begin{example}
\label{ex: unbounded EMD and relaxed EMD}
Consider perfectly calibrated predictor $\predictor$ outputs deterministic prediction $\prediction = 0.5$, and perfectly calibrated predictor $\predictornew$ outputs a prediction $\prediction\sim U[0, 1]$ uniformly at random.
Under this example, predictor $\predictornew$ Blackwell-dominates $\predictor$, thus $\inforGap{\predictor, \predictornew} = 0$, while $\EMD{\PDF, \PDFNew} = 0.25$.
\end{example}

In the remainder of this section, we first provide the formal proof of \Cref{thm:strong-duality-perfectly-calibrated} in \Cref{subsec:strong duality perfectly calibrated proof}. We then establish additional structural characterization results of $\REMD{\cdot,\cdot}$ in \Cref{subsec:structural characterzation perfecly calibrated}.

\subsection{Proof of \texorpdfstring{\Cref{thm:strong-duality-perfectly-calibrated}}{Theorem 3.1}}
\label{subsec:strong duality perfectly calibrated proof}

The proof of \Cref{thm:strong-duality-perfectly-calibrated} consists of three main steps: 
\squishlist
    \item In step 1 (see \Cref{lem:simplify}), we reformulate the informativeness gap as an optimization program over all the 1-Lipschitz convex interim  utility functions. 
    \item In step 2 (see \Cref{lem:primal intermiddle}), 
    we further cast program \ref{eq:primal perfect} to a new program (see program \ref{eq:primal perfect intermiddle}) that is more tractable to establish the equivalence to $\REMD{\dist,\distNew}$.
    \item In step 3 (see \Cref{lem:last step}),
    we use strong duality to prove the equivalence between the program \ref{eq:primal perfect intermiddle} and $\REMD{\dist, \distNew}$.
\squishend

\begin{lemma}
\label{lem:simplify}
For any two perfectly calibrated predictors $\predictor, \predictornew$, their informativeness gap satisfies
\begin{align*}
    \inforGap{\predictor, \predictornew}= 
    \OBJ{\ProGOne{\PDF,\PDFNew}}
\end{align*}
where $\OBJ{\ProGOne{\PDF,\PDFNew}}$ is the optimal objective value of the following program with inputs $\dist\gets \PDF$ and $\distNew\gets\PDFNew$ (i.e., equal to two predictors' prediction distributions $\PDF$ and $\PDFNew$):
\begin{align}
\label{eq:primal perfect}
    \arraycolsep=5.4pt\def\arraystretch{1}
    \tag{$\textsc{P}_1[\dist,\distNew]$}
    &\begin{array}{rlll}
    \sup
    \limits_{\indirectU:[0,1]\rightarrow\reals} 
    ~ &
    \displaystyle 
    \int_0^1
    \indirectU(\prediction) \cdot \left(\dist(\prediction) - \distNew(\prediction)\right)~\dd  \prediction 
    \quad  &
    \vspace{1mm} \text{s.t.}
    \\
    & 
    \displaystyle 
    \indirectU \text{ is $1$-Lipschitz convex},
    &  
    \vspace{1mm}
    \end{array}
\end{align}  
\end{lemma}
\begin{proof}
Invoking \Cref{lem:interim utility expression,lem:Lipschitz convex indirect utility} and the definition of $\inforGap{\predictor, \predictornew}$ completes the proof of \Cref{lem:simplify}.
\end{proof}

\begin{lemma}
\label{lem:primal intermiddle}
For any two distributions $\dist, \distNew$ with supports $\supp(\dist)\subseteq[0, 1], \supp(\distNew)\subseteq[0, 1]$, the optimal objective value of program~$\ProGOne{\dist,\distNew}$ satisfies
\begin{align*}
    \OBJ{\ProGOne{\dist,\distNew}} = \OBJ{\ProGTwo{\dist,\distNew}}
\end{align*}
where $\OBJ{\ProGTwo{\dist,\distNew}}$ is the optimal objective value of the following program
\begin{align}
    \label{eq:primal perfect intermiddle}
    \arraycolsep=5.4pt\def\arraystretch{1}
    \tag{$\textsc{P}_2[\dist,\distNew]$}
    &\begin{array}{rlll}
    \sup
    \limits_{\alpha_1, \alpha_2, \gamma:[0,1]\rightarrow\reals} 
    ~ &
    \displaystyle 
    \int_0^1
    \left(\alpha_1(\prediction) +\gamma(\prediction) \cdot \prediction\right) \cdot \dist(\prediction)
    \, \dd  \prediction
    + 
    \int_0^1 \alpha_2(\predictionNew) 
    \distNew(\predictionNew) \, \dd  \predictionNew 
    \quad & \text{s.t.} &
    \vspace{1mm}
    \\
    & 
    \displaystyle 
    \alpha_1(\prediction) + \alpha_2(\predictionNew) \le 
    - \gamma(\prediction) \cdot \predictionNew,
    &  \prediction, \predictionNew\in[0, 1]
    \vspace{1mm}
    \\
    & 
    \displaystyle 
    \gamma(\prediction) \in [-1, 1],
    &  \prediction \in[0, 1]
    \vspace{1mm}
    \end{array}
\end{align}  
\end{lemma}
\begin{proof}
We first argue that $\OBJ{\text{\ref{eq:primal perfect intermiddle}}} \le \OBJ{\text{\ref{eq:primal perfect}}}$.
To see this, for any feasible solution to program~\ref{eq:primal perfect intermiddle}, we define auxiliary function $\auxfunc:[0,1]\rightarrow\reals$ as
\begin{align*}
    \forall \predictionNew\in[0,1]:\qquad \auxfunc(\predictionNew) \triangleq \inf\nolimits_{\prediction\in[0,1]} -\gamma(\prediction) \cdot \predictionNew
    - \alpha_1(\prediction)
\end{align*}
Notice that function $\auxfunc$ is the highest lower concave envelope over affine functions $\{l_\prediction(\cdot)\}_{\prediction\in[0, 1]}$ where each $l_\prediction(\predictionNew) \triangleq -\gamma(\prediction) \cdot \predictionNew - \alpha_1(\prediction)$.
This implies that function $\auxfunc$ must be 1-Lipschitz concave, since the constraint that $|\gamma(\cdot)|\le 1$ in program~\ref{eq:primal perfect intermiddle}.

With the above definition of auxiliary function $\auxfunc$, we can upper bound the objective function in program~\ref{eq:primal perfect intermiddle} as follows:
\begin{align*}
    \int_0^1 \left(\alpha_1(\prediction) + \gamma(\prediction) \cdot \prediction \right) \cdot \dist(\prediction)
    \, \dd \prediction
    +
    \int_0^1 \alpha_2(\predictionNew) \cdot \distNew(\predictionNew)
    \, \dd \predictionNew 
    & \le 
    \int_0^1 \left(\alpha_1(\prediction) + \gamma(\prediction) \cdot \prediction
    \right) \cdot \dist(\prediction)
    \, \dd \prediction
    +
    \int_0^1 \auxfunc(\predictionNew) \cdot \distNew(\predictionNew)
    \, \dd \predictionNew \\
    & \le 
    \int_0^1 - \auxfunc(\prediction) \cdot \dist(\prediction)
    \, \dd \prediction
    +
    \int_0^1 \auxfunc(\predictionNew) \cdot \distNew(\predictionNew)
    \, \dd \predictionNew\\
    & =
    \int_0^1 - \auxfunc(\prediction) \cdot \left(\dist(\prediction)
    - \distNew(\prediction)\right)
    \, \dd \prediction
\end{align*}
Here, the first inequality holds due to $\alpha_2(\predictionNew) \le V(\predictionNew)$ for all $\predictionNew\in[0, 1]$, which is implied by the construction of auxiliary function $\auxfunc$ and the first constraint in program~\ref{eq:primal perfect intermiddle}, and the second inequality holds due to the construction of auxiliary function $\auxfunc$.
Since $-\auxfunc$ is $1$-Lipschitz convex, it is a feasible solution in program \ref{eq:primal perfect}. Thus, we have shown $\OBJ{\text{\ref{eq:primal perfect intermiddle}}} \le \OBJ{\text{\ref{eq:primal perfect}}}$.

We next show that $\OBJ{\text{\ref{eq:primal perfect intermiddle}}} \ge \OBJ{\text{\ref{eq:primal perfect}}}$.
Let $\indirectU$ be any feasible solution to program \ref{eq:primal perfect}. 
Given $\indirectU$, we next construct a feasible solution $\alpha_1, \alpha_2, \gamma$ to program \ref{eq:primal perfect intermiddle}
Since $\indirectU$ is $1$-Lipschitz convex, there must exist a measurable slope function $\gamma: [0, 1] \rightarrow[-1, 1]$ such that for every $\prediction, \predictionNew\in[0, 1]$, we have
\begin{align}
    \label{ineq:convexity}
    \indirectU(\prediction) \ge \indirectU(\predictionNew) + \gamma(\predictionNew) \cdot (\prediction - \predictionNew)~.
\end{align}
Indeed, such slope function $\gamma$ can be chosen by picking any sub-gradient $\gamma(\prediction)\in\partial \indirectU(\prediction)$ (which exists a.e.\ for convex function $\indirectU$). 
The $1$-Lipschitzness of function $\indirectU$ ensures the second constraint that $|\gamma(\cdot)| \le 1$ in program~\ref{eq:primal perfect intermiddle} is satisfied.
Given $\indirectU$ and a choice of $\gamma$, we construct $\alpha_1$ and $\alpha_2$ as follows: 
\begin{align*}
\forall \prediction\in[0, 1]:\qquad
    \alpha_1(\prediction) \triangleq \indirectU(\prediction) - \gamma(\prediction) \cdot \prediction
    \;\;
    \mbox{and}
    \;\;
    \alpha_2(\prediction) \triangleq -\indirectU(\prediction)
\end{align*}
By construction, inequality~\eqref{ineq:convexity} can be rewritten as
\begin{align*}
    -\alpha_2(\prediction) \geq \alpha_1(\predictionNew) + \gamma(\predictionNew)\cdot \prediction
    \;\;
    \Longleftrightarrow
    \;\;
    \alpha_1(\predictionNew) + \alpha(\prediction) \leq 
    -\gamma(\predictionNew)\cdot \prediction
\end{align*}
and thus the first constraint in program~\ref{eq:primal perfect intermiddle} is satisfied.
Thus, $\alpha_1, \alpha_2, \gamma$ is a feasible solution to program \ref{eq:primal perfect intermiddle}. 
Moreover, the objective value of program \ref{eq:primal perfect intermiddle} under solution $\alpha_1, \alpha_2, \gamma$ satisfies that
\begin{align*}
    & 
    \int_0^1 \left(
    \alpha_1(\prediction) + \gamma(\prediction) \cdot \prediction \right) \cdot \dist(\prediction)
    \, \dd \prediction
    +
    \int_0^1 \alpha_2(\predictionNew) \cdot \distNew(\predictionNew)
    \, \dd \predictionNew 
    \\
     = {}&
    \int_0^1 \left(
    \indirectU(\prediction) - \gamma(\prediction)\cdot \prediction + \gamma(\prediction) \cdot \prediction \right) \cdot \dist(\prediction)
    \, \dd \prediction
    +
    \int_0^1 -\indirectU(\predictionNew) \cdot \distNew(\predictionNew)
    \, \dd \predictionNew \\
     ={}& \int_0^1 \indirectU(\prediction)\cdot (\dist(\prediction) - \distNew(\prediction))\, \dd \prediction~.
\end{align*}
Thus, we have shown $\OBJ{\text{\ref{eq:primal perfect intermiddle}}} \ge \OBJ{\text{\ref{eq:primal perfect}}}$.

Putting the two pieces together, we complete the proof of \Cref{lem:primal intermiddle}.
\end{proof}

\begin{lemma}
\label{lem:last step}
For any two distributions $\dist, \distNew$ with supports $\supp(\dist)\subseteq[0, 1], \supp(\distNew)\subseteq[0, 1]$, the relaxed earth mover's distance $\REMD{\dist,\distNew}$ satisfies
\begin{align*}
    \REMD{\dist,\distNew} = \OBJ{\ProGTwo{\dist,\distNew}}
\end{align*}
where $\OBJ{\ProGTwo{\dist,\distNew}}$ is the optimal objective value of program~$\ProGTwo{\dist,\distNew}$.
\end{lemma}
\begin{proof}
We prove the lemma statement by the strong duality.
Recall the identity $\abs{x} = \max_{\gamma\in[-1, 1]} \gamma\cdot x$. 
Thus, given any feasible coupling $\miscali \in \couplingSpace(\dist, \distNew)$, we can reformulate the objective function in program \eqref{opt:EMD} in the definition of $\REMD{\cdot,\cdot}$ as follows:
\begin{align*}
    \int_0^1 \abs{\int_0^1 (\prediction-\predictionNew)\cdot \miscali(\prediction, \predictionNew) \, \dd \predictionNew }\, \dd \prediction
    & = \int_0^1 \abs{\prediction \cdot \dist(\prediction) -  \int_0^1 \predictionNew\cdot \miscali(\prediction, \predictionNew) \, \dd \predictionNew}  \, \dd \prediction \\ 
    & = \max_{\gamma:[0,1]\rightarrow[-1, 1]} 
    \int_0^1 \gamma(\prediction) \cdot \left(
    \prediction \cdot \dist(\prediction) -  \int_0^1 \predictionNew \cdot \miscali(\prediction, \predictionNew) \, \dd \predictionNew
    \right) \, \dd \prediction  \\
    & = 
    \max_{\gamma:[0,1]\rightarrow[-1, 1]} 
    \int_0^1 \gamma(\prediction) \cdot \prediction \cdot \dist(\prediction)
    \, \dd \prediction
    -
    \int_0^1 \gamma(\prediction) \cdot
    \left(
    \int_0^1 \predictionNew\cdot \miscali(\prediction, \predictionNew) \, \dd \predictionNew
    \right) \, \dd \prediction~.
\end{align*}
We next consider Lagrange multipliers $\alpha_1, \alpha_2:[0,1]\rightarrow\reals$, and $\beta:[0,1]\times[0,1]\rightarrow\reals_+$ for program~\eqref{opt:EMD}:
\begin{align*}
    \mathcal{L}
    & = 
    \int_0^1 \gamma(\prediction) \cdot \prediction \cdot \dist(\prediction)
    \, \dd \prediction
    -
    \int_0^1 \gamma(\prediction) \cdot
    \left(
    \int_0^1 \predictionNew\cdot \miscali(\prediction, \predictionNew) \, \dd \predictionNew
    \right) \, \dd \prediction 
    - \int_0^1\int_0^1 \beta(\prediction, \predictionNew)\cdot \miscali(\prediction, \predictionNew) 
    \, \dd \predictionNew  \dd \prediction \\
    & \quad 
    + 
    \int_0^1 \alpha_1(\prediction) \cdot\left(\dist(\prediction) - \int_0^1 \miscali(\prediction, \predictionNew) \, \dd \predictionNew\right)\, \dd \prediction
    + 
    \int_0^1 \alpha_2(\predictionNew) \cdot\left(\distNew(\predictionNew) - \int_0^1 \miscali(\prediction, \predictionNew) \, \dd \prediction\right)\, \dd \predictionNew \\
    & = 
    \int_{0}^1\int_0^1 \miscali(\prediction, \predictionNew) \cdot\left(
    - \beta(\prediction, \predictionNew) - \gamma(\prediction) \cdot \predictionNew 
    - \alpha_1(\prediction) 
    - \alpha_2(\predictionNew)\right)
    \,  \dd \predictionNew \dd \prediction\\
    & \quad 
    + 
    \int_0^1 \left(\gamma(\prediction) \cdot \prediction 
    + \alpha_1(\prediction)\right) \cdot \dist(\prediction)
    \, \dd \prediction
    +
    \int_0^1 \alpha_2(\predictionNew) \cdot \distNew(\predictionNew)
    \, \dd \predictionNew
\end{align*}
By the first-order condition, the optimal Lagrange multipliers $\beta^*, \alpha_1^*, \alpha_2^*, \gamma^*$ satisfy 
\begin{align*}
    - 
    \beta^*(\prediction, \predictionNew) - \gamma^*(\prediction) \cdot \predictionNew 
    - \alpha_1^*(\prediction) 
    - \alpha_2^*(\predictionNew) = 0
\end{align*}
for all $\prediction, \predictionNew\in[0, 1]$. 
In addition, under the optimal coupling $\miscali^*$, for any $\prediction, \predictionNew\in[0, 1]$, we have the following complementary slackness conditions: 
\begin{align*}
    \alpha_1^*(\prediction) 
    + \alpha_2^*(\predictionNew) 
    & = -\gamma^*(\prediction) \cdot \predictionNew, \quad \miscali^*(\prediction, \predictionNew) > 0 \\
    \alpha_1^*(\prediction) 
    + \alpha_2^*(\predictionNew) 
    & \le -\gamma^*(\prediction) \cdot \predictionNew, \quad \miscali^*(\prediction, \predictionNew) = 0~.
\end{align*}
With the above first-order conditions and complementary slackness, the Lagrange dual then becomes
\begin{align*}
    \mathcal{L} 
    = 
    \int_0^1 \left(\gamma^*(\prediction) \cdot \prediction 
    + \alpha_1^*(\prediction)\right) \cdot \dist(\prediction)
    \, \dd \prediction
    +
    \int_0^1 \alpha_2^*(\predictionNew) \cdot \distNew(\predictionNew)
    \, \dd \predictionNew~, 
\end{align*}
which is essentially the objective function in program \ref{eq:primal perfect intermiddle} and its constraints are the complementary slackness conditions that we derive above. This completes the proof of \Cref{lem:last step}.
\end{proof}

\begin{proof}[Proof of \Cref{thm:strong-duality-perfectly-calibrated}]
    Combining \Cref{lem:simplify,lem:primal intermiddle,lem:last step} finishes the proof of \Cref{thm:strong-duality-perfectly-calibrated}.
\end{proof}

\subsection{Structure Characterizations of \texorpdfstring{$\REMD{\cdot, \cdot}$}{REMD[., .]}}
\label{subsec:structural characterzation perfecly calibrated}

In this section, we provide additional implications and intuitions behind our proposed relaxed earth mover's distance $\REMD{\cdot, \cdot}$.

\xhdr{SCDF representation}  
Our proposed relaxed earth mover's distance $\REMD{\cdot, \cdot}$ admits a succinct closed-form expression in terms of the super-cumulative distribution functions (SCDF).

\begin{proposition}[SCDF representation]
\label{prop:idf}
Given any two distributions $\dist, \distNew$ with supports $\supp(\dist)\subseteq[0, 1], \supp(\distNew)\subseteq[0, 1]$ and the same mean $\expect[\prediction\sim\dist]{\prediction} = \expect[\prediction\sim\distNew]{\prediction}$,
the relaxed earth mover's distance $\REMD{\dist, \distNew}$ satisfies
\begin{align*}
    \REMD{\dist, \distNew} = 2\cdot \max\nolimits_{t\in[0, 1]}~
    \distSCDF(t) - \distSCDFNew(t)
\end{align*}
where $\distSCDF(t) \triangleq \int_0^t \distCDF(\prediction)\,\dd\prediction$ and $\distSCDFNew(t) \triangleq \int_0^t \distCDFNew(\prediction)\,\dd\prediction$ are the SCDFs of distributions $\dist,\distNew$, respectively.
\end{proposition}
Combining the proposition above with \Cref{thm:strong-duality-perfectly-calibrated}, we immediately obtain the following corollary for the informativeness gap between perfectly predictors. See \Cref{fig:CCDF} for an illustration.
\begin{corollary}
\label{cor:idf predictor}
Given any two perfectly calibrated predictors $\predictor,\predictornew$, informativeness gap $\inforGap{\predictor,\predictornew}$ satisfies
\begin{align*}
    \inforGap{\predictor,\predictornew} = 2 \cdot 
    \max\nolimits_{t\in[0, 1]}~ \SCDF(t) - \SCDFNew(t)
\end{align*}
where $\SCDF$ and $\SCDFNew$ are the SCDF of prediction distributions $\PDF$ and $\PDFNew$, respectively.
\end{corollary}
\begin{proof}
    For perfectly calibrated predictors $\predictor$ and $\predictornew$, their prediction distributions has the same mean and support contained by $[0,1]$. Hence, invoking \Cref{thm:strong-duality-perfectly-calibrated} and \Cref{prop:idf} finishes the proof of \Cref{cor:idf predictor}.
\end{proof}

\begin{figure}[H]
  \centering

\begin{tikzpicture}[>=stealth,baseline=(current axis.south)]
\pgfmathsetmacro{\tstar}{0.74}
  \begin{axis}[
    width=10cm, height=6cm,
    axis lines=left,
    axis line shift=0pt,
    xmin=0, xmax=1, ymin=-0.01, ymax=0.65,
xtick={0.15,0.26,0.544},
    xticklabels={,,,}, 
    ytick=\empty, 
    clip=false,  
  ]
\addplot[gray!90!white, line width=1.3mm,opacity=0.6] coordinates {
      (0.00,   0.00)   (0.15,   0.00)   (0.544,  0.034)  (1.00,   0.61)   };
    
    \addplot[blue, line width=1mm,opacity=0.6] coordinates {
      (0.00,   0.00)   (0.26,   0.00)   (0.74,   0.13)   (1.00,   0.61)   };

    \addplot[black, dashed, line width=0.7mm,opacity=0.6] coordinates {
      (0.00,   0.00)   (0.15,   0.00)   (0.311438, 0.013931)   (0.587249, 0.08863)
      (1, 0.61) };

       \draw[dashed, thin, gray]
      (axis cs:\tstar,0) -- (axis cs:\tstar,{ 0.281579});

      \draw[thick, black]
      (axis cs:\tstar,0.13) -- (axis cs:\tstar,{0.281579});

      \draw[dashed, ->, thick, black] 
      (axis cs:0.73,0.2) 
      -- 
      (axis cs:0.6,0.3)
      node[anchor=south east,font=\fontsize{10pt}{11pt}\selectfont] at (axis cs:0.73,0.3) {$0.5\inforGap{\predictor,\predictornew}$};

      \node[anchor=north] at (axis cs:\tstar+0.04,0.07) {$t^*$};

  \end{axis}
\end{tikzpicture}

   \caption{
  The gray solid line is $\SCDF$ and blue solid line is $\SCDFNew$ where the predictors $\predictor$ and $\predictornew$ are not  dominated (in the Blackwell's sense) with each other.
  The black dashed line is $S_h$ where $h = \predictor\vee \predictornew$ is the join of the predictors $\predictor$ and $\predictornew$, 
  and $t^* =  \argmax\nolimits_{t\in[0, 1]}~ \SCDF(t) - \SCDFNew(t)$.
  The informativeness gap $0.5\inforGap{\predictor, \predictornew}$ is exactly the height of vertical black solid line.}
  \label{fig:CCDF}
\end{figure}

We remark that \Cref{cor:idf predictor} together with \Cref{prop:prelim:informativeness gap captures blackwell informativeness} recovers the SCDF representation of the Blackwell informativeness (\Cref{lem:Blackwell order:SCDF definition}).
We now prove \Cref{prop:idf}.
\begin{proof}[Proof of \Cref{prop:idf}]
    Invoking \Cref{lem:primal intermiddle,lem:last step}, we obtain the identity between $\REMD{\dist,\distNew}$ and the optimal objective of program~$\ProGOne{\dist,\distNew}$, i.e., $\REMD{\dist,\distNew} = \OBJ{\ProGOne{\dist,\distNew}}$.

    Define auxiliary function $\omega_{\dist,\distNew}:[0,1]\rightarrow\reals$ as 
    \begin{align*}
        \forall t\in[0,1]:\qquad
        \omega_{\dist,\distNew}(t) \triangleq \displaystyle\int_0^t 
        \dist(\prediction) - \distNew(\prediction)\,\dd \prediction
    \end{align*}
    The objective function of program~$\ProGOne{\dist,\distNew}$ can be reformulated as follows:
    \begin{align*}
        \int_0^1
        \indirectU(\prediction) \cdot \left(\dist(\prediction) - \distNew(\prediction)\right)~\dd  \prediction 
        & = \left[\indirectU(\prediction)\cdot \omega_{\dist,\distNew}(\prediction)\right]_0^1 - \int_0^1 \omega_{\dist,\distNew}(\prediction) \, \dd \indirectU(\prediction) 
         = 
        - \int_0^1 \omega_{\dist,\distNew}(\prediction)\cdot  \indirectUDeriv(\prediction)\, \dd \prediction
    \end{align*}
    where the first equality holds due to integration by parts, and the second equality considers derivative $\indirectUDeriv(\prediction) \triangleq \indirectU'(\prediction)$ which exists almost everywhere and uses the fact that $\omega_{\dist,\distNew}(0) = \omega_{\dist,\distNew}(1) = 0$. 
Consequently, program~$\ProGOne{\dist,\distNew}$ is equivalent to the following program:
    \begin{align}
        \label{eq:primal equivalent}
        \arraycolsep=5.4pt\def\arraystretch{1}
        \tag{$\textsc{P}_1\primed[\dist,\distNew]$}
        &\begin{array}{rlll}
        \inf
        \limits_{\indirectUDeriv:[0,1]\rightarrow\reals} 
        ~ &
        \displaystyle 
        \int_0^1
        \indirectUDeriv(\prediction) \cdot \omega_{\dist,\distNew}(\prediction)~\dd  \prediction 
        \quad & \text{s.t.} &
        \vspace{1mm}
        \\
        & 
        \displaystyle 
        \indirectUDeriv(\prediction)\in[-1,1], 
        &  \prediction\in[0, 1]
        \vspace{1mm}
        \\
        & 
        \displaystyle 
        \indirectUDeriv(\prediction) \le \indirectUDeriv(\predictionNew) 
        &  \prediction\in[0, 1], \predictionNew\in[\prediction, 1]
        \vspace{1mm}
        \end{array}
    \end{align}
    where the constraints is the sufficient and necessary condition for the derivative $\indirectUDeriv$ of 1-Lipschitz convex interim utility function $\indirectUDeriv$.

    According to \cite{S-06,B-15}, the extreme points of the uniformly bounded, non-decreasing functions are step functions with only one jump. Thus, the optimal solution $\indirectUDeriv^*$ to the (infinite-dimensional) linear program \ref{eq:primal equivalent} must be a step function where 
    \begin{align*}
        \indirectUDeriv^*(\prediction) = -1 \cdot \indicator{\prediction\le t^*} + 1 \cdot\indicator{\prediction> t^*}
    \end{align*}
    for some threshold $t^*\in[0, 1]$.

    With the above observation, to solve program~$\ProGOne{\dist,\distNew}$, it suffices to optimize over all V-shaped functions across different kinks $t\in[0, 1]$.
    In particular, we can write the objective in program~$\ProGOne{\dist,\distNew}$ under a V-shaped function with kink $t\in[0, 1]$ as follows: 
    \begin{align*}
        &\int_0^t -(\prediction-t) \cdot (\dist(\prediction) - \distNew(\prediction))\, \dd \prediction 
        + 
        \int_{t}^1 (\prediction-t) \cdot (\dist(\prediction) - \distNew(\prediction))\, \dd \prediction
        \\
        = {} &
        \int_0^t -(\prediction-t) \cdot (\dist(\prediction) - \distNew(\prediction))\, \dd \prediction 
        + 
        \int_{0}^1 (\prediction-t) \cdot (\dist(\prediction) - \distNew(\prediction))\, \dd \prediction
        -
        \int_{0}^t (\prediction-t) \cdot (\dist(\prediction) - \distNew(\prediction))\, \dd \prediction
        \\
        = {}&
        \int_0^t -2(\prediction-t) \cdot (\dist(\prediction) - \distNew(\prediction))\, \dd \prediction~.
    \end{align*}
    Here the second equality holds since $\int_{0}^1 (\prediction-t) \cdot (\dist(\prediction) - \distNew(\prediction))\, \dd \prediction = 0$, which is implied by the statement assumption that distributions $\dist,\distNew$ are valid distribution and have the same mean.

    Observe that 
    \begin{align*}
        \frac{\partial}{\partial t}\int_0^t -2(\prediction-t) \cdot (\dist(\prediction) - \distNew(\prediction))\, \dd \prediction
        &=
        \int_0^t 2 \cdot (\dist(\prediction) - \distNew(\prediction))\,\dd\prediction
        -
        2(t - t) (\dist(t) - \distNew(t)) 
        \\
        &= 
        2 (\distCDF(t) - \distCDFNew(t))
    \end{align*}
    Thus, the optimal objective value of program~$\ProGOne{\dist,\distNew}$ can be further reformulated as 
    \begin{align*}
        \max\nolimits_{t\in[0, 1]}~
        \int_0^t 2 (\distCDF(\prediction) - \distCDFNew(\prediction))\,\dd \prediction = 
        2\cdot \max\nolimits_{t\in[0, 1]}~ \distSCDF(t) - \distSCDFNew(t)
    \end{align*}
    where the equality holds due to SCDF definition. This completes the proof of \Cref{prop:idf}.
\end{proof}

\xhdr{Lattice representation} Given any two distributions $\dist,\distNew$ with the same mean and supports contained in $[0,1]$, define their \emph{join} $\join\triangleq \dist\vee \distNew$ as the unique distribution whose SCDF $\joinSCDF$ satisfies 
\begin{align*}
    \forall \prediction\in[0,1]:\qquad 
    \joinSCDF(\prediction) = \max\{\distSCDF(\prediction),\distSCDFNew(\prediction)\}
\end{align*}
As a sanity check, since both distributions' SCDF $\distSCDF,\distSCDFNew$ are convex, SCDF $\joinSCDF$ is also convex and thus $\join$ is a valid distribution and has the same mean as distributions $\dist,\distNew$. As \Cref{lem:Blackwell order:lattice structure} states, the space of information structures (perfectly calibrated predictors) and the Blackwell order over it forms a lattice. Thus, by viewing distributions $\dist,\distNew$ and $\join$ as information structures, join $\join$ is the least Blackwell informative information structure (perfectly calibrated predictor) that Blackwell-dominants both information structures $\dist$ and $\distNew$.

We now establish the property of our proposed relaxed earth mover's distance $\REMD{\dist,\distNew}$ and their corresponding join $\join=\dist\vee \distNew$.

\begin{proposition}[Lattice representation]
\label{prop:join equivalency}
    Given any two distributions $\dist, \distNew$ with supports $\supp(\dist)\subseteq[0, 1], \supp(\distNew)\subseteq[0, 1]$ and the same mean $\expect[\prediction\sim\dist]{\prediction} = \expect[\prediction\sim\distNew]{\prediction}$,
    the relaxed earth mover's distance $\REMD{\dist, \distNew}$ satisfies
    \begin{align*}
        \REMD{\dist,\distNew} = \REMD{\join,\distNew}
    \end{align*}
    where $\join=\dist\vee \distNew$ is the join of distributions $\dist,\distNew$ under the Blackwell order.
\end{proposition}
\begin{proof}
    We rewrite $\REMD{\dist,\distNew}$ as 
    \begin{align*}
        \REMD{\dist,\distNew} &=  2\cdot \max\nolimits_{t\in[0, 1]}~
        \distSCDF(t) - \distSCDFNew(t)
        = 
        2\cdot 
        \max\nolimits_{t\in[0, 1]}~
        \max\{\distSCDF(t),\distSCDFNew(t)\} - \distSCDFNew(t)
        \\
        & = 
        2\cdot 
        \max\nolimits_{t\in[0, 1]}~
        \joinSCDF(t) - \distSCDFNew(t)
        =
        \REMD{\join,\distNew} 
    \end{align*}
    where the first and last equality holds due to \Cref{prop:idf}, the second equality holds since $\distSCDF(0) = \distSCDFNew(0)$ and algebra, and the third equality holds due to the definition of join $\join$. This completes the proof of \Cref{prop:join equivalency}.
\end{proof}
Combining the proposition above with \Cref{thm:strong-duality-perfectly-calibrated}, we immediately obtain the following corollary for the informativeness gap between perfectly predictors.
\begin{corollary}
\label{cor:join equivalency predictor}
Given any two perfectly calibrated predictors $\predictor,\predictornew$, informativeness gap $\inforGap{\predictor,\predictornew}$ satisfies
\begin{align*}
    \inforGap{\predictor,\predictornew} = 
    \inforGap{h,\predictornew}
\end{align*}
where $h$ is the perfectly calibrated predictor with prediction distribution $f_h = \PDF\vee\PDFNew$. 
\end{corollary}
With a slight abuse of notation, we also denote by $\predictor \vee \predictornew$ the predictor $h$ in \Cref{cor:join equivalency predictor}, referring to it as the \emph{join of predictors $\predictor$ and $\predictornew$}.

We make the following interpretation from \Cref{cor:join equivalency predictor}: first, suppose that the perfectly calibrated predictor $\predictornew$ is more Blackwell informative than the perfectly calibrated predictor $\predictor$. By definition, their join is exactly the predictor $\predictornew$, i.e., $\predictor \vee \predictornew = \predictornew$. Hence, \Cref{cor:join equivalency predictor} immediately implies that $\inforGap{\predictor, \predictornew} = \inforGap{\predictornew, \predictornew} = 0$, which recovers \Cref{prop:prelim:informativeness gap captures blackwell informativeness}.

More importantly, \Cref{cor:join equivalency predictor} also provides insight into cases where perfectly calibrated predictors $\predictor$ and $\predictornew$ are not comparable under the Blackwell order. In particular, it suggests that the maximum payoff advantage of perfectly calibrated predictor $\predictor$ over perfectly calibrated predictor $\predictornew$ is always achieved in a decision task where predictor $\predictor$ performs as well as their join predictor (which is always more Blackwell informative than both). See \Cref{fig:join} for an illustration.

\begin{figure}[H]
  \centering

\begin{tikzpicture}[>=stealth, font=\Large, node distance=5cm]
\node (join)  [circle, draw, fill=black, minimum size=4mm, inner sep=2pt, label={[font=\fontsize{12pt}{12pt}\selectfont]above:{$h$}} ] {};
  
  \node (predictor) [circle, draw, fill=black, minimum size=4mm, inner sep=2pt, below left=of join,  label={[font=\fontsize{12pt}{12pt}\selectfont]below:{$\predictor$}}] {};
  
  \node (predictornew) [circle, draw, fill=black, minimum size=4mm, inner sep=2pt, below right=of join, label={[font=\fontsize{12pt}{12pt}\selectfont]below:{$\predictornew$}}] {};

\draw[->, -{Latex[length=3mm,width=2mm]}] (join) -- (predictor) node[midway,left, font=\fontsize{10pt}{10pt}\selectfont] {$\inforGap{h,\predictor}$};
  \draw[->, -{Latex[length=3mm,width=2mm]}] (join) -- (predictornew)
  node[midway,right, font=\fontsize{10pt}{10pt}\selectfont] {$\inforGap{h,\predictornew}$};;

  \draw[dashed,->, -{Latex[length=3mm,width=2mm]}]
    ([yshift=3pt]predictor.east) -- ([yshift=3pt]predictornew.west) node[midway,above,  font=\fontsize{10pt}{10pt}\selectfont] {$\inforGap{\predictor, \predictornew}$};
  \draw[dashed,->, -{Latex[length=3mm,width=2mm]}]
    ([yshift=-3pt]predictornew.west) -- ([yshift=-3pt]predictor.east) node[midway,below,  font=\fontsize{10pt}{10pt}\selectfont] {$\inforGap{\predictornew, \predictor}$};
\end{tikzpicture}   \caption{A graphic illustration of join $h \triangleq \predictor \vee \predictornew$ over two  calibrated predictors $\predictor, \predictornew$.
  The solid arrow implies a Blackwell's order between two predictors, while a dashed arrow not necessarily implies a Blackwell's order. According to \Cref{cor:join equivalency predictor}, we have that $\inforGap{\predictor, \predictornew} = \inforGap{h, \predictornew}$ and 
  $\inforGap{\predictornew, \predictor} = \inforGap{h, \predictor}$.}
  \label{fig:join}
\end{figure}

\section{The Informativeness of Miscalibrated Predictors}
\label{sec:misc}

In this section, we characterize the informativeness gap between any two (possibly miscalibrated) predictors. The core of our characterization is a generalized notion of $\REMD{\cdot, \cdot}$ to incorporate miscalibrated predictions, which we define it as follows.
\begin{definition}
\label{def:REMD miscali}
Given any two distributions $\dist, \distNew$ with supports $\supp(\dist)\subseteq[0, 1], \supp(\distNew)\subseteq[0, 1]$ and two functions $\truePredic_1,\truePredic_2:[0,1]\rightarrow[0,1]$,
the \emph{informativeness measure} $\earthDistFlow{\dist, \distNew,\truePredic_1,\truePredic_2}$ is defined as the optimal objective value of the following optimization program:
\begin{align}
    \label{opt:flow miscali}
    \inf\limits_{\miscali\in \feasibleflowSet(\dist, \distNew)} ~  & 
    \int_0^1 \abs{\int_0^1 \miscali(\prediction, \predictionNew) \cdot (\prediction-\predictionNew)\,\dd \predictionNew 
    + (\truePredic_1(\prediction) - \prediction) \cdot \dist(\prediction)
    - (\truePredic_2(\prediction) - \prediction) \cdot \distNew(\prediction) }  \, \dd \prediction
\end{align}
where 
\begin{align*}
    \feasibleflowSet(\dist, \distNew) \triangleq \left\{\miscali\in\Delta([0, 1]\times[0,1]): \dist(\prediction) - \distNew(\prediction) - \int_0^1 \miscali(\prediction, \predictionNew)\,\dd \predictionNew + \int_0^1 \miscali(\predictionNew, \prediction)\, \dd \predictionNew = 0,  ~~\prediction\in[0, 1]\right\}
\end{align*}
is referred to as the \emph{flow set}.
\end{definition}
With the above definition, the main results in this section are the following 
(i) dual characterization of $\inforGap{\predictor, \predictornew}$; 
(ii) Calibration-adjusted SCDF (henceforth {\generalizedSCDF}) representation of $\inforGap{\predictor, \predictornew}$, 
when predictors $\predictor, \predictornew$ are possibly miscalibrated.
\begin{theorem}
\label{thm:strong-duality-miscalibrated}
For any two (possibly miscalibrated) predictors $\predictor, \predictornew$, their informativeness gap satisfies
\begin{align*}
    \inforGap{\predictor, \predictornew}
    & = 
    \earthDistFlow{\PDF, \PDFNew, \truePredic_\predictor,\truePredic_\predictornew}\\
    &= 2\cdot \max_{t\in [0, 1]} ~
    \left(\scdf_\predictor(t) + \int_0^t (\prediction - \truePredic_\predictor(\prediction))\cdot \PDF(\prediction)\, \dd \prediction \right)
    - 
    \left(\scdf_\predictornew(t) + \int_0^t (\prediction - \truePredic_\predictornew(\prediction)) \cdot \PDFNew(\prediction)\, \dd \prediction\right)
\end{align*}
where $\PDF,\PDFNew$ (resp.\ $\truePredic_\predictor,\truePredic_\predictornew$) are the prediction distributions (resp.\ true expected outcome functions) of predictors $\predictor,\predictornew$, respectively;
and $\scdf_\predictor(t) \triangleq \int_0^t \PDF(\prediction)\,\dd\prediction$ and $\scdf_\predictornew(t) \triangleq \int_0^t \PDFNew(\prediction)\,\dd\prediction$ are the SCDFs of distributions $\PDF,\PDFNew$, respectively.
\end{theorem}
Below, we make the several remarks regarding \Cref{thm:strong-duality-miscalibrated}.

First, unlike our proposed relaxed earth mover's distance $\REMD{\cdot,\cdot}$, which is defined over distributions, the new measure $\earthDistFlow{\cdot,\cdot,\cdot,\cdot}$ is defined over distributions and two additional functions. This distinction is intuitive, as a miscalibrated predictor $\predictor$ is jointly determined by its prediction distribution $\PDF$ and its true expected outcome function $\truePredic_\predictor$. Therefore, a measure that only takes prediction distributions as input does not collect enough information to capture the decision payoff for general miscalibrated predictors.

Second, the flow set $\feasibleflowSet(\cdot,\cdot)$ in the definition above can be viewed as the flow conservation constraint over the 2-dimensional space $[0,1] \times [0,1]$. It is conceptually similar to the flow constraint in the mechanism design literature \citep{CKM-13,CDW-16}. In particular, both our flow constraint and the flow constraint in mechanism design can be interpreted as the dual constraint for interim utility under the best response (or equilibrium).\footnote{We conjecture that there might be a deeper connection between our model and the mechanism design model, which we leave for future work.}

Third, inspired the by second representation in \Cref{thm:strong-duality-miscalibrated}, we define the \emph{calibration-adjusted SCDF}---a corrected 
version of the SCDF that accounts for the discrepancy between predicted and true probabilities as follows:
\begin{definition}[CA-SCDF]
\label{def:ca-scdf}
    Given a (possibly miscalibrated) predictor $\predictor$ its \emph{calibration-adjusted SCDF (CA-SCDF)} is defined as 
    \begin{align*}
        \scdf_\predictor(t) + \int_0^t \bigl(\prediction - \truePredic_\predictor(\prediction)\bigr) 
        \cdot \PDF(\prediction)\, \dd\prediction
    \end{align*}
\end{definition}
As a sanity check, when the predictor $\predictor$ is perfectly calibrated, we have 
$\truePredic_\predictor(\prediction) \equiv \prediction$, and the CA-SCDF  reduces to the SCDF of the prediction distribution $\PDF$.

The formal proof of \Cref{thm:strong-duality-miscalibrated} is presented in \Cref{subsec:strong duality miscalibrated proof}. 
We next first compare our characterization in \Cref{thm:strong-duality-miscalibrated} with the characterization that we established in \Cref{thm:strong-duality-perfectly-calibrated} for perfectly calibrated predictors, and then
discuss the connection between our proposed informativeness gap and other calibration measure in \Cref{subsec:relation to other measures}.

\subsection{Comparison with the Characterization of Perfectly Calibrated Predictors}
In this section, we discuss the comparisons between \Cref{thm:strong-duality-miscalibrated} and \Cref{thm:strong-duality-perfectly-calibrated}.

We first observe that when the predictors $\predictor$ and $\predictornew$ are both perfectly calibrated, the objective function in program \eqref{opt:flow miscali} is the same as that in program \eqref{opt:EMD}, while the flow set $\feasibleflowSet(\PDF, \PDFNew)$ is a strict superset of the coupling set $\couplingSpace(\PDF,\PDFNew)$. \Cref{thm:strong-duality-miscalibrated} implies that both $\earthDistFlow{\PDF, \PDFNew, \truePredic_\predictor,\truePredic_\predictornew}$ and $\REMD{\PDF, \PDFNew}$ achieve the identical optimal objective value when predictors $\predictor$ and $\predictornew$ are perfectly calibrated. Alternatively, we can summarize this observation as the following characterization for $\REMD{\cdot,\cdot}$:
\begin{corollary}[Flow representation]
\label{cor:flow perfectly calibrated}
Given any two distributions $\dist, \distNew$ with supports $\supp(\dist)\subseteq[0, 1], \supp(\distNew)\subseteq[0, 1]$ and the same mean $\expect[\prediction\sim\dist]{\prediction} = \expect[\prediction\sim\distNew]{\prediction}$,
the relaxed earth mover's distance $\REMD{\dist, \distNew}$ satisfies
\begin{align*}
    \REMD{\dist,\distNew} 
    = 
    \inf\nolimits_{\miscali\in \feasibleflowSet(\dist, \distNew) } \int_0^1 \abs{\int_0^1 \miscali(\prediction, \predictionNew)\cdot (\prediction-\predictionNew) \, \dd \predictionNew 
    }  \, \dd \prediction~.
\end{align*}
where $\feasibleflowSet(\dist, \distNew)$ is a strict superset of $\couplingSpace(\dist,\distNew)$ used in the definition of $\REMD{\cdot,\cdot}$ (\Cref{def:REMD}).
\end{corollary}
\begin{proof}
Since both distributions have supports contained in $[0, 1]$ and have the same mean, they can be considered as the prediction distributions of two perfectly calibrated predictors, denoted by $\predictor,\predictornew$. Invoking \Cref{thm:strong-duality-perfectly-calibrated,thm:strong-duality-miscalibrated}, we obtain $\REMD{\dist,\distNew} = \inforGap{\predictor,\predictornew} = \earthDistFlow{\dist,\distNew,\truePredic_{C},\truePredic_C}$ where function $\truePredic_C(\prediction)\triangleq \prediction$ for all $\prediction\in[0,1]$. This completes the proof of \Cref{cor:flow perfectly calibrated}.
\end{proof}

Finally, with \Cref{cor:flow perfectly calibrated} in mind, one tempting thought is whether directly optimizing over the coupling set $\couplingSpace(\PDF, \PDFNew)$ in $\earthDistFlow{\PDF, \PDFNew, \truePredic_{\predictor}, \truePredic_{\predictornew}}$ would suffice---or at least yield a good approximation---of the informativeness gap $\inforGap{\predictor, \predictornew}$.
However, this is not the case: there exist predictors $\predictor,\predictornew$ for which the gap between optimizing over $\couplingSpace(\PDF, \PDFNew)$ in $\earthDistFlow{\PDF, \PDFNew, \truePredic_{\predictor}, \truePredic_{\predictornew}}$ and $\inforGap{\predictor, \predictornew}$ can be arbitrarily large.

\begin{example}
\label{ex:unbounded gap miscalibrated predictor measure}
Consider perfectly calibrated predictor $\predictor$ which outputs prediction $\prediction = 0$ and $\prediction = 1$ both with half probability, and
miscalibrated predictor $\predictornew$ with prediction distribution
\begin{align*}
    \PDFNew(\prediction) \triangleq  
    \begin{cases}
        0, & \prediction\in[0, \varepsilon] \\
        \displaystyle 
        \frac{\varepsilon}{2(1-2\varepsilon)} \cdot\frac{1}{\prediction^2}
        & \prediction\in(\varepsilon, 0.5] \\
        \displaystyle 
        \frac{\varepsilon}{2(1-2\varepsilon)} \cdot\frac{1}{(1-\prediction)^2}
        & \prediction\in[0.5, 1-\varepsilon) \\
        0, & \prediction \in[1-\varepsilon,1]\\
    \end{cases}
\end{align*}
and true expected outcome function $\truePredic_\predictornew(\prediction) = \indicator{\prediction\ge 0.5}$.
\end{example}

\begin{restatable}{proposition}{propBadApproxMarg}
\label{prop:bad approx of marg}
Given any two distributions $\dist, \distNew$ with supports $\supp(\dist)\subseteq[0, 1], \supp(\distNew)\subseteq[0, 1]$ and two functions $\truePredic_1,\truePredic_2:[0,1]\rightarrow[0,1]$, define a generalized version of relaxed EMD, denoted by $\REMDNew{\dist, \distNew, \truePredic_1, \truePredic_2}$ as the optimal objective value of the following optimization program:
\begin{align*}
    \inf\limits_{\miscali\in \couplingSpace(\dist, \distNew) } \int_0^1 \abs{\int_0^1 \miscali(\prediction, \predictionNew)\cdot (\prediction-\predictionNew) \, \dd \predictionNew 
    + (\truePredic_1(\prediction) - \prediction) \cdot \dist(\prediction)
    - (\truePredic_2(\prediction) - \prediction) \cdot \distNew(\prediction)
    }  \, \dd \prediction~.
\end{align*}
where decision variables $\miscali$ is optimized over coupling set $\couplingSpace(\dist,\distNew)$.

For every $\varepsilon \in (0, 0.5]$, in the setting of \Cref{ex:unbounded gap miscalibrated predictor measure} where predictor $\predictor$ is perfectly calibrated while predictor $\predictornew$ is miscalibrated, it satisfies 
\begin{align*}
    \inforGap{\predictor, \predictornew} = \Theta(\varepsilon)
    \;\;
    \mbox{and}
    \;\;
    \REMDNew{\PDF, \PDFNew,\truePredic_\predictor,\truePredic_\predictornew}
    =
    \Omega\left(\varepsilon\ln\frac{1}{\varepsilon}\right)
\end{align*}  
and thus
\begin{align*}
     \frac{\REMDNew{\PDF, \PDFNew,\truePredic_\predictor,\truePredic_\predictornew}}{\inforGap{\predictor, \predictornew}}  = \Omega\left(\ln \frac{1}{\varepsilon}\right)~.
\end{align*}
\end{restatable} 
We remark that the prediction distribution $\PDFNew$ of the miscalibrated predictor $\predictornew$, constructed in \Cref{ex:unbounded gap miscalibrated predictor measure}, is essentially the (truncated) \emph{equal-revenue distribution}, which is widely used in the mechanism design literature. In addition, for this particular example, the closed-form expressions of $\inforGap{\predictor, \predictornew}$ and $\REMDNew{\PDF, \PDFNew, \truePredic_\predictor, \truePredic_\predictornew}$ are essentially the revenue and welfare in the context of mechanism design, whose ratio is known to be unbounded \citep{har-16}. Before proving \Cref{prop:bad approx of marg}, we first present the following lemma, which generalizes \Cref{prop:idf} (essentially from perfectly calibrated predictors to miscalibrated predictors). Its proof is similar to the one for \Cref{prop:idf} and can be found in \Cref{apx:idf miscali proof}.

\subsection{On the Relation to Other Calibration Measures}
\label{subsec:relation to other measures}

Our proposed informativeness gap $\inforGap{\predictor, \predictornew}$ has connections to the existing calibration measures, such as the Expected Calibration Error (ECE) \citep{FV-98}, and U-Calibration (UCal) \citep{KLST-23}, calibration decision loss (CDL) \citep{HW-24}, distance to calibration \citep{BGHN-23}.
We discuss and summarize these connections below.

\xhdr{Subsuming existing decision-theoretic calibration measures}
We first show that the informativeness gap $\inforGap{\predictor, \predictornew}$ subsumes existing decision-theoretic calibration measures.
\begin{definition}[U-Calibration (UCal), adopted from \citealp{KLST-23}]
Given any (possibly miscalibrated) predictor $\predictor$, its \emph{U-Calibration}, denoted by $\UCal{\predictor}$, is defined as follows:
\begin{align*}
    \UCal{\predictor} \triangleq 
    \sup\nolimits_{(\actionSpace, \agentU)\in\agentUClass} ~ 
    \displaystyle 
    \agentExpPay[{(\actionSpace, \agentU)}]{\predictor^{\NI}} - 
    \agentExpPay[{(\actionSpace, \agentU)}]{\predictor}~,
\end{align*}
where $\predictor^{\NI}$ denotes the perfectly calibrated predictor that consistently generates prediction $\expect[\prediction\sim\PDF]{\truePredic_\predictor(\prediction)}$, i.e., its prediction distribution $f_{\predictor^{\NI}} \triangleq \Dirac(\expect[\prediction\sim\PDF]{\truePredic_\predictor(\prediction)})$ where $\Dirac$ is the Dirac delta function.
\end{definition}

\begin{definition}[Calibration decision loss (CDL), adopted from \citealp{HW-24}]
\label{def:CDL}
Given any (possibly miscalibrated) predictor $\predictor$, its \emph{calibration decision loss}, denoted by $\CDL{\predictor}$, is defined as follows:
\begin{align*}
    \CDL{\predictor} \triangleq 
    \sup\nolimits_{(\actionSpace, \agentU)\in\agentUClass} ~ 
    \displaystyle 
    \agentExpPay[{(\actionSpace, \agentU)}]{\predictor^{\Bayes}} - 
    \agentExpPay[{(\actionSpace, \agentU)}]{\predictor}~,
\end{align*}
where $\predictor^{\Bayes}$ denotes the perfectly calibrated predictor that generates the true expected outcome conditional on the prediction that is generated by the predictor $\predictor$. 
\end{definition}

First, given any miscalibrated predictor $\predictornew$, we observe that by letting $\predictor$ to be some perfectly calibrated predictors that are related to $\predictornew$, informativeness gap $\inforGap{\predictor, \predictornew}$ subsumes the UCal and CDL as follows.
\begin{fact}
\label{fact:subsume UCal and CDL}
Given any miscalibrated predictor $\predictor$, 
we have 
\begin{align*}
    \inforGap{\predictor^{\NI}, \predictor} = \UCal{\predictor}; \quad
    \inforGap{\predictor^{\Bayes}, \predictor} = \CDL{\predictor}~.
\end{align*}
\end{fact}

\xhdr{Upper bounding the informativeness gap}
We next show that we can use existing calibration measures (e.g., the Expected calibration error (ECE), defined as in \Cref{defn:ECE}), together with our relaxed earth mover's distance (REMD), to upper bound  informativeness gap $\inforGap{\cdot, \cdot}$.
\begin{definition}[Expected calibration error (ECE)]
\label{defn:ECE}
Given any (possibly miscalibrated) predictor $\predictor$, its \emph{expected calibration error}, denoted by $\ECE{\predictor}$, is defined as follows:
\begin{align*}
    \ECE{\predictor} 
    \triangleq \expect[\prediction\sim\PDF]{\abs{\prediction - \truePredic_\predictor(\prediction)}}
    = \int_0^1 \PDF(\prediction)\cdot \abs{\prediction - \truePredic_\predictor(\prediction)}\, \dd\prediction~.
\end{align*}
\end{definition}
With the above definition, we can obtain the following upper bound for $\inforGap{\predictor, \predictornew}$.
\begin{proposition}
\label{prop:upper bound dFlow}
For any two (possibly miscalibrated) predictors $\predictor, \predictornew$, their informativeness gap can be upper bounded as
\begin{align*}
    \inforGap{\predictor, \predictornew}
    \leq \REMD{\PDF, \PDFNew} + \ECE{\predictor} + \ECE{\predictornew}~.
\end{align*}
\end{proposition}
\begin{proof}
Given any two predictors $\predictor, \predictornew$, let 
\begin{align*}
    \miscali_{\Marg}^* & = \arg\inf\nolimits_{\miscali \in \couplingSpace(\PDF, \PDFNew) }
    \int_0^1\abs{ \int_0^1 \miscali(\prediction, \predictionNew)\cdot(\prediction - \predictionNew)\, \dd \predictionNew}\, \dd \prediction~.
\end{align*}
be the optimal coupling that achieves the relaxed earth's mover distance between the two prediction distributions $\PDF$ and $\PDFNew$.
Clearly we also have $\miscali_{\Marg}^*\in \feasibleflowSet(\PDF, \PDFNew)$.
From \Cref{thm:strong-duality-miscalibrated}, we have 
\begin{align*}
    \inforGap{\predictor, \predictornew}
    & = \earthDistFlow{\PDF, \PDFNew, \truePredic_{\predictor}, \truePredic_{\predictornew}}\\
    & = \inf_{\miscali\in\feasibleflowSet(\PDF, \PDFNew)}
    \int_0^1 \left|\int_0^1 \miscali(\prediction, \predictionNew) \cdot (\prediction-\predictionNew)\,\dd \predictionNew 
    + (\truePredic_\predictor(\prediction) - \prediction) \PDF(\prediction)
    - (\truePredic_\predictornew(\prediction) - \prediction) \PDFNew(\prediction) \right|  \, \dd \prediction \\
    & \le 
    \int_0^1 \abs{\int_0^1 \miscali_{\Marg}^*(\prediction, \predictionNew) \cdot (\prediction-\predictionNew)\,\dd \predictionNew 
    + (\truePredic_\predictor(\prediction) - \prediction) \PDF(\prediction)
    - (\truePredic_\predictornew(\prediction) - \prediction) \PDFNew(\prediction) }  \, \dd \prediction 
    \\
    & \le 
    \int_0^1 \abs{\int_0^1 \miscali_{\Marg}^*(\prediction, \predictionNew) \cdot (\prediction-\predictionNew)\,\dd \predictionNew }  \, \dd \prediction 
    + \ECE{\predictor} + \ECE{\predictornew} 
    \\
& = \REMD{\PDF, \PDFNew}
    + \ECE{\predictor} + \ECE{\predictornew}~.
\end{align*}
where the first inequality holds since $\miscali_{\Marg}^*\in\feasibleflowSet(\PDF, \PDFNew)$, and the second inequality holds due to the triangle inequality. This completes the proof of \Cref{prop:upper bound dFlow}.
\end{proof}

Combing \Cref{prop:upper bound dFlow} and \Cref{fact:subsume UCal and CDL},
we obtain the following tighter upper bound on CDL.
\begin{corollary}
\label{cor:upper bound CDL}
Given any (possibly miscalibrated) predictor $\predictor$, let $\predictor^{\Bayes}$ be its perfectly calibrated predictor (see construction in \Cref{def:CDL}), its calibration decision loss can be upper bounded as
\begin{align*}
\CDL{\predictor}
    \le 
    \REMD{f_{\predictor^{\Bayes}}, \PDF} + 
    \ECE{\predictor}
    \le 2\ECE{\predictor}~.
\end{align*}
\end{corollary}
\begin{proof}
Given any miscalibrated predictor $\predictor$ and its corresponding perfectly calibrated predictor $\predictor^{\Bayes}$, we know that 
\begin{align*}
    \CDL{\predictor} 
    & = \inforGap{\predictor^{\Bayes}, \predictor} \tag{by \Cref{fact:subsume UCal and CDL}}\\
    & = \earthDistFlow{f_{\predictor^{\Bayes}}, \PDF, \truePredic_{\predictor^{\Bayes}}, \truePredic_{\predictor}} 
    \tag{by \Cref{thm:strong-duality-miscalibrated}}\\
    & \le \REMD{f_{\predictor^\Bayes}, \PDF} + \ECE{\predictor^\Bayes} + \ECE{\predictor} 
    \tag{by \Cref{prop:upper bound dFlow}}\\
    & = 
    \REMD{f_{\predictor^\Bayes}, \PDF}
    + \ECE{\predictor}
    \tag{$\ECE{\predictor^\Bayes} = 0$}\\
    & \le 
    \EMD{f_{\predictor^\Bayes}, \PDF}
    + \ECE{\predictor} 
    \tag{by triangle inequality}\\
    & \le  2 \ECE{\predictor}
\end{align*}
which finishes the proof of \Cref{cor:upper bound CDL}.
\end{proof}
Previous work \cite{HW-24} have established that $\CDL{\predictor} \le 2\ECE{\predictor}$.
Here, we obtain a slightly tighter upper bound where the CDL of a predictor $\predictor$ can be upper bounded by the summation of the relaxed earth mover's distance $\REMD{f_{\predictor^\Bayes}, \PDF}$ between the two prediction distributions $f_{\predictor^\Bayes}, \PDF$ and its $\ECE{\predictor}$.
Notice that we have 
\begin{align*}
    \REMD{f_{\predictor^\Bayes}, \PDF} \le\EMD{f_{\predictor^\Bayes}, \PDF} \le \ECE{\predictor}~.
\end{align*}

\xhdr{Trilemma for complete, sound, and continuous informativeness measure}
\cite{BGHN-23} proposed one desiderata for an ideal calibration measure that should satisfy: the consistency -- the measure is robust completeness (correct predictions have low error) and robust soundness (incorrect predictions have high error).
Our proposed informativeness measure $\earthDistFlow{\cdot, \cdot, \cdot, \cdot}$ can be served as a tool for quantifying the informativeness gap between any pair of predictors.
{By \Cref{thm:strong-duality-miscalibrated}, $\earthDistFlow{\cdot, \cdot, \cdot, \cdot}$ is both complete and sound (in fact, it satisfies these criteria exactly), and thus it is consistent.}

However, modern machine learning predictors typically make predictions continuously over a probabilistic space. This motivates an additional desideratum: we seek calibration error measures that are continuous in prediction values, meaning they should not be sensitive to small perturbations in predictions. 
For instance, the distance to calibration measure (or the earth mover's distance $\EMD{\cdot, \cdot}$ between predictors) \citep{BGHN-23}, the smooth calibration error \citep{KF-08} and the smooth ECE \citep{KFE-18}, all satisfy this desirable continuity property.

This raises a natural question: Does our informativeness measure $\earthDistFlow{\cdot, \cdot, \cdot, \cdot}$ also satisfy such a continuity property? 
If not, is there an alternative informativeness measure, $\Dist{\cdot, \cdot}$\footnote{Here we allow the informativeness measure $\Dist{\cdot, \cdot}$ to be an arbitrary function over the pair of two predictors.}, that is both consistent w.r.t.\ informativeness gap $\inforGap{\cdot, \cdot}$ and ``continuous''?

Unfortunately, as we show below, our informativeness measure $\earthDistFlow{\cdot, \cdot, \cdot, \cdot}$ fails to satisfy a standard notion of continuity (specifically, continuity under the earth mover's distance $\EMD{\cdot, \cdot}$).\footnote{We adopt $\EMD{\cdot, \cdot}$ as our continuity measure. 
This is also widely adopted in the literature \citep[e.g.,][]{BGHN-23,HWY-25}.}
Moreover, we establish a stronger impossibility result: no informativeness measure  $\Dist{\cdot, \cdot}$ can satisfy completeness, soundness, and continuity simultaneously.
\begin{proposition}[Informativeness measure trilemma]
\label{prop:impossibility}
For any informativeness measure $\Dist{\cdot, \cdot}$ that takes as input any pair of predictors, at least one of the following property must fail:
\begin{itemize}
    \item \emph{\underline{(Approximate completeness)}}:
    for any predictors $\predictor,\predictornew$,
    \begin{align*}
        \Dist{\predictor, \predictornew} = O\left(\left(\inforGap{\predictor,\predictornew}\right)^\completeConstOne\right)
    \end{align*}
    where $\completeConstOne > 0$ is an absolute constant.
    \item \emph{\underline{(Approximate soundness)}}:
    for any predictors $\predictor,\predictornew$,
    \begin{align*}
        \Dist{\predictor, \predictornew} = \Omega\left(\left(\inforGap{\predictor,\predictornew}\right)^\soundConstOne\right)
    \end{align*}
    where $\soundConstOne > 0$ is an absolute constant.
    \item \emph{\underline{(Continuity)}}:  for any predictors $\predictor,\predictor\primed, \predictornew, \predictornew\primed$, 
    \begin{align*}
        \abs{\Dist{\predictor, \predictornew} - \Dist{\predictor\primed, \predictornew\primed}
        }
        = O\left(\EMD{\PDF, f_{\predictor\primed}} + 
        \EMD{\PDFNew, f_{\predictornew\primed}}\right)
    \end{align*}
    where $\EMD{\cdot, \cdot}$ is defined as in \Cref{defn:emd}.
\end{itemize}
\end{proposition}
\begin{proof}
We prove the proposition statement by contradiction. Suppose there exists informativeness measure $\Dist{\cdot,\cdot}$ that satisfies all three properties. Fix a sufficiently small $\varepsilon > 0$.
We consider the following two predictors:
\begin{itemize}
    \item Predictor $\predictor$ is perfectly calibrated with a point mass on $\prediction = 0.5$;
    \item Predictor $\predictornew$ is miscalibrated and satisfies that $\PDFNew(0.5+\varepsilon) = \PDFNew(0.5-\varepsilon) = 0.5$, and its true expected outcome function satisfies $\truePredic_\predictornew(0.5+\varepsilon) = 0$, and  $\truePredic_\predictornew(0.5 - \varepsilon) = 1$.
\end{itemize}
By construction, the informativeness gap satisfies
\begin{align*}
    \inforGap{\predictor,\predictor} &= 0
   \intertext{and}
    \inforGap{\predictor,\predictornew} & = \earthDistFlow{\PDF,\PDFNew,\truePredic_\predictor,\truePredic_\predictornew}
    \\
    &= 
    2\cdot \max_{t\in [0, 1]} ~
    \left(\SCDF(t) + \int_0^t (\prediction - \truePredic_\predictor(\prediction))\cdot \PDF(\prediction)\, \dd \prediction \right)
    - 
    \left(\SCDFNew(t) + \int_0^t (\prediction - \truePredic_\predictornew(\prediction)) \cdot \PDFNew(\prediction)\, \dd \prediction\right)
    \\
    &= \Theta(1)
\end{align*}
where the first equality holds due to \Cref{thm:strong-duality-miscalibrated}, the second equality holds due to \Cref{lem:idf miscali}, and the third equality holds by algebra.
Due to the approximate completeness property and the approximate soundness property, the informativeness measure $\Dist{\cdot,\cdot}$ satisfies
\begin{align*}
    \Dist{\predictor,\predictor} = 0
    \;\;
    \mbox{and}
    \;\;
    \Dist{\predictor,\predictornew} = \Theta(1)
\end{align*}
Therefore, 
\begin{align*}
    \Theta(1) = 
    \abs{\Dist{\predictor,\predictornew}-\Dist{\predictor,\predictor}}
    =
    O(\EMD{\PDF,\PDF} + \EMD{\PDFNew,\PDF}) = O(\varepsilon)
\end{align*}
which is a contradiction. 
(Here the second equality holds due to the continuity property.)
This completes the proof of \Cref{prop:impossibility}.
\end{proof}

\section{Sample Efficiency for Estimating \texorpdfstring{$\inforGap{\predictor, \predictornew}$}{Informativeness Gap}}
\label{sec:sample complexity}

In practice, we may not have direct access to the prediction distribution $\PDF$ or the true expected outcome function $\truePredic_\predictor$ of a given predictor $\predictor$. Instead, we often only have sample access. Motivated by this, we study in this section the sample complexity of estimating informativeness gap $\inforGap{\cdot,\cdot}$.
Our main result shows that informativeness gap $\inforGap{\cdot,\cdot}$ can be estimated sample-efficiently under the prediction-only access model, defined as follows:\footnote{In addition to the prediction-only access model, prior work \citep{DKRRY-21,BGHN-23} also considers a more powerful alternative, known as the \emph{sample-access model}, which allows sampling of feature–prediction–label triples. We also remark that our proposed informativeness gap $\inforGap{\cdot,\cdot}$ is inherently defined under the prediction-only access model. This stands in contrast to the ground-truth distance from calibration proposed in \cite{BGHN-23}, which is defined under the more powerful sample-access model.}

\begin{definition}[Prediction-only access model, adopted from \citealp{BGHN-23}]
Given a (possibly miscalibrated) predictor $\predictor$, the \emph{prediction-only access model} assumes the existence of an oracle that samples prediction-label pairs $(\prediction, \outcome)$ consistent with predictor $\predictor$. Specifically, prediction $\prediction$ is drawn from the prediction distribution $\PDF$, and conditional on $\prediction$, the outcome $\outcome$ is drawn from a Bernoulli distribution $\Bern(\truePredic_\predictor(\prediction))$ with mean equal to the true expected outcome $\truePredic_\predictor(\prediction)$.
\end{definition}

\begin{restatable}[Sample complexity]{theorem}{thmSampleMiscali}
    \label{thm:sample miscali}
    Fix any two (possibly miscalibrated) predictors $\predictor, \predictornew$ and any $\varepsilon,\delta \in(0, 1)$. Define $n \triangleq \Theta(\frac{1}{\varepsilon^2}\ln\frac{1}{\delta})$. With $n$ i.i.d.\ prediction-label samples $\{(\prediction_i^{\predictor},\outcome_i^{\predictor})\}_{i\in[n]}$ from predictor $\predictor$ and $n$ i.i.d.\ prediction-label samples $\{(\prediction_i^{\predictornew},\outcome_i^{\predictornew})\}_{i\in[n]}$ from predictor $\predictornew$, define estimator~$\empricaldInforGap{\predictor, \predictornew}$ as:\footnote{Operator $\plus{\cdot} \triangleq \max\{0, \cdot\}$.}
    \begin{align*}
    \empricaldInforGap{\predictor, \predictornew} \triangleq
        \max\limits_{t\in[0, 1]}~
        \frac{1}{n}\sum\limits_{i\in[n]} 
        \left(\plus{t - \prediction_i^{\predictor}} + 
        (\prediction_i^{\predictor} - \outcome_i^{\predictor})\cdot 
        \indicator{\prediction_i^{\predictor} \leq t}
        \right)
        -
        \left(\plus{t - \prediction_i^{\predictornew}} + 
        (\prediction_i^{\predictornew} - \outcome_i^{\predictornew})\cdot 
        \indicator{\prediction_i^{\predictornew} \leq t}
        \right)
    \end{align*}
    Then, informativeness gap $\inforGap{\predictor,\predictornew}$ and estimator~$\empricaldInforGap{\predictor, \predictornew}$ satisfy
    \begin{align*}
        \abs{\inforGap{\predictor,\predictornew} - \empricaldInforGap{\predictor, \predictornew}}\leq \varepsilon
    \end{align*}
    with probability at least $1 - \delta$.
\end{restatable}

For a miscalibrated predictor, the true expected outcome function can be highly non-smooth and, as a result, is unlikely to be estimated in a sample-efficient manner. However, although this function plays a crucial role in determining the agent's payoff (e.g., see \Cref{lem:interim utility expression}), we bypass this challenge in \Cref{thm:sample miscali} by avoiding the need to estimate the true expected outcome function directly.

\begin{proof}[Proof of \Cref{thm:sample miscali}]
Invoking \Cref{thm:strong-duality-miscalibrated}, the informativeness gap $\inforGap{\predictor,\predictornew}$ can be expressed as 
\begin{align*}
    \inforGap{\predictor,\predictornew} 
    &=
    \earthDistFlow{\PDF, \PDFNew,\truePredic_\predictor,\truePredic_\predictornew} 
    \\
    &= 
    2\cdot \max_{t\in [0, 1]} ~
    \left(\SCDF(t) + \int_0^t (\prediction - \truePredic_\predictor(\prediction))\cdot \PDF(\prediction)\, \dd \prediction \right)
    - 
    \left(\SCDFNew(t) + \int_0^t (\prediction - \truePredic_\predictornew(\prediction)) \cdot \PDFNew(\prediction)\, \dd \prediction\right)
    \\
\end{align*}
Note that 
\begin{align*}
    \SCDF(t) &= \int_0^t \CDF(\prediction)\,\dd\prediction = \expect[\prediction\sim\PDF]{\plus{t - \prediction}}
\intertext{and} 
    \int_0^t (\prediction - \truePredic_\predictor(\prediction))\cdot \PDF(\prediction)
    & = \expect[\prediction\sim\PDF]{{(\prediction- \truePredic_\predictor(\prediction))\cdot \indicator{\prediction\le t}}}
    \\    
    & = \expect[\prediction\sim\PDF]{{\left(\prediction- 
    \expect[\outcome\sim\Bern(\truePredic_\predictor(\prediction))]{\outcome}\right)\cdot \indicator{\prediction\le t}}}
    \\
    & = \expect[\prediction\sim\PDF]{\expect[\outcome\sim\Bern(\truePredic_\predictor(\prediction))]{(\prediction-\outcome)\cdot \indicator{\prediction\le t}}}
\end{align*}
where the last equality holds due to the law of total expectation. Similar expressions hold for predictor $\predictornew$. Thus, the informativeness gap $\inforGap{\predictor,\predictornew}$ can be expressed as 
\begin{align}
\label{eq:inforGap reformulation for estimating}
    \begin{split}
    \inforGap{\predictor,\predictornew} 
    & = 
    2\cdot \max_{t\in [0, 1]}~
    \left(
    \expect[\prediction\sim\PDF]{\expect[\outcome\sim\Bern(\truePredic_\predictor(\prediction))]{
    \plus{t - \prediction} + 
    (\prediction-\outcome)\cdot \indicator{\prediction\le t}
    }}
    \right.
    \\
    & \qquad\qquad\qquad \left.
    -
    \expect[\prediction\sim\PDFNew]{\expect[\outcome\sim\Bern(\truePredic_\predictornew(\prediction))]{
    \plus{t - \prediction} + 
    (\prediction-\outcome)\cdot \indicator{\prediction\le t}
    }}
    \right)
    \end{split}
\end{align}
Given the construction of $\empricaldInforGap{\cdot,\cdot}$, we conclude that $\empricaldInforGap{\predictor,\predictornew}$ is a consistent estimator of the informativeness gap $\inforGap{\predictor,\predictornew}$. 

We next analyze the sample complexity of estimator $\empricaldInforGap{\predictor,\predictornew}$. Note that in the reformulation~\eqref{eq:inforGap reformulation for estimating} of the informativeness gap $\inforGap{\predictor,\predictornew}$, the family of functions $\{\plus{t - \prediction}\}_{t \in[0, 1]}$ has a pseudo-dimension of $1$, and family of of functions $\{(\prediction-\outcome)\cdot \indicator{\prediction\leq t}\}_{t\in[0, 1]}$ also has a pseudo-dimension of $1$. Therefore, by the uniform convergence theorem (for bounded real-valued class), we have that when the number of samples, $n$, satisfies 
\begin{align*}
    n\geq \frac{16}{\varepsilon^2}\cdot \left(\ln (2e n) + \ln \frac{8}{\delta}\right)
\end{align*}
each of the following holds with probability at least $1-\frac{\delta}{2}$:
\begin{align*}
    \max_{t\in[0, 1]} \abs{
    \frac{1}{n}\sum\limits_{i\in[n]} 
    \plus{t - \prediction_i^{\predictor}} + 
    (\prediction_i^{\predictor} - \outcome_i^{\predictor})\cdot 
    \indicator{\predictor_i^{\predictor} \leq t}
    -
    \expect[\prediction\sim\PDF]{\expect[\outcome\sim\Bern(\truePredic_\predictor(\prediction)]{
    \plus{t - \prediction} + 
    (\prediction-\outcome)\cdot \indicator{\prediction\le t}
    }}
    } & \leq \frac{\varepsilon}{2}\\
    \max_{t\in[0, 1]} \abs{
    \frac{1}{n}\sum\limits_{i\in[n]} 
    \plus{t - \prediction_i^{\predictornew}} + 
    (\prediction_i^{\predictornew} - \outcome_i^{\predictornew})\cdot 
    \indicator{\predictor_i^{\predictornew} \leq t}
    -
    \expect[\prediction\sim\PDFNew]{\expect[\outcome\sim\Bern(\truePredic_\predictornew(\prediction)]{
    \plus{t - \prediction} + 
    (\prediction-\outcome)\cdot \indicator{\prediction\le t}
    }}
    } & \leq \frac{\varepsilon}{2}
\end{align*}
Thus, invoking the triangle inequality, we obtain
\begin{align*}
    \abs{\inforGap{\predictor,\predictornew} - \empricaldInforGap{\predictor, \predictornew}}\leq \varepsilon
\end{align*}
and completes the proof of \Cref{thm:sample miscali}.
\end{proof}

\section{Experiments}
\label{sec:numerical}

In this section, we conduct two set of experiments on real-world models and datasets. 
Our goal is to illustrate two practical benefits of the proposed informativeness gap (\InfoGap):
(1) it provides an alternative -- and more decision-relevant -- criterion for benchmarking LLMs' forecasting capabilities
(2) it offers a principled way to evaluate how ad hoc post-processing methods designed for reducing model miscalibration impact decision usefulness.

\subsection{Experiment Setup}
\label{subsec:experiment setup}
Throughout our experiments, we focus on binary prediction tasks with realized outcome $\outcome\in\{0, 1\}$. Each predictor outputs a probabilistic forecast $\prediction\in [0,1]$, interpreted as its belief that the event will occur, i.e., $\outcome = 1$. 

\xhdr{Datasets} 
We conduct experiments on two binary forecasting datasets: weather rain dataset and Bitcoin price dataset. The detailed description of the two datasets are as follows:
\begin{itemize}
    \item The weather rain data are obtained from public Kaggle dataset.\footnote{\url{https://www.kaggle.com/datasets/guillemservera/global-daily-climate-data}} It contains meteorological variables such as temperature, humidity, wind speed, cloud cover and pressure, along with an observed rain indicator. Based on these data, we construct two families of binary events.
\item The Bitcoin price data are obtained from CoinMarketCap.\footnote{\url{https://coinmarketcap.com/currencies/bitcoin/historical-data/}} 
    We use daily historical records retrieved from 2024-10-04 to 2025-11-06, including daily historical records with timestamps, open/close prices and daily high/low prices. 
\end{itemize}
For each dataset, we construct a collection of event templates to generate labeled binary outcomes; see \Cref{apx:experimental process} for the event definitions and threshold selections.

\xhdr{Models} We evaluate four LLMs as probabilistic predictors: GPT-4o mini \citep{MLZW-24}, DeepSeek-V3.2 \citep{LMLX-25}, Gemini 2.0 Flash-Lite \citep{Gemini-25}, Qwen-Plus \citep{Qwen-25}. Throughout the paper, we refer to these models using shortened names: {\ChatGPT} (GPT-4o-mini), {\DeepSeek} (DeepSeek-V3.2), {\Gemini} (Gemini 2.0 Flash-Lite) and {\Qwen} (Qwen-Plus).

\xhdr{Metrics} In addition to informativeness gap, we report two other standard metrics that capture different dimensions of probabilistic predictive quality. 
\begin{itemize}
    \item We use Brier score ($\textsc{BS}$) to quantify the deviation between probabilistic predictions and true outcomes.
Given a label set $\{\outcome_i\}_{i\in[n]}$ and corresponding prediction set $\{\prediction_i\}_{i\in[n]}$, the Brier score of a LLM $\predictor$ is defined as: 
    \begin{align*}
        \brierScore{\predictor}=\frac{1}{n}\sum\nolimits_{i\in[n]}(\prediction_i-\outcome_i)^2~.
    \end{align*}
In essence, it measures the squared deviation of predicted probabilities from the realized outcomes; thus, lower values indicate better performance.
    \item 
    We use ECE to quantify the degree of miscalibration in predicted probabilities.
To compute ECE in our setting, we further derive a discrete and computational form. Given the prediction set $\{\prediction_i\}_{i\in[n]}$ along with their frequency function $\frequency_\predictor(\cdot)$ and the empirical outcome frequency $\truePredic_\predictor(\cdot)$, the ECE of a LLM $\predictor$ is computed as:
    \begin{align*}
        \ECE{\predictor}=\sum\nolimits_{i\in[n]} \frequency_\predictor(\prediction_i)\cdot\abs{\prediction_i-\truePredic_\predictor(\prediction_i)}~.
    \end{align*}
    Intuitively, it aggregates the absolute gaps $\abs{\prediction_i-\truePredic_\predictor(\prediction_i)}$ weighted by the empirical frequency $\frequency_\predictor(\prediction_i)$, and small values indicate better calibration.
\end{itemize}
See \Cref{apx:experimental process} for the detailed experimental process.

\subsection{Experiment I: \InfoGap\ to Benchmark LLMs}
\label{subsec:benchmark}

Experiment I investigates whether \InfoGap\ offers an alternative yet more decision-relevant benchmark for evaluating predictors' forecasting performance, particularly when standard metrics provide limited guidance.

A priori, {\InfoGap} need not induce a total ordering: it is possible that both $\InfoGap[\predictor,\predictornew]$ and $\InfoGap[\predictornew,\predictor]$ are large, meaning that each predictor can outperform the other on different downstream tasks.
This motivates our first question: how often does {\InfoGap} nevertheless yield a clear, actionable deployment choice in practice?
To answer it, we begin by quantifying how frequently predictor pairs are comparable under {\InfoGap} and  find that in most cases they are indeed comparable.
We then use two representative cases to illustrate the mechanisms driving these comparisons.

\xhdr{How comparable \InfoGap\ is in practice} 
We introduce an {\InfoGap} tolerance (\infogapTolerance) $\IGtol\in[0,1]$, interpreted as a tolerance level, to quantify the extent to which pairs of LLM predictors are comparable under our proposed \InfoGap, aggregated across all events.
\begin{definition}
Two predictors $\predictor$ and $\predictornew$ are {\em $\IGtol$-comparable} if 
\begin{align*}
    \min\{\inforGap{\predictor, \predictornew}, \inforGap{\predictornew, \predictor}\}\le \IGtol
\end{align*}
and they are {\em strictly $\IGtol$-comparable} if 
\begin{align*}
    \min\{\inforGap{\predictor, \predictornew}, \inforGap{\predictornew, \predictor}\}\le \IGtol
    \;\;
    \mbox{and}
    \;\;\max\{\inforGap{\predictor, \predictornew}, \inforGap{\predictornew, \predictor}\}> \IGtol
\end{align*}
\end{definition}
This definition distinguishes two regimes: (1) both directional gaps are at most $\IGtol$, in which case neither predictor can achieve more than $\IGtol$ payoff advantage over the other across decision tasks (so they are effectively interchangeable under {\InfoGap}); 
or (2) exactly one direction is at most $\IGtol$ while the other exceeds it, in which case one predictor is never much worse than the other across tasks, yet the other direction can exhibit a substantial advantage, providing a clear directional preference.
Figure \ref{fig:IG tolerance} reports the fraction of comparable (red) and strictly comparable (blue) LLM pairs aggregated across all events.

\begin{figure}[H]
  \centering
  \includegraphics[width=0.7\textwidth]{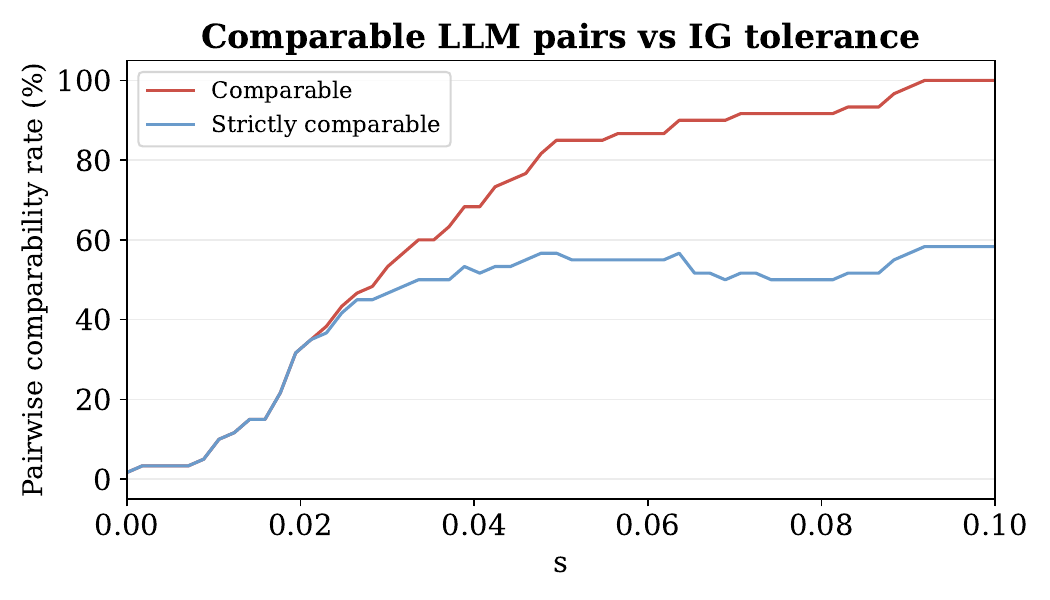}
  \caption{\textbf{{\infogapTolerance} threshold and pairwise comparability.} The plot reports the fraction of LLM pairs that are comparable (red) and strictly comparable (blue) as the {\infogapTolerance} threshold $\IGtol$ increases, aggregated over 10 events (4 LLMs per event, $\binom{4}{2}=6$ pairs; 60 pairs total). 
  }
  \label{fig:IG tolerance}
\end{figure}

The red curve suggests that approximate comparability is common even under a small tolerance: around $\IGtol = 0.05$, more than $80\%$
$\IGtol$-comparable, and by $\IGtol = 0.1$ essentially all pairs in our experiments are comparable. Put differently, while exact comparability at $\IGtol = 0$ is rare in our evaluated setting, allowing a modest {\infogapTolerance} threshold quickly yields meaningful comparisons for most model pairs. 
This observation is consistent with the fact that the classical Blackwell order is an exact, task-uniform notion of informativeness (it defines a partial order and requires one predictor to be weakly better for every downstream decision problem). 
Our results should not be read as challenging the relevance of Blackwell dominance; rather, they suggest that in practical evaluations of predictors, an approximate notion of comparability can be useful for producing actionable comparisons.

Beyond comparability, we also ask whether the comparisons are strict or not, i.e., whether the high comparability rate is driven merely by near-ties. 
The blue curve reports the fraction of strictly $\IGtol$-comparable pairs. Even at $\IGtol = 0.1$, nearly $60\%$ of LLM pairs are strictly comparable, indicating that for a substantial fraction of model pairs, 
{\InfoGap} yields a clear directional preference rather than declaring the two predictors essentially interchangeable. Overall, these results suggest that a small tolerance $\IGtol$ makes {\InfoGap} broadly comparable while still preserving directional discrimination for many predictor pairs in our experiments.

\xhdr{Representative case studies} 
We next analyze two representative cases that illustrate how \InfoGap\ provides additional insight when Brier score and ECE provide conflicting results (\textbf{Case I}) or are nearly tied (\textbf{Case II}). 
We select one case from the weather rain dataset and the other from the Bitcoin price dataset.

\indent \textbf{Case I: Brier score and ECE conflict.}
Figure \ref{fig:bs and ece of weather} reports results for weather single-day event. We observe a metric disagreement between two LLM predictors, {\Gemini} and {\Qwen}. Concretely, {\Gemini} has a Brier score of $0.232$ compared to $0.252$ for {\Qwen}, while {\Qwen} has an ECE of $0.358$ compared to $0.365$ for {\Gemini}. This example shows that Brier score and ECE may lead to opposite conclusions for the same pair of predictors, reflecting their emphasis on different aspects of predictive quality (overall probabilistic accuracy versus calibration). 
In contrast, the pairwise \InfoGap\ comparison in Figure \ref{fig:infogap of weather} gives a clear separation between the two models:
$\inforGap{\Gemini, \Qwen}=0.215$ versus $\inforGap{\Qwen, \Gemini}=0.082$.
Thus, even when standard metrics point disagree, 
{\InfoGap} can serve as a complementary criterion by quantifying how large a decision-relevant advantage one predictor may have over the other across tasks.

\begin{figure}
  \centering
  \subfloat[Weather: single-day rain event]{\includegraphics[width=0.48\linewidth]{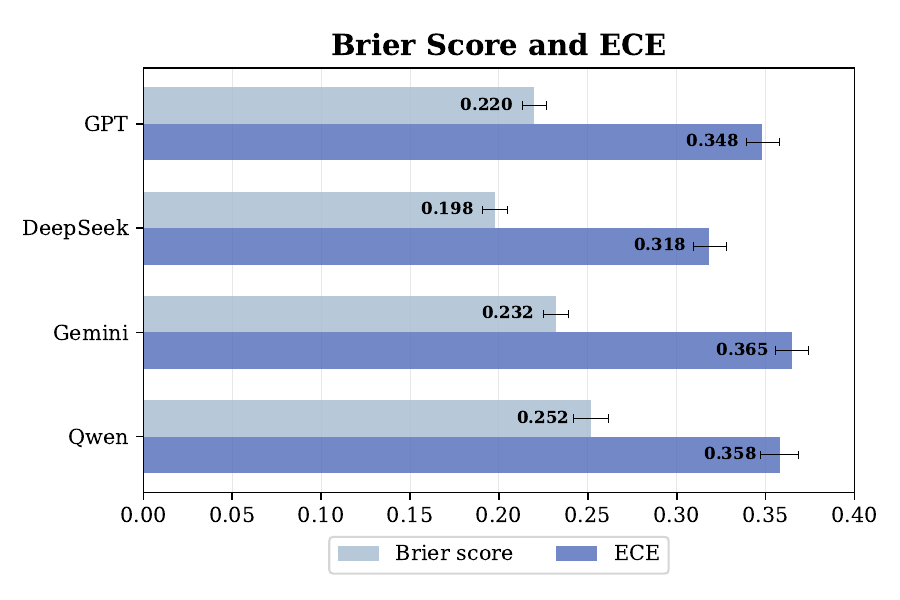}\label{fig:bs and ece of weather}}\hspace{3mm}
  \subfloat[Bitcoin: \$100 price-increase event]{\includegraphics[width=0.48\linewidth]{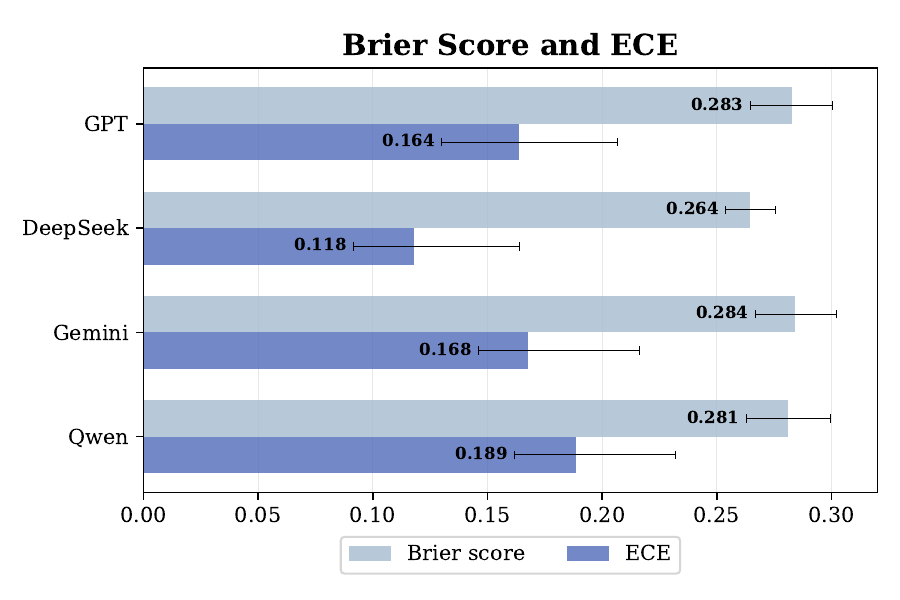}\label{fig:bs and ece of bitcoin}}
  \caption{\textbf{Brier score and ECE across different LLMs.} We report Brier score and expected calibration error (ECE) of four LLM predictors on two representative events: single-day rain event on weather rain dataset (left) and absolute increase threshold of \$100 on Bitcoin price dataset (right). Error bars show 95\% confidence intervals computed via bootstrap resampling of the obtained $\{\prediction_i, \outcome_i\}$ pairs.}
  \label{fig:bs and ece of both}
\end{figure}

\indent \textbf{Case II: Brier and ECE nearly tied.}
Figure \ref{fig:bs and ece of bitcoin} reports results for Bitcoin event with an absolute threshold of $\$100$. In this setting, both Brier score and ECE metrics are almost the same for two models, {\ChatGPT} and {\Gemini}. Specifically, {\ChatGPT} has a Brier score of $0.283$ compared to $0.284$ for {\Gemini}, while {\ChatGPT} has an ECE of $0.164$ compared to $0.168$ for {\Gemini}, making it difficult to separate the two predictors using standard metrics. In contrast, 
the pairwise \InfoGap\ comparison in Figure \ref{fig:infogap of bitcoin} produces a more pronounced directional distinction:
$\inforGap{\Gemini, \ChatGPT}=0.118$ and $\inforGap{\ChatGPT, \Gemini}=0.042$. This example highlights that \InfoGap\ can provide finer resolution even when both Brier score and ECE are inconclusive.
\begin{figure}
  \centering
  \subfloat[Weather: single-day rain event]{\includegraphics[width=0.4\linewidth]{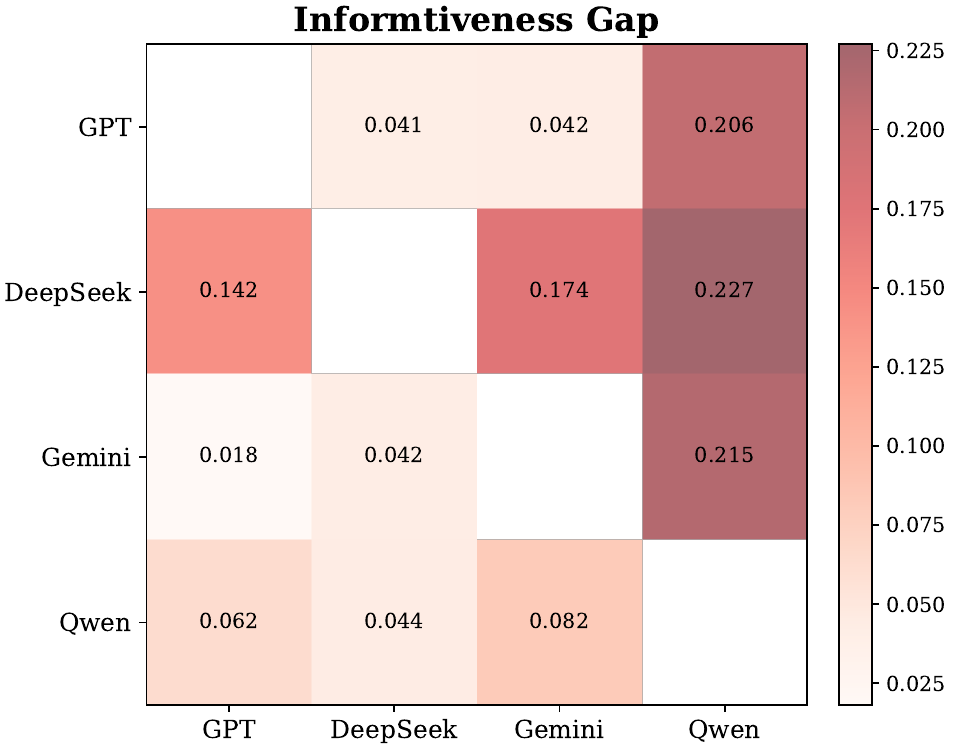}\label{fig:infogap of weather}}\hspace{5mm}
  \subfloat[Bitcoin: \$100 price-increase event]{\includegraphics[width=0.4\linewidth]{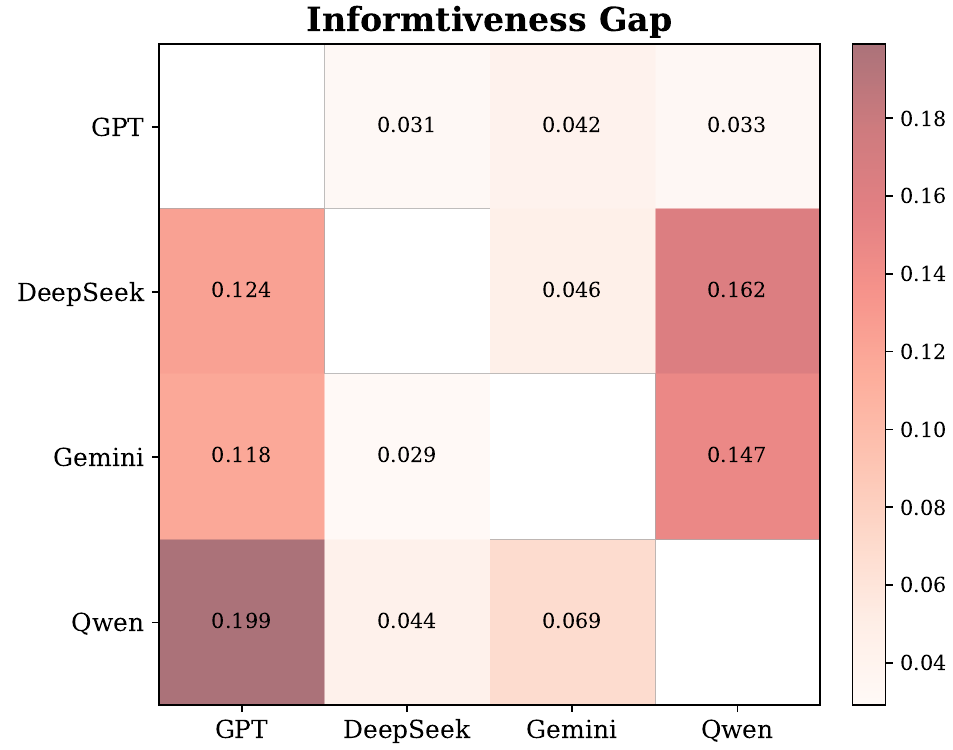}\label{fig:infogap of bitcoin}}
  \caption{\textbf{{\InfoGap} across LLMs.} Heatmaps compare the (directional) {\InfoGap} values among four LLMs for the same two events as in Figure \ref{fig:bs and ece of both}. Each entry $(i, j)$ reports $\inforGap{\predictor_i,\predictor_j}$, with larger values indicating a larger decision-relevant advantage of $\predictor_i$ over $\predictor_j$.}
  \label{fig:infogap of both}
\end{figure}

\xhdr{Why \InfoGap\ resolves both cases} 
Brier score and ECE capture different aspects of predictive quality: the former evaluates overall probabilistic accuracy, while the latter measures how closely predicted probabilities align with empirical outcome frequencies conditional on the forecast.
As a result, these metrics can disagree when accuracy and calibration trade off (as in \textbf{Case I}), or become inconclusive when two models perform nearly the same under both metrics (as in \textbf{Case II}).
In both scenarios, our results show that the informativeness gap ({\InfoGap}) offers additional, decision-relevant insights.

We interpret these insights through two complementary tools: the {\generalizedSCDF} characterization (see \Cref{thm:strong-duality-miscalibrated} and \Cref{def:ca-scdf}) of {\InfoGap} and reliability diagrams that highlight miscalibration around thresholds relevant to decision-making.
Intuitively, a predictor $\predictor$ that has higher {\generalizedSCDF} curve than another predictor $\predictor$ at some threshold $t$, then predictor $\predictor$ yields a higher decision payoff for the decision task whose indirect utility is $V$-shaped with threshold $t$.
More generally, the pairwise \InfoGap\ between two predictors is governed by the maximum vertical separation between their {\generalizedSCDF} curves over thresholds $t\in[0,1]$.

In \textbf{Case I}, the {\generalizedSCDF} curves of {\Qwen} and {\Gemini} intersect near $t=0.5$, indicating that neither predictor uniformly dominates the other across all threshold: {\Qwen} is favored at lower thresholds, whereas {\Gemini} is favored at higher ones. In particular, {\Qwen} reaches its largest advantage over {\Gemini} at $t^\dag=0.23$, where the vertical gap is approximately $0.041$, yielding $\inforGap{\Qwen, \Gemini}=0.082$. The ordering reverses in higher-threshold region where the separation becomes substantially more pronounced: it peaks at $t^*=0.65$ with a vertical gap reaches $0.1077$, implying $\inforGap{\Gemini, \Qwen}=0.215$. Since {\Gemini}'s maximal advantage substantially exceeds {\Qwen}'s in the opposite direction, the \InfoGap\ criterion ultimately favors {\Gemini}. 

To further explain the behavior near $t^*=0.65$, we inspect the reliability diagrams in Figure \ref{single day RD: gemini and qwen}. Empirically, {\Qwen} assigns noticeably more weight in the decision-flipping quadrants, while {\Gemini} concentrates relatively more mass in the decision-consistent quadrants, aligning with the larger directional gap $\inforGap{\Gemini,\Qwen}$. Taken together, {\generalizedSCDF} identifies the threshold region that maximally separates the two predictors, and the reliability diagram provides complementary and decision-level evidence for why this separation concentrates around $t^*$. Consequently, \InfoGap\ remains decisive even when Brier score and ECE lead to conflicting rankings.

\begin{figure}[H]
  \centering
  \subfloat[Weather: single-day rain event]{
    \includegraphics[width=0.45\linewidth]{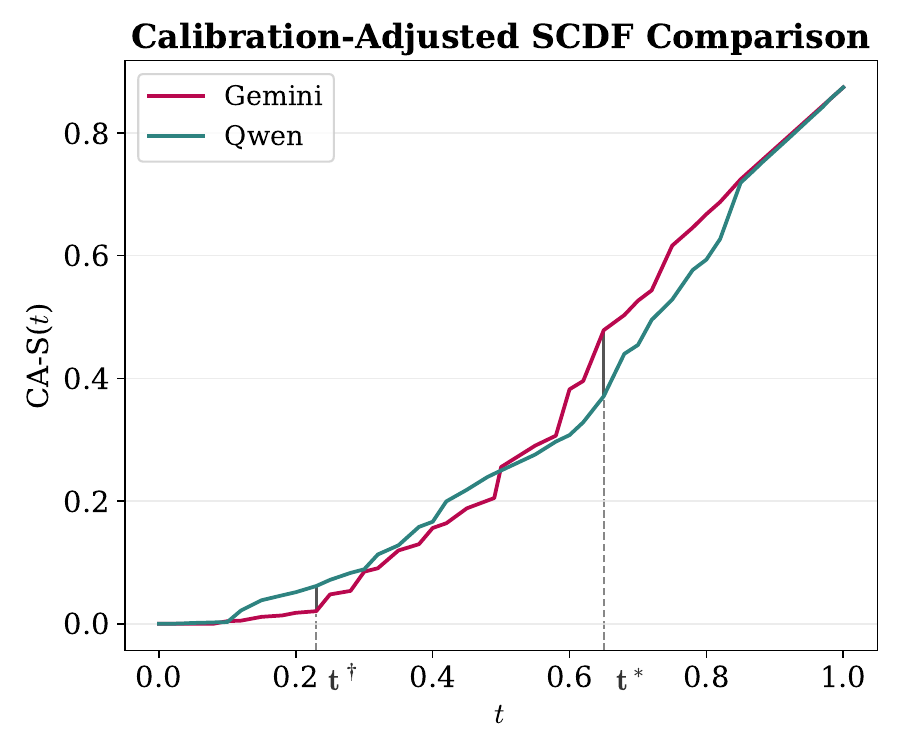}
    \label{single day scdf gemini qwen}
    }\hspace{6mm}
  \subfloat[Bitcoin: \$100 price-increase event]{
    \includegraphics[width=0.45\linewidth]{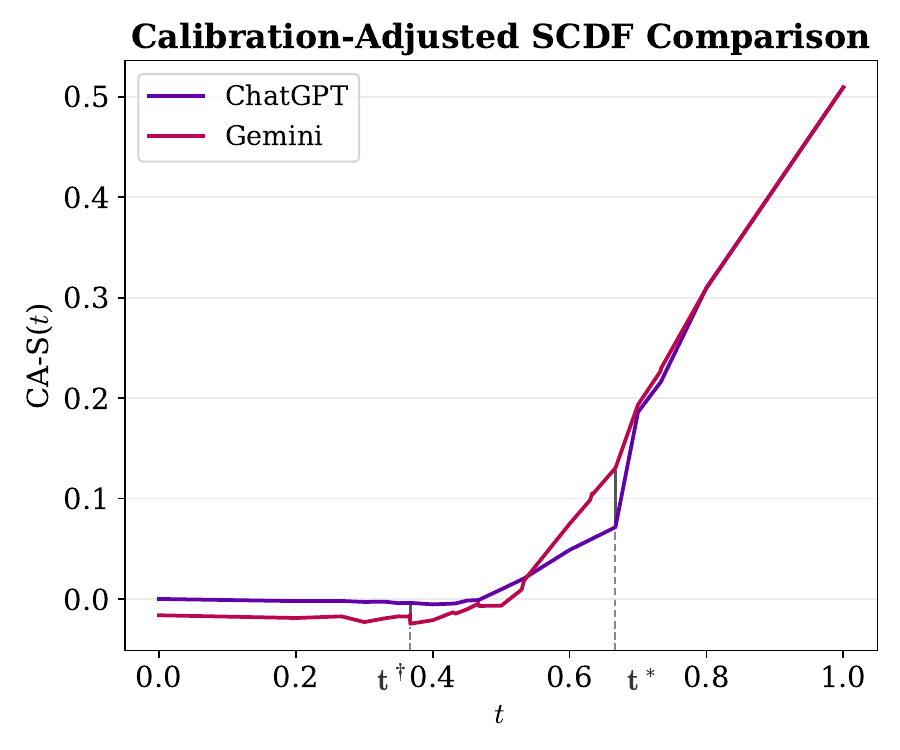}
    \label{abs 100 scdf: gemini and gpt}
    }
  \caption{\textbf{{\ProGeneralizedSCDF} curves and {\InfoGap} maximizer.} Each panel plots the {\generalizedSCDF} for two LLMs on the weather and Bitcoin tasks. The pairwise {\InfoGap} is characterized by the maximum vertical separation between the two curves over thresholds $t$, and the dashed line marks the maximizer $t^*$ where the two LLMs are most distinguishable.}
  \label{CA-SCDF}
\end{figure}

In \textbf{Case II}, despite both {\ChatGPT} and {\Gemini} have similar Brier score and ECE, the {\generalizedSCDF} still reveals meaningful region-dependent differences. For small thresholds $t$, the two curves remain very close with {\Gemini} slightly below {\ChatGPT} in this region, suggesting that the two predictors behave similarly in the low-$t$ region and are weakly distinguishable there. As $t$ increases, a clearer asymmetry emerges: {\ChatGPT} attains its largest advantage over {\Gemini} at $t^\dag=0.367$, where the vertical gap is approximately $0.021$, corresponding to $\inforGap{\ChatGPT, \Gemini}=0.042$. In the opposite direction, {\Gemini} exhibits a markedly stronger separation, peaking at $t^*=0.667$ with a vertical gap of $0.059$, which implies $\inforGap{\Gemini,\ChatGPT}=0.118$. Comparing the two directional maxima shows that {\Gemini}'s advantage is substantially stronger, so the {\InfoGap} criterion prefers {\Gemini}
even though the two standard metrics are numerically inconclusive.

To interpret why the dominant separation concentrates near $t^*=0.667$, we again inspect the reliability diagram in Figure \ref{abs 100 RD: gemini and gpt}. We observe that {\ChatGPT} places more weight in the decision-flipping quadrants, whereas {\Gemini} concentrates relatively more mass in the decision-consistent quadrants. This local imbalance near $t^*$ aligns with the {\generalizedSCDF} attaining its maximal separation near $t^*$ and provides a decision-level explanation for why $\inforGap{\Gemini, \ChatGPT}$ exceeds $\inforGap{\ChatGPT,\Gemini}$ in this case.

\begin{figure}
  \centering
  \subfloat[Weather: single-day rain event]{
    \includegraphics[width=0.75\linewidth]{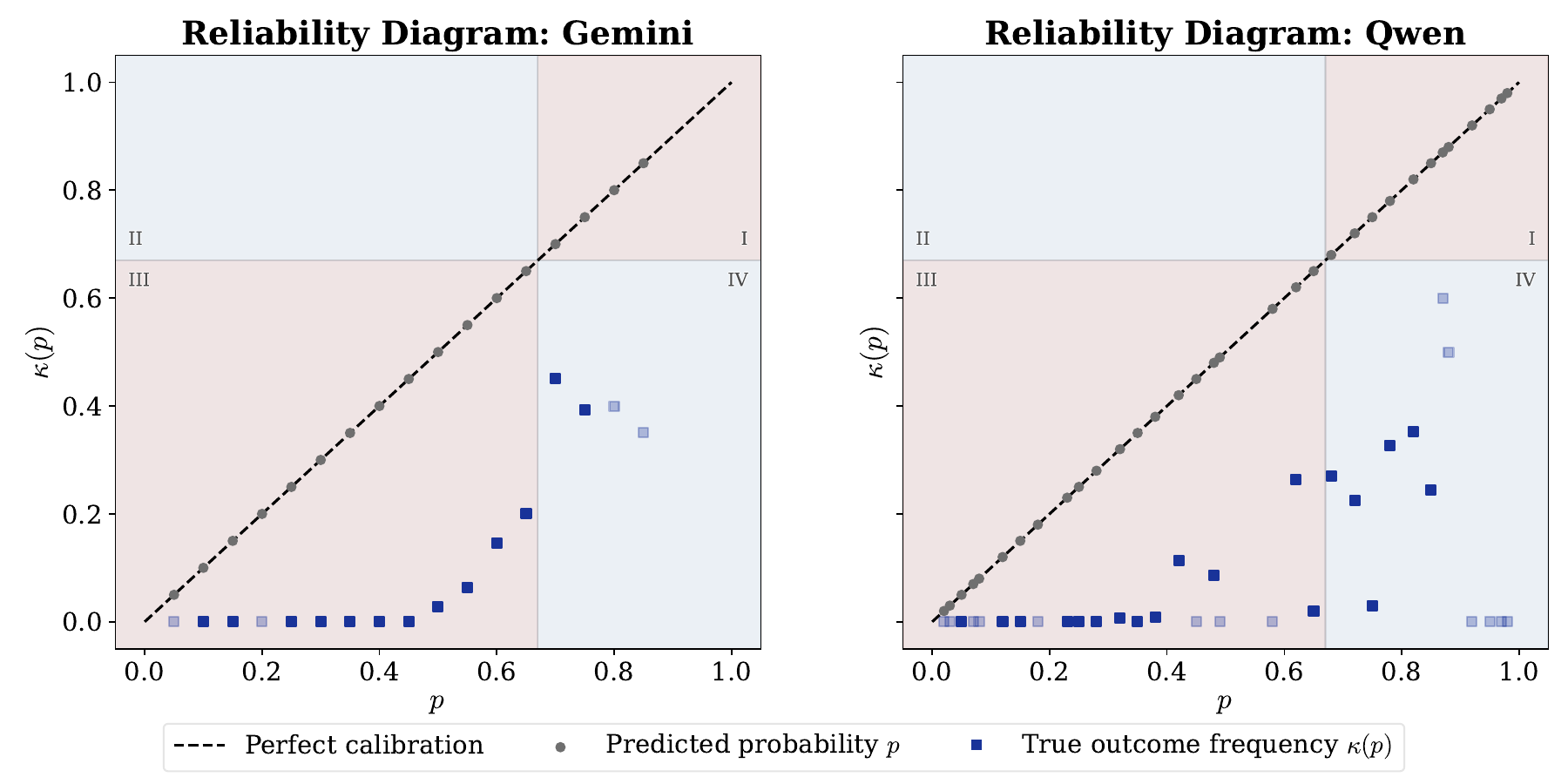}
    \label{single day RD: gemini and qwen}
    }\hspace{0.5mm}
  \subfloat[Bitcoin: \$100 price-increase event]{
    \includegraphics[width=0.75\linewidth]{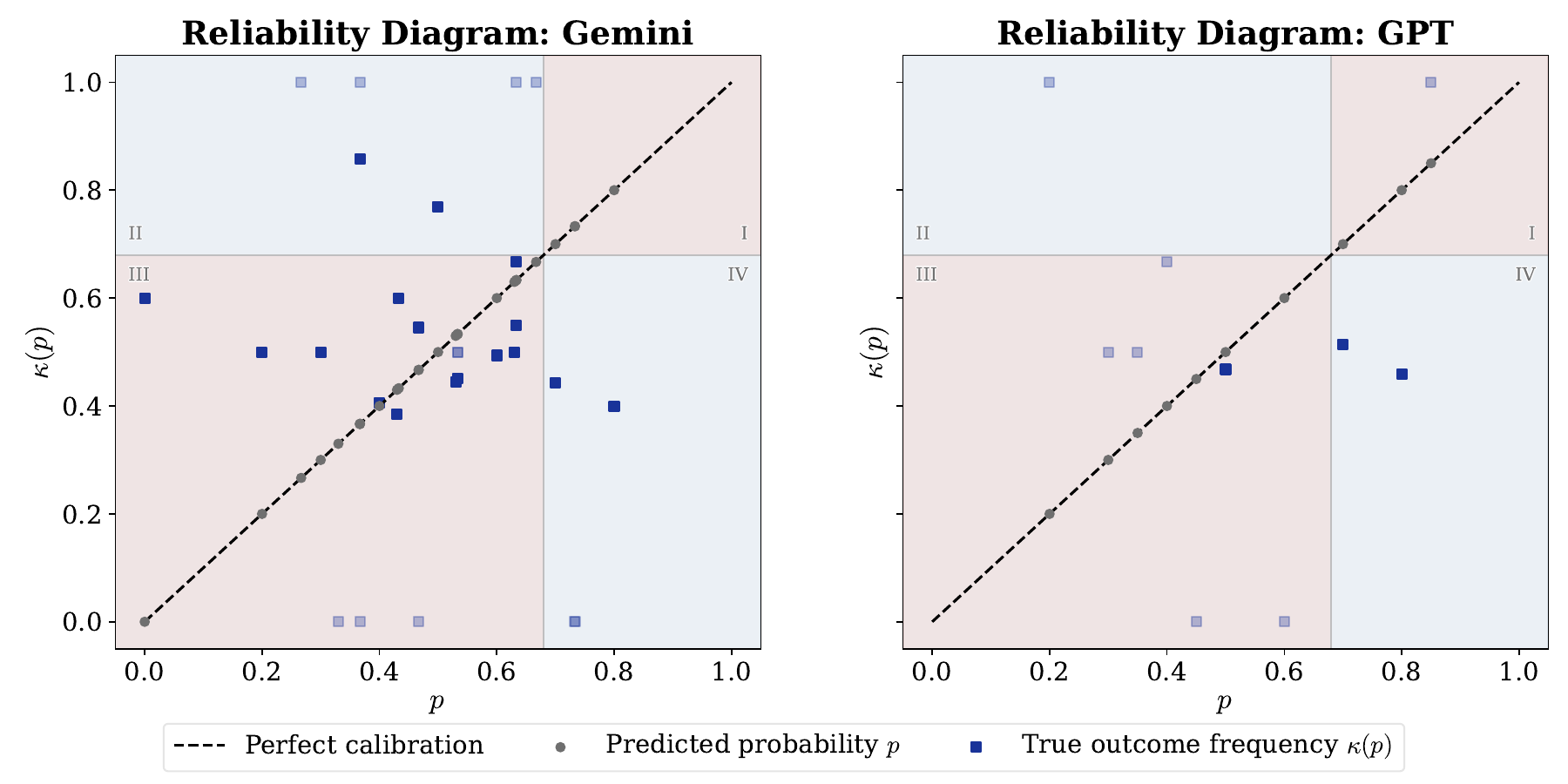}
    \label{abs 100 RD: gemini and gpt}
    }
  \caption{\textbf{Reliability diagram with decision-relevant quadrants.} Each panel plots the true outcome frequency $\truePredic(\prediction)$ (darker squares correspond to high frequency predictions, while the lighter ones represent the less frequent predictions) against predicted probability $\prediction$, and the dashed diagonal indicates perfect calibration ($\truePredic(\prediction)=\prediction$). For a reference threshold $t^*$, we draw a threshold cross at $\prediction=t^*$ and $\truePredic(\prediction)=t^*$, partitioning the unit square into four quadrants (I-IV). Quadrants I and III correspond to decision-consistent regions where $\prediction$ and $\truePredic(\prediction)$ lie on the same side of $t^*$, while Quadrants II and IV are decision-flipping regions where they fall on opposite sides. Specifically, the total prediction frequency of decision-relevant regions (Quadrants I and III) are 0.82 for {\Gemini} and 0.60 for {\Qwen} in Weather setting; 0.75 for {\Gemini} and 0.34 for {\ChatGPT} in Bitcoin setting.}
  \label{reliability diagram}
\end{figure}

\subsection{Experiment II: \InfoGap\ to Evaluate Post-processing Methods}
\label{subsec:post-processing}

Most calibration post-processing methods are designed to reduce ECE. In this experiment, we aim to examine their impact through the lens of {\InfoGap}: specifically, whether post-processing improves or degrades decision-relevant performance as measured by {\InfoGap}, and which method yields the best {\InfoGap} performance when applied to the same underlying probabilistic forecasts.

\xhdr{Post-processing methods and implementation} 
We evaluate six widely used post-processing methods: Bayesian binning, Histogram binning, Isotonic regression, Platt scaling, Scaling binning, and Temperature scaling. 
Each method requires a hold-out validation set to learn its recalibration mapping.In our experiments, we randomly split each dataset into a 45\% calibration set used to fit the post-processing mappings and a 55\% test set on which we evaluate the post-processed predictors.
Implementation details for each method are provided in \Cref{apx:post-processing}.

\xhdr{Impact of post-processing on {\InfoGap}} Figure \ref{infogap curve by methods} summarizes how different post-processing methods affect {\InfoGap}. For any fixed tolerance $\IGtol$,  one method outperforms another if its curve lies higher, meaning that it improves {\InfoGap} under threshold $s$ in a larger fraction of cases. Figure \ref{infogap curve by methods} reveals a clear ranking across methods: temperature scaling stands out as the weakest performer across all thresholds, followed by histogram binning. In contrast, scaling binning consistently yields the highest curve, indicating the most reliable improvement in {\InfoGap}. The remaining three methods - Bayesian binning, Isotonic regression and Platt scaling - are very close when the threshold is small, but become distinguishable as $\IGtol$ grows. It suggests that the relative advantage of these three methods depends on how much threshold we can accept. Notably, even for a moderate threshold, some methods still worsen {\InfoGap} on a substantial subset of cases in our experiments: at $\IGtol=0.05$, temperature scaling exceeds the threshold in nearly half of all the cases, and histogram binning in roughly $30\%$, indicating that post-processing can still degrade {\InfoGap} for many instances.

Overall, Figure \ref{infogap curve by methods} suggests that these post-processing methods for recalibration are not uniformly beneficial from the {\InfoGap} perspective. While scaling binning delivers the most consistent gains, several widely used post-processing methods can degrade {\InfoGap} on a substantial fraction of instances in our experiments, even when the threshold is moderately large. This highlights that selecting post-processing methods solely on ECE may be insufficient, especially when we primarily concern about the performance on decision-level tasks.

\begin{figure}[H]
  \centering
  \includegraphics[width=0.7\textwidth]{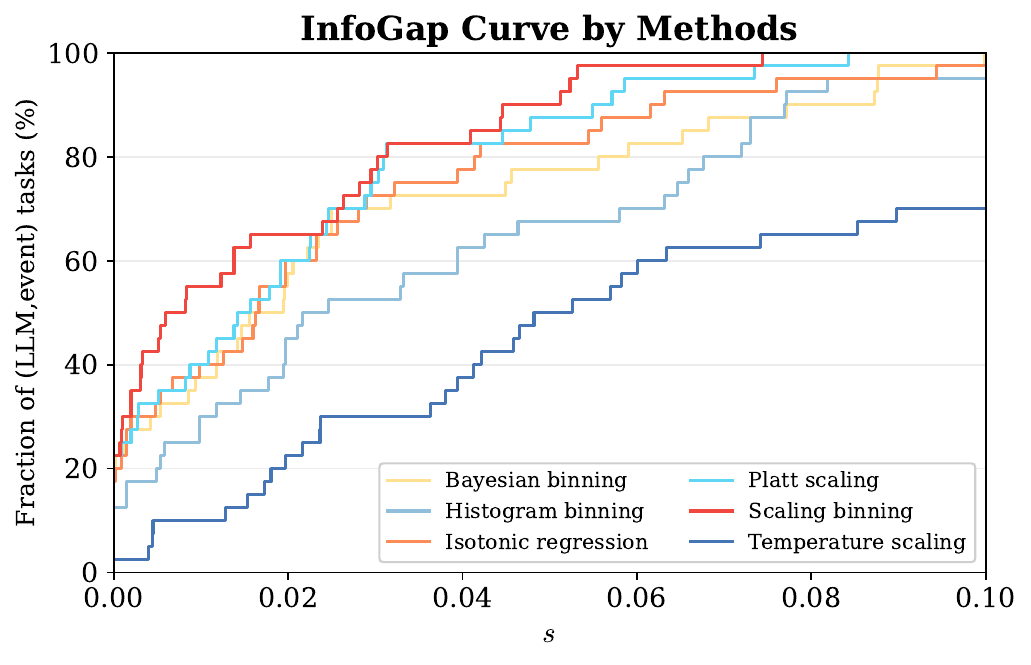}
  \caption{\textbf{Post-processing effects on {\InfoGap}.} For each post-processing method, we compute the directional change $\inforGap{\predictor^{raw},\predictor^{post}}$ for every (LLM, event) forecasting task, where $\predictor^{raw}$ denotes the original probabilities and $\predictor^{post}$ denotes the post-processed probabilities. The curve reports the fraction of LLMs whose $\inforGap{\predictor^{raw},\predictor^{post}}\leq\IGtol$ for each threshold $\IGtol$, and higher curves indicate reduced {\InfoGap} degradation.}
  \label{infogap curve by methods}
\end{figure}

\xhdr{{\InfoGap}-regularized post-processing} In the above experiment, we evaluate six widely used post-processing methods and illustrate how they affect our {\InfoGap}. These methods can be roughly categorized into parametric, non-parametric, and semi-parametric approaches. Parametric methods learn a transformation by optimizing an explicit objective (e.g. Platt scaling and temperature scaling). Non-parametric methods learn flexible mappings, for example via binning or monotone constraint (e.g. Bayesian binning, histogram binning and isotonic regression). Semi-parametric methods combine a scaling step with a subsequent binning step (e.g. scaling-binning). A key observation from the results above is that many post-processing methods designed to reduce calibration errors may simultaneously degrade {\InfoGap}. This raises a natural question that can we introduce a simple modification to these post-processing methods to avoid large {\InfoGap} deterioration?

To address this question, we propose a {\InfoGap}-regularized post-processing method. Since our modification operates by altering the optimizing objective, it is naturally most effective for methods that are fit through objective-based optimization, i.e. methods with a parametric component. Concretely, instead of minimizing negative log-likelihood (NLL) alone, we minimize a weighted combination of NLL and an {\InfoGap} penalty:
\begin{align*}
    \postParamEstim{\alpha}\in\argmin\nolimits_{\postParam\in\proPostParam}\left\{\hyperParam\cdot\nll{\predictor_\postParam^{post}}+(1-\hyperParam)\cdot\inforGap{\predictor^{ref},\predictor_\postParam^{post}}\right\}_\calibData,\quad\hyperParam\in[0,1]
\end{align*}
Here, $\calibData=\{\prediction_i,\outcome_i\}_{i\in[n]}$ denotes the calibration dataset of predicted probabilities and outcomes. The post-processed predictor $\predictor^{post}$ is obtained by applying the post-processing map to the original predictions on $\calibData$. The reference predictor $\predictor^{ref}$ is selected as a Bayes-updated predictor $\predictor^{Bayes}$, which is constructed by replacing each prediction $\prediction_i$ with its corresponding true outcome frequency $\truePredic(\prediction_i)$, yielding the paired data $\{\truePredic(\prediction_i),\outcome_i\}_{i\in[n]}$. The hyper-parameter $\hyperParam$ controls the tradeoff between likelihood-based optimization and decision-relevant informativeness preservation. 

In what follows, we replace the NLL objective with our {\InfoGap}-regularized objective for post-processing methods with a parametric component. In particular, for scaling binning, we implement its scaling step via Platt scaling in our experiment and keep the subsequent discretization step unchanged. The results below compare the original methods with their {\InfoGap}-regularized variants.

\begin{figure}[H]
  \centering
  \subfloat[Temperature scaling]{\includegraphics[width=0.32\linewidth]{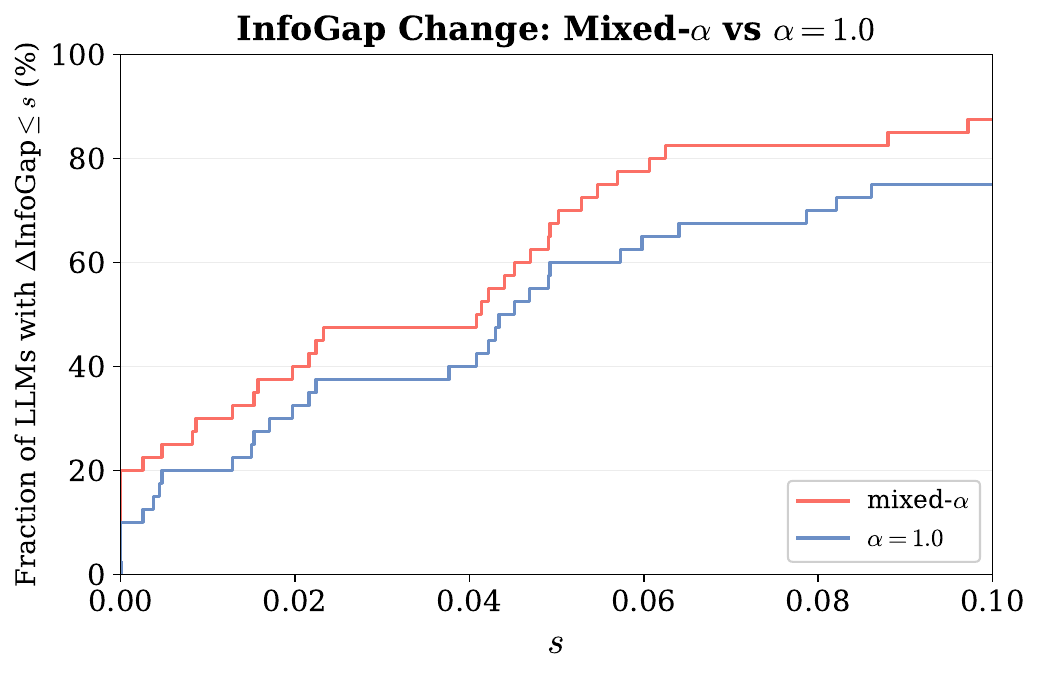}\label{temperature scaling}}\hspace{1mm}
  \subfloat[Platt scaling]{\includegraphics[width=0.32\linewidth]{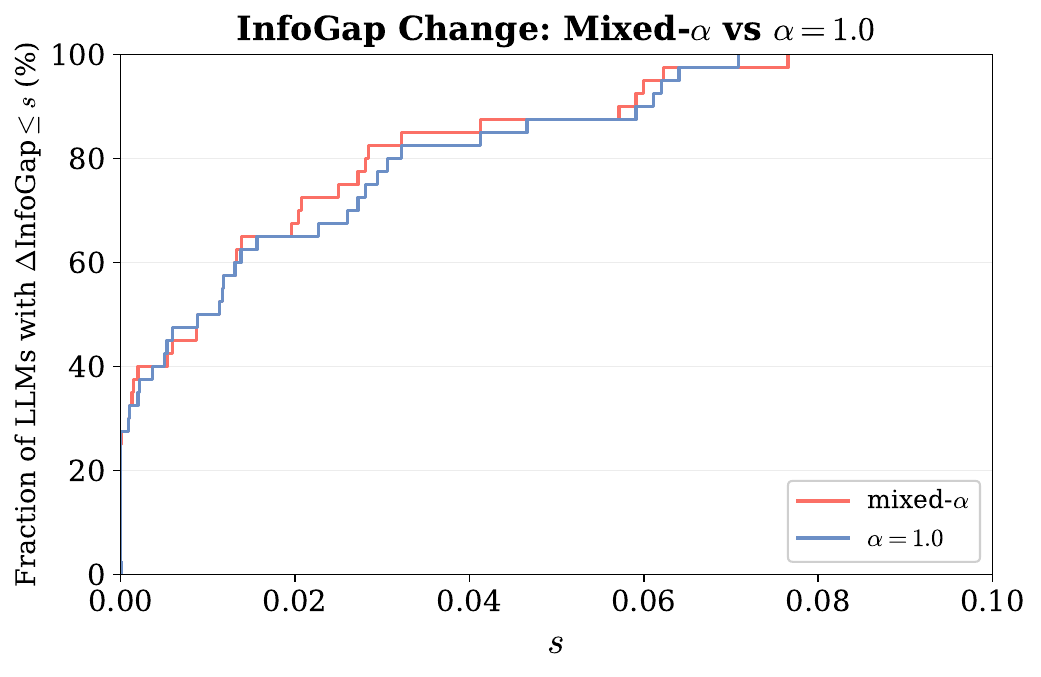}\label{platt scaling}}\hspace{1mm}
  \subfloat[Scaling binning]{\includegraphics[width=0.32\linewidth]{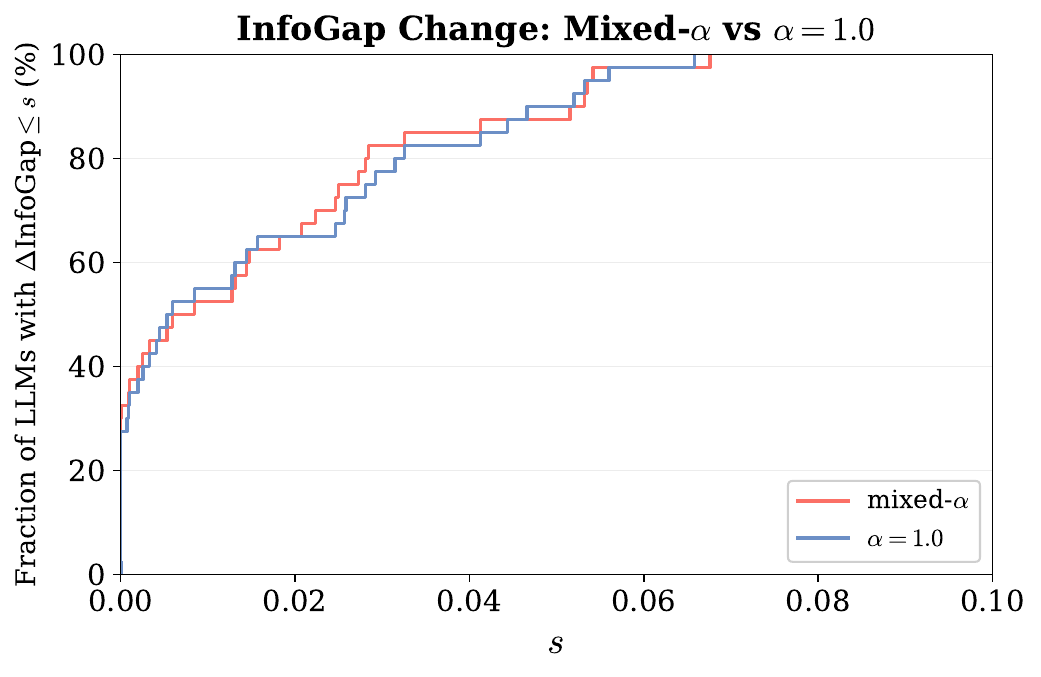}\label{scaling binning}}
  \caption{\textbf{Comparison of standard post-processing with {\InfoGap}-regularized variants.} We compare three post-processing methods trained with the standard NLL-only objective ($\hyperParam=1.0$) against our {\InfoGap}-regularized counterparts (mixed-$\hyperParam$). Here, ``mixed-$\hyperParam$'' means that we allow $\hyperParam$ to be chosen separately for each event, resulting in an event-specific set of $\hyperParam$.}
  \label{post-processing improvement}
\end{figure}

We start with temperature scaling, whose one-parameter optimization provides the most transparent testbed for comparing our {\InfoGap}-regularized method with the original post-processing method. In Figure \ref{temperature scaling}, the mixed-$\hyperParam$ curve consistently lies above the $\hyperParam=1.0$ (NLL-only) baseline across all the entire range of $\IGtol$, indicating a systematic reduction in {\InfoGap} harm across all LLMs. And for Platt scaling, which optimizes two parameters, the improvement is more moderate and depends on the threshold: the mixed-$\hyperParam$ curve exceeds than the baseline over a subset of $\IGtol$, meaning that for these thresholds our method yields less {\InfoGap} degradation for more LLMs, while the two approaches behave similarly over other regions. Finally, scaling binning adds an additional discretization step on top of the scaling step. Correspondingly, the two curves are close for many thresholds, yet the mixed-$\hyperParam$ variant still dominates the baseline for some $\IGtol$ intervals. Overall, Figure \ref{post-processing improvement} shows that {\InfoGap}-regularized method can mitigate the {\InfoGap} deterioration caused by standard post-processing, with the clearest gains observed for temperature scaling.

More broadly, the {\InfoGap}-regularized objective provides a simple and effective way to make post-processing more robust from a decision-relevant perspective. Across the three methods with an explicit optimization step, mixed-$\hyperParam$ strategy consistently shifts the {\InfoGap} curves upward, indicating reduced {\InfoGap} harm across LLMs. The effect is strongest for temperature scaling and weaker for scaling binning, which possibly suggest that the remaining discretization step can partially limit the performance of objective optimization. This points to an interesting direction of extending {\InfoGap}-relevant design beyond the scaling step, for example, to discretization choices or other nonparametric components.

\bibliographystyle{plainnat}
\bibliography{mybib.bib}

\newpage
\appendix

\section{Missing Proofs in \texorpdfstring{\Cref{sec:prelim}}{Section 2}}

\subsection{Proof of Proposition~\ref{prop:prelim:informativeness gap captures blackwell informativeness}}
\label{apx:prelim:informativeness gap captures blackwell informativeness proof}
\propInfoGAPandBW*
\begin{proof}
    The only if direction is trivial since the definition of Blackwell informativeness considers all decision problems. 
    
    For the if direction, 
    It suffices to show for all (possibly unbounded) convex (univariate) interim utility functions the expected utility under predictor $\predictor$ is weakly higher than the expected utility under predictor $\predictornew$. To see this, fix an arbitrary (possibly unbounded) convex interim utility function $\indirectU$. Define bounded convex interim utility function $\indirectU^{(k)}$ which is obtained by truncating the derivative of function $\indirectU$ at a large constant $k$ from top. Consider the sequence of bounded convex interim utility function $(\indirectU^{(k)})$ as $k$ goes to infinite. Observe that $\inforGap{\predictornew,\predictor} = 0$ implies that for every bounded convex interim utility function $\indirectU^{(k)}$, the expected utility under predictor $\predictor$ is weakly higher than the expected utility under predictor $\predictornew$, i.e., $\expect[\prediction\sim\PDF]{\indirectU^{(k)}(\prediction)} \geq \expect[\prediction\sim\PDFNew]{\indirectU^{(k)}(\prediction)}$. Invoking the dominated convergence theorem, we obtain $\lim_{k\rightarrow\infty}\expect[\prediction\sim\PDF]{\indirectU^{(k)}(\prediction)} = \expect[\prediction\sim\PDF]{\indirectU(\prediction)}$ and $\lim_{k\rightarrow\infty}\expect[\prediction\sim\PDFNew]{\indirectU^{(k)}(\prediction)} = \expect[\prediction\sim\PDFNew]{\indirectU(\prediction)}$. Hence, $\expect[\prediction\sim\PDF]{\indirectU(\prediction)} \geq \expect[\prediction\sim\PDFNew]{\indirectU(\prediction)}$. This completes the proof of \Cref{prop:prelim:informativeness gap captures blackwell informativeness}.
\end{proof} 
\subsection{Proof of Lemma~\ref{lem:Lipschitz convex indirect utility}}
\label{apx:Lipschitz convex indirect utility proof}
\lemLipConvexIndirectUtil*
\begin{proof}  
    According to \citet{KLST-23,HW-24}, for every decision problem (of the agent), its interim utility function is equivalent to a proper scoring function,
    which is piecewise linear with maximum absolute slope of $\max_{\prediction\in[0, 1]} |\tildeagentU(\prediction, 1) - \tildeagentU(\prediction, 0)|$. 
    The bounded utility difference $\max_{\prediction\in[0, 1]} |\tildeagentU(\prediction, 1) - \tildeagentU(\prediction, 0)| \le 1$ is equivalent to the 1-Lipschitzness of the univariate form of the interim utility function $\indirectU$. Finally, recall that a scoring function is proper if and only if its univariate form is convex \citep{Mcc-56}. This completes the proof of \Cref{lem:Lipschitz convex indirect utility}.
\end{proof}

\section{Missing Proofs in \texorpdfstring{\Cref{sec:misc}}{Section 4}}

\subsection{Proof of \texorpdfstring{\Cref{thm:strong-duality-miscalibrated}}{Theorem 4.1}}
\label{subsec:strong duality miscalibrated proof}
\label{apx:idf miscali proof}

Similar to the proof of \Cref{thm:strong-duality-perfectly-calibrated}, we first reformulate the program \eqref{eq:primal} as a program with treating the univariate form of indirect utility as the variables. 
\begin{lemma}
\label{lem:simplify miscali}
For any two perfectly calibrated predictors $\predictor, \predictornew$, their informativeness gap satisfies
\begin{align*}
    \inforGap{\predictor, \predictornew}= 
    \OBJ{\ProGOneMisC{\PDF,\PDFNew,\truePredic_\predictor,\truePredic_\predictornew}}
\end{align*}
where $\OBJ{\ProGOneMisC{\PDF,\PDFNew,\truePredic_\predictor,\truePredic_\predictornew}}$ is the optimal objective value of the following program with inputs $\dist\gets \PDF$, $\distNew\gets\PDFNew$, $\truePredic_1\gets \truePredic_\predictor$ and $\truePredic_2 \gets \truePredic_\predictornew$ (i.e., equal to two predictors' prediction distributions $\PDF,\PDFNew$ and true expected outcome functions $\truePredic_\predictor,\truePredic_\predictornew$):
\begin{align}
\label{eq:primal miscali}
    \arraycolsep=5.4pt\def\arraystretch{1}
    \tag{$\textsc{P}^\textsc{MisC}_1[\dist,\distNew,\truePredic_1,\truePredic_2]$}
    \begin{array}{rll}
    \sup
    \limits_{\indirectU, \delta:[0,1]\rightarrow\reals} 
    ~ &
    \displaystyle 
    \int_0^1
    \indirectU(\prediction) \cdot \left(\dist(\prediction) - \distNew(\prediction)\right)
    + \delta(\prediction)\cdot \left(
    (\truePredic_1(\prediction) - \prediction)\cdot \dist(\prediction) 
     - (\truePredic_2(\prediction) - \prediction) \cdot\distNew(\prediction) \right) \, \dd \prediction
    \vspace{1mm} & \text{s.t.}
    \\
     & 
    \delta(\prediction) \in[-1, 1]
    &  \prediction\in[0, 1]
    \vspace{1mm}
    \\
    & 
    \displaystyle 
    \indirectU(\prediction) - \indirectU(\predictionNew) \le \delta(\prediction) \cdot (\prediction - \predictionNew)
    &  \prediction, \predictionNew\in[0, 1]
    \vspace{1mm}
    \end{array}
\end{align}
\end{lemma}
\begin{proof}
Invoking \Cref{lem:interim utility expression} and the definition of $\inforGap{\predictor,\predictornew}$, the objective function in program~\eqref{eq:primal} can be rewritten as follows: 
\begin{align*}
    & \int_0^1
    \indirectU(\truePredic_\predictor(\prediction); \prediction) \cdot \PDF(\prediction) - 
    \indirectU(\truePredic_\predictornew(\prediction); \prediction) \cdot \PDFNew(\prediction) ~\dd  \prediction \\
    = ~ &  
    \int_0^1
    \left(\truePredic_\predictor(\prediction) \cdot \delta(\prediction) + \tildeagentU(\prediction, 0)\right)
    \cdot \PDF(\prediction) - 
    \left(\truePredic_\predictornew(\prediction) \cdot \delta(\prediction) + \tildeagentU(\prediction, 0)\right) \cdot \PDFNew(\prediction) ~\dd  \prediction \\
    = ~ &  
    \int_0^1
    \indirectU(\prediction)\cdot \left(\PDF(\prediction) - \PDFNew(\prediction)\right) ~\dd  \prediction 
    + 
    \int_0^1 \delta(\prediction)\cdot \left(\truePredic_\predictor(\prediction) - \prediction\right) \cdot \dist(\prediction) ~\dd  \prediction
    - 
    \int_0^1 \delta(\prediction)\cdot \left(\truePredic_\predictornew(\prediction) - \prediction\right)
    \cdot\distNew(\prediction) ~\dd  \prediction~,
\end{align*}
where function $\delta(\prediction) \triangleq \tildeagentU(\prediction, 1) - \tildeagentU(\prediction, 0)$.
Invoking \Cref{lem:Lipschitz convex indirect utility} and the definition of $\inforGap{\predictor,\predictornew}$, both constraints are satisfied due to the 1-Lipschitzness and convexity of interim utility $\indirectU$, respectively. This complete the proof of \Cref{lem:simplify miscali}.
\end{proof}

With \Cref{lem:simplify miscali}, we are ready to prove \Cref{thm:strong-duality-miscalibrated}. 
Unlike the proof of \Cref{thm:strong-duality-perfectly-calibrated} where we established the strong duality by deriving a dual representation of program~\eqref{opt:EMD}, in the proof of \Cref{thm:strong-duality-miscalibrated}, we instead derive a dual of program~$\ProGOneMisC{\PDF,\PDFNew,\truePredic_\predictor,\truePredic_\predictornew}$ and we show that such dual is essentially program \eqref{opt:flow miscali}.
\begin{proof}[Proof of \Cref{thm:strong-duality-miscalibrated}]
Our proof proceeds with two main steps.

\xhdr{Step 1: Prove that $ \inforGap{\predictor, \predictornew}
 = 
\earthDistFlow{\PDF, \PDFNew, \truePredic_\predictor,\truePredic_\predictornew}$}
To see this, we consider the following Lagrange of program~$\ProGOneMisC{\PDF,\PDFNew,\truePredic_\predictor,\truePredic_\predictornew}$
with Lagrange multipliers $\lambda:[0,1]\times[0,1]\rightarrow\reals_+$,
\begin{align*}
    \mathcal{L} 
    & = 
    \int_0^1
    \indirectU(\prediction) \cdot (\PDF(\prediction) - \PDFNew(\prediction))
     ~\dd  \prediction  
     + \int_0^1 (\truePredic_\predictor(\prediction) - \prediction) \cdot \delta(\prediction)\cdot \PDF(\prediction)~\dd \prediction 
     - \int_0^1 (\truePredic_\predictornew(\prediction) - \prediction) \cdot \delta(\prediction)\cdot \PDFNew(\prediction)~\dd \prediction \\
     & \quad + \int_0^1 \int_0^1 \lambda(\prediction, \predictionNew) \cdot \left(
     \delta(\prediction) \cdot (\prediction-\predictionNew) - (\indirectU(\prediction) - \indirectU(\predictionNew))\right) ~\dd \prediction\dd \predictionNew \\
     & = 
     \int_0^1
    \indirectU(\prediction) \cdot \left(\PDF(\prediction) - \PDFNew(\prediction) - \int_0^1 \lambda(\prediction, \predictionNew)\,\dd \predictionNew + \int_0^1 \lambda(\predictionNew, \prediction)\, \dd \predictionNew\right)\, \dd \prediction \\
    & \quad 
    + \int_0^1 (\truePredic_\predictor(\prediction) - \prediction) \cdot \delta(\prediction)\cdot \PDF(\prediction)~\dd \prediction
    - \int_0^1 (\truePredic_\predictornew(\prediction) - \prediction) \cdot \delta(\prediction)\cdot \PDFNew(\prediction)~\dd \prediction 
    + 
    \int_0^1 \int_0^1 \lambda(\prediction, \predictionNew) \cdot 
    \delta(q) \cdot (\predictionNew - \prediction)~\dd \prediction\dd \predictionNew
\end{align*}
By the first-order condition, the optimal dual variable $\lambda^*$ satisfies  
\begin{align}
    \label{eq:flow first-order}
    \PDF(\prediction) - \PDFNew(\prediction) - \int_0^1 \lambda^*(\prediction, \predictionNew)\,\dd \predictionNew + \int_0^1 \lambda^*(\predictionNew, \prediction)\, \dd \predictionNew = 0
\end{align}
for all $\prediction\in[0,1]$.
Thus, under optimal dual variable $\lambda^*$, we can simplify $\mathcal{L}$ as follows:
\begin{align*}
    \mathcal{L} 
    & = 
    \int_0^1 \delta(\prediction) \cdot \left(
    \int_0^1 \lambda^*(\prediction, \predictionNew)  \cdot (\prediction-\predictionNew) ~\dd \predictionNew
    \right)\dd \prediction
    + 
    \int_0^1 (\truePredic_\predictor(\prediction) - \prediction) \cdot \delta(\prediction)\cdot  \PDF(\prediction)~\dd \prediction
    - 
    \int_0^1 (\truePredic_\predictornew(\prediction) - \prediction) \cdot \delta(\prediction)\cdot \PDFNew(\prediction)~\dd \prediction \\
    & = 
    \int_0^1 \delta(\prediction) \cdot \left(
    \int_0^1 \lambda^*(\prediction, \predictionNew)  \cdot (\prediction-\predictionNew) ~\dd \predictionNew
    + (\truePredic_\predictor(\prediction) - \prediction) \cdot \PDF(\prediction)
    - (\truePredic_\predictornew(\prediction) - \prediction) \cdot \PDFNew(\prediction)
    \right)\dd \prediction
\end{align*}
Since $\delta(\prediction)\in [-1, 1]$, under optimal $\delta^*$, it satisfies 
\begin{align*}
    \mathcal{L} = 
    \int_0^1 \abs{
    \int_0^1 \lambda^*(\prediction, \predictionNew)  \cdot (\prediction-\predictionNew) ~\dd \predictionNew
    + (\truePredic_\predictor(\prediction) - \prediction) \cdot \PDF(\prediction)
    - (\truePredic_\predictornew(\prediction) - \prediction) \cdot \PDFNew(\prediction)
    }\dd \prediction~,
\end{align*}
which is equivalent to the objective function in program \eqref{opt:flow miscali}.
The flow constraint in program \eqref{opt:flow miscali} is equivalent to the first-order condition that we derived in Equation~\eqref{eq:flow first-order}.
Thus, due to \Cref{lem:simplify miscali}, program \eqref{opt:flow miscali} is indeed the dual of program~$\ProGOneMisC{\PDF,\PDFNew,\truePredic_\predictor,\truePredic_\predictornew}$. 

\xhdr{Step 2: Establish a {\generalizedSCDF} representation for the $\earthDistFlow{\dist, \distNew,\truePredic_1,\truePredic_2}$}
We next establish derive a {\generalizedSCDF} representation for the $\earthDistFlow{\dist, \distNew,\truePredic_1,\truePredic_2}$.
We summarize this step as the following lemma. 
\begin{restatable}[{\generalizedSCDF} representation]{lemma}{lemIDFMiscali}
\label{lem:idf miscali}
Given any two distributions $\dist, \distNew$ with supports $\supp(\dist)\subseteq[0, 1], \supp(\distNew)\subseteq[0, 1]$ and two functions $\truePredic_1,\truePredic_2:[0,1]\rightarrow[0,1]$, if $\expect[\prediction\sim\dist]{\truePredic_1(\prediction)} = \expect[\prediction\sim\distNew]{\truePredic_2(\prediction)}$,
the informativeness measure $\earthDistFlow{\dist, \distNew,\truePredic_1,\truePredic_2}$ satisfies
\begin{align*}
    \earthDistFlow{\dist, \distNew,\truePredic_1,\truePredic_2} = 
    2\cdot \max_{t\in [0, 1]} ~
    \left(\distSCDF(t) + \int_0^t (\prediction - \truePredic_1(\prediction))\cdot \dist(\prediction)\, \dd \prediction \right)
    - 
    \left(\distSCDFNew(t) + \int_0^t (\prediction - \truePredic_2(\prediction)) \cdot \distNew(\prediction)\, \dd \prediction\right)
\end{align*}
where $\distSCDF(t) \triangleq \int_0^t \distCDF(\prediction)\,\dd\prediction$ and $\distSCDFNew(t) \triangleq \int_0^t \distCDFNew(\prediction)\,\dd\prediction$ are the SCDFs of distributions $\dist,\distNew$, respectively.
\end{restatable}
We below prove \Cref{lem:idf miscali}.
    Invoking previous step and \Cref{lem:simplify miscali}, we obtain the identity between $\earthDistFlow{\dist, \distNew,\truePredic_1,\truePredic_2}$ and the optimal objective of program~$\ProGOneMisC{\dist,\distNew,\truePredic_1,\truePredic_2}$, i.e.,
    \begin{align*}
        \earthDistFlow{\dist, \distNew,\truePredic_1,\truePredic_2} = \OBJ{\ProGOneMisC{\dist,\distNew,\truePredic_1,\truePredic_2}}
    \end{align*}
    Define auxiliary function $\omega_{\dist,\distNew}:[0,1]\rightarrow\reals$ as 
    \begin{align*}
        \forall t\in[0,1]:\qquad
        \omega_{\dist,\distNew}(t) \triangleq \displaystyle\int_0^t 
        \dist(\prediction) - \distNew(\prediction)\,\dd \prediction
    \end{align*}
    The objective function of program~$\ProGOneMisC{\dist,\distNew,\truePredic_1,\truePredic_2}$ can be reformulated as follows:
    \begin{align*}
        & \int_0^1
        \indirectU(\prediction) \cdot \left(\dist(\prediction) - \distNew(\prediction)\right)
        + \delta(\prediction)\cdot \left(
        (\truePredic_1(\prediction) - \prediction)
        \cdot \dist(\prediction) 
         - (\truePredic_2(\prediction) - \prediction) \cdot \distNew(\prediction) \right) \, \dd \prediction \\
        = {} & \left[\indirectU(\prediction)\cdot\omega_{\dist,\distNew}(\prediction)\right]_0^1 - \int_0^1 \omega_{\dist,\distNew}(\prediction)\cdot\indirectUDeriv(\prediction)\, \dd \prediction + 
        \int_0^1 \delta(\prediction)\cdot \left(
        (\truePredic_1(\prediction) - \prediction) \cdot \dist(\prediction) 
        - (\truePredic_2(\prediction) - \prediction) 
        \cdot \distNew(\prediction) \right) \, \dd \prediction\\
        = {} & 
        - \int_0^1 \omega_{\dist,\distNew}(\prediction) \cdot \indirectUDeriv(\prediction)\, \dd \prediction + 
        \int_0^1 \delta(\prediction)\cdot \left(
        (\truePredic_1(\prediction) - \prediction) \cdot \dist(\prediction) 
        - (\truePredic_2(\prediction) - \prediction) \cdot \distNew(\prediction) \right) \, \dd \prediction~.
        \\
        = {} & 
        - \int_0^1 \omega_{\dist,\distNew}(\prediction) \cdot \indirectUDeriv(\prediction)\, \dd \prediction + 
        \int_0^1 \indirectUDeriv(\prediction)\cdot \left(
        (\truePredic_1(\prediction) - \prediction) \cdot \dist(\prediction) 
        - (\truePredic_2(\prediction) - \prediction) \cdot \distNew(\prediction) \right) \, \dd \prediction~.
    \end{align*}
    where the first equality holds due to integration by parts, and the second equality considers derivative $\indirectUDeriv(\prediction) \triangleq \indirectU'(\prediction)$ which exists almost everywhere and uses the fact that $\omega_{\dist,\distNew}(0) = \omega_{\dist,\distNew}(1) = 0$. 
    To understand the third equality, note that constraint $\indirectU(\prediction) - \indirectU(\predictionNew) \le \delta(\prediction)\cdot(\prediction-\predictionNew)$ for all $\prediction, \predictionNew\in[0, 1]$ in program~$\ProGOneMisC{\dist,\distNew,\truePredic_1,\truePredic_2}$ implies that 
    \begin{itemize}
        \item function $\delta$ is non-decreasing. 
        To see this, fix any $0\le \prediction < \predictionNew\le 1$. Observe that for every pair $(\prediction, \predictionNew)$, $\indirectU(\prediction) - \indirectU(\predictionNew) \le \delta(\prediction)\cdot(\prediction-\predictionNew)$, while for pair $(\predictionNew, \prediction)$, $\indirectU(\predictionNew) - \indirectU(\prediction) \le \delta(\predictionNew)\cdot(\predictionNew-\prediction)$.
        Thus, combining these two inequalities we have 
        \begin{align*}
            \delta(\prediction) \le \frac{\indirectU(\predictionNew) - \indirectU(\prediction)}{\predictionNew - \prediction} \le \delta(\predictionNew)~,
        \end{align*}
        which shows the monotonicity of function $\delta$.
    
        \item Furthermore, it guarantees that
        for every $\prediction \in [0, 1]$, $\delta(\prediction)\in [\indirectU'_-(\prediction), \indirectU'_+(\prediction)]$ where $\indirectU'_-(\prediction)$ and $\indirectU'_+(\prediction)$ are the left and right derivative of function $\indirectU$ at $\prediction$.  
        This can be verified by plugging in $(\prediction, \prediction-\varepsilon)$ and $(\prediction+\varepsilon, \prediction)$ for sufficiently small $\varepsilon$ in the constraint $\indirectU(\prediction) - \indirectU(\predictionNew) \le \delta(\prediction)\cdot(\prediction-\predictionNew)$.
        Since $\delta\in [\indirectU'_-(\prediction), \indirectU'_+(\prediction)]$ and it is monotone,  $\indirectU$ is a convex function, and thus $\indirectU'_-(\prediction) = \indirectU'_+(\predictionNew)$ almost everywhere except on a set of Lebesgue measure $0$. Thus, functions $\indirectUDeriv = \delta$ almost everywhere.
    \end{itemize}
    Consequently, program~$\ProGOneMisC{\dist,\distNew,\truePredic_1,\truePredic_2}$ is equivalent to the following program:
    \begin{align*}
        \arraycolsep=5.4pt\def\arraystretch{1}
        &\begin{array}{rlll}
        \inf
        \limits_{\indirectUDeriv:[0,1]\rightarrow\reals} 
        ~ &
        \displaystyle 
        \int_0^1
        \indirectUDeriv(\prediction) \cdot \left(
        \omega_{\dist,\distNew}(\prediction) - 
        (\truePredic_1(\prediction) - \prediction)\cdot  \dist(\prediction) 
        + (\truePredic_2(\prediction) - \prediction) \cdot \distNew(\prediction)
        \right) ~\dd  \prediction 
        \quad & \text{s.t.} &
        \vspace{1mm}
        \\
        & 
        \displaystyle 
        \indirectUDeriv(\prediction)\in[-1,1], 
        &  \prediction\in[0, 1]
        \vspace{1mm}
        \\
        & 
        \displaystyle 
        \indirectUDeriv(\prediction) \le \indirectUDeriv(\predictionNew) 
        &  \prediction\in[0, 1], \predictionNew\in[p, 1]
        \vspace{1mm}
        \end{array}
    \end{align*}
    According to \cite{S-06,B-15}, the extreme points of the uniformly bounded, non-decreasing functions are step functions with only one jump. Thus, the optimal solution $\indirectUDeriv^*$ to the above (infinite-dimensional) linear program must be a step function where 
    \begin{align*}
        \indirectUDeriv^*(\prediction) = -1 \cdot \indicator{\prediction\le t^*} + 1 \cdot\indicator{\prediction> t^*}
    \end{align*}
    for some threshold $t^*\in[0, 1]$.

    With the above observation, to solve program~$\ProGOneMisC{\dist,\distNew,\truePredic_1,\truePredic_2}$, it suffices to optimize over all V-shaped functions across different kinks $t\in[0, 1]$.
    In particular, we can write the objective in program~$\ProGOneMisC{\dist,\distNew,\truePredic_1,\truePredic_2}$  under a V-shaped function with kink $t\in[0, 1]$ as follows: 
    \begin{align*}
        &\int_0^t -(\prediction-t) \cdot (\dist(\prediction) - \distNew(\prediction))\, \dd \prediction 
        + 
        \int_{t}^1 (\prediction-t) \cdot (\dist(\prediction) - \distNew(\prediction))\, \dd \prediction
        \\
        &\quad -
        \int_0^t 
        (\truePredic_1(\prediction) - \prediction)
        \cdot \dist(\prediction) 
         - (\truePredic_2(\prediction) - \prediction) \cdot \distNew(\prediction)  \, \dd \prediction
        +
        \int_t^1 
        (\truePredic_1(\prediction) - \prediction)
        \cdot \dist(\prediction) 
         - (\truePredic_2(\prediction) - \prediction) \cdot \distNew(\prediction)  \, \dd \prediction
        \\
        = {} &
        \int_0^t -2(\prediction-t) \cdot (\dist(\prediction) - \distNew(\prediction))\, \dd \prediction 
        + 
        \int_{0}^1 (\prediction-t) \cdot (\dist(\prediction) - \distNew(\prediction))\, \dd \prediction
        \\
        &\quad -
        \int_0^t 2\left(
        (\truePredic_1(\prediction) - \prediction)
        \cdot \dist(\prediction) 
         - (\truePredic_2(\prediction) - \prediction) \cdot \distNew(\prediction) \right) \, \dd \prediction
        +
        \int_0^1 
        (\truePredic_1(\prediction) - \prediction)
        \cdot \dist(\prediction) 
         - (\truePredic_2(\prediction) - \prediction) \cdot \distNew(\prediction)  \, \dd \prediction
        \\
        = {}&
        \int_0^t -2(\prediction-t) \cdot (\dist(\prediction) - \distNew(\prediction))\, \dd \prediction~
        -
        \int_0^t 2\left(
        (\truePredic_1(\prediction) - \prediction)
        \cdot \dist(\prediction) 
         - (\truePredic_2(\prediction) - \prediction) \cdot \distNew(\prediction) \right) \, \dd \prediction
    \end{align*}
    Here the second equality holds since $\int_{0}^1 (\prediction-t) \cdot (\dist(\prediction) - \distNew(\prediction))\, \dd \prediction
    +
    \int_0^1 
    (\truePredic_1(\prediction) - \prediction)
    \cdot \dist(\prediction) 
    - (\truePredic_2(\prediction) - \prediction) \cdot \distNew(\prediction)  \, \dd \prediction = 0$, which is implied by the statement assumption that distributions $\dist,\distNew$ are valid distribution and have the same mean.

    Observe that 
    \begin{align*}
        \frac{\partial}{\partial t}\int_0^t -2(\prediction-t) \cdot (\dist(\prediction) - \distNew(\prediction))\, \dd \prediction
        &=
        \int_0^t 2  (\dist(\prediction) - \distNew(\prediction))\,\dd\prediction
        -
        2(t - t)\cdot (\dist(t) - \distNew(t)) 
        \\
        &= 
        2(\distCDF(t) - \distCDFNew(t))
    \end{align*}
    Thus, the optimal objective value of program~$\ProGOneMisC{\dist,\distNew,\truePredic_1,\truePredic_2}$ can be further reformulated as 
    \begin{align*}
        &\max\nolimits_{t\in[0, 1]}~
        \int_0^t 2(\distCDF(\prediction) - \distCDFNew(\prediction))\,\dd \prediction
        -
        \int_0^t 2\left(
        (\truePredic_1(\prediction) - \prediction)
        \cdot \dist(\prediction) 
         - (\truePredic_2(\prediction) - \prediction) \cdot \distNew(\prediction) \right) \, \dd \prediction
        \\
        = {} & 
    2\cdot \max_{t\in [0, 1]} ~
    \left(\distSCDF(t) + \int_0^t (\prediction - \truePredic_1(\prediction))\cdot \dist(\prediction)\, \dd \prediction \right)
    - 
    \left(\distSCDFNew(t) + \int_0^t (\prediction - \truePredic_2(\prediction)) \cdot \distNew(\prediction)\, \dd \prediction\right)
    \end{align*}
    where the equality holds due to SCDF definition. This completes the proof of \Cref{lem:idf miscali}.

Combining above two steps would finish the proof of \Cref{thm:strong-duality-miscalibrated}.
\end{proof}

\subsection{Proof of \texorpdfstring{\Cref{prop:bad approx of marg}}{}}
To prove \Cref{prop:bad approx of marg}, we first show the following lower bound.

\begin{lemma}
\label{lem:lower bound bar-REMDMISCB}
For the predictors $\predictor, \predictornew$ defined as in \Cref{ex:unbounded gap miscalibrated predictor measure}, we have that 
$\REMDNew{\PDF, \PDFNew,\truePredic_\predictor,\truePredic_\predictornew}
\geq 4 \int_0^{\frac{1}{2}} t\cdot \PDFNew(t)\,\dd t$.
\end{lemma}
\begin{proof}
    We consider Lagrange multipliers $\alpha_1, \alpha_2:[0,1]\rightarrow\reals$, and $\beta:[0,1]\times[0,1]\rightarrow\reals_+$ for the optimization program in \Cref{prop:bad approx of marg}:
\begin{align*}
    \mathcal{L}
    & = 
    \int_0^1 \gamma(p)\cdot\left(\int_0^1 \miscali(\prediction, \predictionNew)\cdot (\prediction-\predictionNew) \, \dd \predictionNew 
    + (\truePredic_1(\prediction) - \prediction) \cdot \dist(\prediction)
    - (\truePredic_2(\prediction) - \prediction) \cdot \distNew(\prediction)
    \right)  \, \dd \prediction \\
    & \quad 
    - \int_0^1\int_0^1 \beta(\prediction, \predictionNew)\cdot \miscali(\prediction, \predictionNew) 
    \, \dd \predictionNew  \dd \prediction \\
    & \quad 
    + 
    \int_0^1 \alpha_1(\prediction) \cdot\left(\dist(\prediction) - \int_0^1 \miscali(\prediction, \predictionNew) \, \dd \predictionNew\right)\, \dd \prediction
    + 
    \int_0^1 \alpha_2(\predictionNew) \cdot\left(\distNew(\predictionNew) - \int_0^1 \miscali(\prediction, \predictionNew) \, \dd \prediction\right)\, \dd \predictionNew \\
    & = 
    \int_0^1 \gamma(\prediction) \cdot \prediction \cdot \dist(\prediction)
    \, \dd \prediction
    -
    \int_0^1 \gamma(\prediction) \cdot
    \left(
    \int_0^1 \predictionNew\cdot \miscali(\prediction, \predictionNew) \, \dd \predictionNew
    \right) \, \dd \prediction 
    - \int_0^1\int_0^1 \beta(\prediction, \predictionNew)\cdot \miscali(\prediction, \predictionNew) 
    \, \dd \predictionNew  \dd \prediction  \\
    & \quad 
    + \int_0^1 \gamma(\prediction) \cdot \left((\truePredic_1(\prediction) - \prediction) \cdot \dist(\prediction)
    - (\truePredic_2(\prediction) - \prediction) \cdot \distNew(\prediction)
    \right)\, \dd \prediction \\
    & \quad 
    + 
    \int_0^1 \alpha_1(\prediction) \cdot\left(\dist(\prediction) - \int_0^1 \miscali(\prediction, \predictionNew) \, \dd \predictionNew\right)\, \dd \prediction
    + 
    \int_0^1 \alpha_2(\predictionNew) \cdot\left(\distNew(\predictionNew) - \int_0^1 \miscali(\prediction, \predictionNew) \, \dd \prediction\right)\, \dd \predictionNew \\
    & = 
    \int_{0}^1\int_0^1 \miscali(\prediction, \predictionNew) \cdot\left(
    - \beta(\prediction, \predictionNew) - \gamma(\prediction) \cdot \predictionNew 
    - \alpha_1(\prediction) 
    - \alpha_2(\predictionNew)\right)
    \,  \dd \predictionNew \dd \prediction\\
    & \quad 
    + 
    \int_0^1 \left(\gamma(\prediction) \cdot \truePredic_1(\prediction) 
    + \alpha_1(\prediction)\right) \cdot \dist(\prediction)
    \, \dd \prediction
    +
    \int_0^1 (\alpha_2(\prediction) - \gamma(\prediction)(\truePredic_2(\prediction) - \prediction)) \cdot \distNew(\prediction)
    \, \dd \prediction
\end{align*}
By the first-order condition, the optimal Lagrange multipliers $\beta^*, \alpha_1^*, \alpha_2^*, \gamma^*$ satisfy 
\begin{align*}
    - 
    \beta^*(\prediction, \predictionNew) - \gamma^*(\prediction) \cdot \predictionNew 
    - \alpha_1^*(\prediction) 
    - \alpha_2^*(\predictionNew) = 0
\end{align*}
for all $\prediction, \predictionNew\in[0, 1]$. 
In addition, under the optimal coupling $\miscali^*$, for any $\prediction, \predictionNew\in[0, 1]$, we have the following complementary slackness conditions: 
\begin{align*}
    \alpha_1^*(\prediction) 
    + \alpha_2^*(\predictionNew) 
    & = -\gamma^*(\prediction) \cdot \predictionNew, \quad \miscali^*(\prediction, \predictionNew) > 0 \\
    \alpha_1^*(\prediction) 
    + \alpha_2^*(\predictionNew) 
    & \le -\gamma^*(\prediction) \cdot \predictionNew, \quad \miscali^*(\prediction, \predictionNew) = 0~.
\end{align*}
With the above first-order conditions and complementary slackness, together with the miscalibrated predictions defined as in \Cref{ex:unbounded gap miscalibrated predictor measure}, the Lagrange dual then becomes
\begin{align*}
    \mathcal{L} 
    & = 
    \int_0^1 \left(\gamma^*(\prediction) \cdot \truePredic_1(\prediction) 
    + \alpha_1^*(\prediction)\right) \cdot \dist(\prediction)
    \, \dd \prediction
    +
    \int_0^1 (\alpha_2^*(\prediction) - \gamma^*(\prediction)(\truePredic_2(\prediction) - \prediction)) \cdot \distNew(\prediction)
    \, \dd \prediction  \\
    & = 
    \int_0^1 \left(\gamma^*(\prediction) \cdot \prediction
    + \alpha_1^*(\prediction)\right)
    \cdot \PDF(\prediction)
    \, \dd \prediction
    +
    \int_0^1 (\alpha_2^*(\prediction) - \gamma^*(\prediction)(\truePredic_{\predictornew}(\prediction) - \prediction)) \cdot \PDFNew(\prediction)
    \, \dd \prediction \tag{due to $\truePredic_{\predictor}(\prediction) = \prediction$} \\
    & = 
    \int_0^1 \left(\gamma^*(\prediction) \cdot \prediction
    + \alpha_1^*(\prediction)\right)
    \cdot \PDF(\prediction)
    \, \dd \prediction 
    +
    \int_0^{0.5} (\alpha_2^*(\prediction) + \gamma^*(\prediction)\prediction ) \cdot \PDFNew(\prediction)
    \, \dd \prediction
    + 
    \int_{0.5}^1 (\alpha_2^*(\prediction) - \gamma^*(\prediction)(1-\prediction)) \cdot \PDFNew(\prediction)
    \, \dd \prediction 
    \tag{due to $\truePredic_{\predictornew}(\prediction) = \indicator{\prediction\ge 0.5}$}\\
    & = 
    \int_0^1 \left(\gamma^*(\prediction) \cdot \prediction
    + \alpha_1^*(\prediction)\right)
    \cdot \PDF(\prediction)
    \, \dd \prediction
    +
    \int_0^1\alpha_2^*(\prediction) 
    \PDFNew(\prediction)
    \, \dd \prediction \\
    & \quad +
    \int_0^{0.5} \gamma^*(\prediction)\prediction \cdot \PDFNew(\prediction)
    \, \dd \prediction
    - 
    \int_{0.5}^1 \gamma^*(\prediction)(1-\prediction) \cdot \PDFNew(\prediction)
    \, \dd \prediction
\end{align*}
Similar to the analysis of \Cref{thm:strong-duality-perfectly-calibrated}, we know that 
\begin{align*}
    \mathcal{L} 
    & \ge \sup\limits_{U \text{ is 1-Lipschitz}}\int_0^1 U(\prediction) \cdot (\PDF(\prediction) - \PDFNew(\prediction)) \, \dd\prediction 
    + \int_0^{0.5} U'(\prediction)\prediction \cdot \PDFNew(\prediction)
    \, \dd \prediction
    - 
    \int_{0.5}^1 U'(\prediction)(1-\prediction) \cdot \PDFNew(\prediction)
    \, \dd \prediction\\
& \ge 
    \max_{t\in[0, 1]}\int_0^1 V_t(\prediction)(\prediction) \cdot (\PDF(\prediction) - \PDFNew(\prediction)) \, \dd\prediction 
    + \int_0^{0.5} V'_t(\prediction)(\prediction)\prediction \cdot \PDFNew(\prediction)
    \, \dd \prediction
    - 
    \int_{0.5}^1 V'_t(\prediction)(\prediction)(1-\prediction) \cdot \PDFNew(\prediction)
    \, \dd \prediction\\
    & = \max_{t\in[0, 1]} 
    4 \int_0^{\frac{1}{2}} t\cdot \PDFNew(t)\,\dd t~,
\end{align*}
where $V_t(\prediction) = |\prediction - t|$.
Then by weak duality, we further have 
\begin{align*}
    \REMDNew{\PDF, \PDFNew,\truePredic_\predictor,\truePredic_\predictornew} 
    \ge 4 \int_0^{\frac{1}{2}} t\cdot \PDFNew(t)\,\dd t~.
\end{align*}
We thus finish the proof.
\end{proof}
 
We are ready to prove \Cref{prop:bad approx of marg}.
\propBadApproxMarg*
\begin{proof}
Our first step is to lower bound the value $\REMDNew{\PDF, \PDFNew,\truePredic_\predictor,\truePredic_\predictornew}$. 
In particular, 
we show that by analyzing the dual of the optimization program in $\REMDNew{\PDF, \PDFNew,\truePredic_\predictor,\truePredic_\predictornew}$, we can lower bound its value as follows
\footnote{\label{footnote:welfare revenue interpretation}Notably, two terms $4\int_0^{\frac{1}{2}}t\cdot \PDFNew(t) \,\dd t$ and $t\cdot \left(1 - \frac{t-\varepsilon}{(1-2\varepsilon)t}\right) = t\cdot (1 - 2\int_0^t\PDF(\prediction)\,\dd\prediction)$ can be interpreted as the welfare and revenue in the context of mechanism design.} 
\begin{align*}
    \REMDNew{\PDF, \PDFNew,\truePredic_\predictor,\truePredic_\predictornew}
    \geq 4 \int_0^{\frac{1}{2}} t\cdot \PDFNew(t)\,\dd t
    =
    4\int_0^{\frac{1}{2}} t \cdot \frac{\varepsilon}{1-2\varepsilon}\cdot \frac{1}{t^2}\,\dd t
    =
    \frac{2\varepsilon}{1-2\varepsilon}\ln\frac{1}{2\varepsilon}~.
\end{align*}
The above lower bound analysis utilizes the previous analysis in \Cref{thm:strong-duality-perfectly-calibrated} for optimizing over the coupling set $\couplingSpace(\PDF, \PDFNew)$ and the specific structure of constructed predictions $\predictor, \predictornew$. 
Invoking \Cref{thm:strong-duality-miscalibrated} and \Cref{lem:idf miscali}, informativeness gap $\inforGap{\predictor,\predictornew}$ can be computed as\textsuperscript{\ref{footnote:welfare revenue interpretation}}
\begin{align*}
    \inforGap{\predictor,\predictornew} 
    &=
    \earthDistFlow{\PDF, \PDFNew,\truePredic_\predictor,\truePredic_\predictornew} 
    \\
    &= 
    2\cdot \max_{t\in [0, 1]} ~
    \left(\SCDF(t) + \int_0^t (\prediction - \truePredic_\predictor(\prediction))\cdot \PDF(\prediction)\, \dd \prediction \right)
    - 
    \left(\SCDFNew(t) + \int_0^t (\prediction - \truePredic_\predictornew(\prediction)) \cdot \PDFNew(\prediction)\, \dd \prediction\right)
    \\
    &= 
    2\cdot \max_{t\in [0, 0.5]} ~
    \left(\SCDF(t) + \int_0^t (\prediction - \truePredic_\predictor(\prediction))\cdot \PDF(\prediction)\, \dd \prediction \right)
    - 
    \left(\SCDFNew(t) + \int_0^t (\prediction - \truePredic_\predictornew(\prediction)) \cdot \PDFNew(\prediction)\, \dd \prediction\right)
    \\
    &= 
    \max_{t\in [0, 0.5]} ~
    t
    -
    \int_\varepsilon^t\int_\varepsilon^s \frac{\varepsilon}{1-2\varepsilon}\cdot \frac{1}{\prediction^2}\,\dd\prediction
    \,\dd s
    -
    \int_\varepsilon^t
    \prediction\cdot \frac{\varepsilon}{1-2\varepsilon}\cdot \frac{1}{\prediction^2}\,\dd\prediction
    \\
    &=
    \max_{t\in [0, 0.5]} ~
    t
    \cdot 
    \left(
    1 - \frac{t-\varepsilon}{(1-2\varepsilon)t}
    \right)
    =
    \Theta(\varepsilon)
\end{align*}
where the third equality holds since both predictors are symmetric around $t = 0.5$.

Combing the two pieces together, we completes the proof of \Cref{prop:bad approx of marg}.
\end{proof}

\section{Missing Experimentation Details}

\subsection{Experimental Process in Section~\ref{subsec:experiment setup}}
\label{apx:experimental process}

\xhdr{Event construction} For each dataset, we define a collection of event templates that can occur repeatedly and instantiate each template on every record, yielding a labeled binary outcome $\outcome\in\{0,1\}$ for each corresponding record. On the weather rain dataset, we design two events as \textbf{Event 1} and \textbf{Event 2} over the daily meteorological records.
\begin{itemize}
    \item \textbf{Event 1 (Single-day rain)}: Will it rain on day $\newday$?
    \item \textbf{Event 2 (Two-day rain)}: Will it rain on two consecutive days $(\newday,\newday+1)$?\footnote{We follow the construction of combination questions proposed by \cite{KBYJ-24}, which composes multiple base events into a single one to expand the dataset}
\end{itemize}
Next, on the Bitcoin price dataset, we define eight threshold-based events over the daily price changes: \textbf{Event 3-6} use absolute thresholds $\absoIncr_i$ while \textbf{Event 7-10} use relative thresholds $\relaIncr_i$.\footnote{Here, $\absoIncr_i\in\{\$100, \$500, \$1500, \$3000\}$ and $\relaIncr_i\in\{0.5\%, 1.5\%, 2.0\%, 4.0\%\}$. To select these threshold candidates, we compute the empirical occurrence frequency of each event under different absolute and relative thresholds. We then choose thresholds whose occurrence frequencies span a wide range from $49\%$ down to about $6\%$, thereby covering events of varying rarity and forecasting difficulty.}
\begin{itemize}
    \item \textbf{Event 3-6 (Absolute increase)}: Will the close price of day $\newday$ increase by $\absoIncr_i$?
    \item \textbf{Event 7-10 (Relative increase)}: Will the close price of day $\newday$ increase by $\relaIncr_i$?
\end{itemize}
For each event, we evaluate 4 different LLM models, so in total, we have $10\times 4$ different pairs of $(\text{LLM predictor}, \text{Event})$.
With the event templates fixed, we instantiate each template on every record $\record_i$ to obtain labels $\outcome_i\in\{0,1\}$, yielding a labeled set $\{(\record_i,\outcome_i)\}_{i\in[n]}$ for each event, which serves as the basis for subsequent forecasting and evaluation.

\xhdr{Prompt design} To ensure a fair comparison across the four LLMs, we adopt a unified prompting protocol to elicit a probabilistic prediction $\prediction\in[0,1]$ for each event instance. Inspired by previous work of \cite{JWMC-23}, our prompt input follows a structured format that separates the event description, the instance-specific record data and explicit restrictions on information access. See \Cref{apx:prompt} for the prompt samples.

Having specified the prompt input, we further impose a strict requirement for the output format to ensure standardized and parseable predictions across all LLMs. Specifically, each LLM must return only a valid JSON object without additional text and all probability entries must be real numbers in $[0,1]$.

\xhdr{Prediction collection} Each LLM is queried once per event instance using default settings, without tuning generation parameters. Responses are parsed as JSON and validated to ensure the extracted probability lies in $[0,1]$. Raw outputs are stored for traceability and a cleaned file containing only scalar probabilities is exported for metric computation. Whenever the validation fails. the identical prompt is retried until a valid prediction is obtained.

\xhdr{Evaluation} For each LLM and each event, we form the paired dataset $\{(\prediction_i,\outcome_i)\}_{i\in[n]}$ from the predictions and the ground-truth labels. Utilizing this dataset, we compute three metrics: Brier score, ECE and informativeness gap.

\subsection{Missing Details in Section~\ref{subsec:post-processing}}
\label{apx:post-processing}
In our experiment, we study binary-event calibration. Each raw prediction is a probability $\prediction_i\in(0,1)$ and the corresponding outcome is $\outcome_i\in\{0,1\}$. For the implementation of all post-processing methods, we fit the calibrator using a hold-out calibration set and then apply the learned mapping to the test set without further adjustment.

\xhdr{Bayesian binning \citep{NCH-15}} Bayesian binning is an expanded version of histogram binning that improves robustness by model averaging over a collection of candidate binning models rather than committing to a single discretization. For a candidate binning model $\binningModel$, we sort the raw predictions and partition them into $\binsNum$ equal-frequency bins. Within each bin $\bin$, the labels are modeled by a Binomial distribution with parameter $\binomPara_\bin$, and a Beta prior is placed on $\binomPara_\bin$. The Beta-Binomial conjuacy yields a closed form marginal likelihood $\prob{\calibData\givenn\binningModel}$, which can be used to score candidate binning models and compute their posterior weights. The final calibrated probability is obtained by averaging the model-specific outputs:
\begin{align*}
\prob{\outcome=1 \givenn \prediction}=\sum\nolimits_{i\in[n]}\frac{\prob{\binningModel_i}\prob{\calibData\givenn\binningModel_i}}{\sum\nolimits_{j\in[n]}\prob{\binningModel_j}\prob{\calibData\givenn\binningModel_j}}\prob{\outcome=1\givenn\prediction, \binningModel_i}
\end{align*}
In our setting, we directly apply Bayesian binning to the predicted probabilities on the calibration set. In the test set, each prediction is mapped to its bin under every candidate model, yielding a candidate calibrated probability; we then output the weighted average across candidate models using the learned posterior weights. All bin boundaries and model weights are fixed after calibration and reused for the test set.

\xhdr{Histogram binning \citep{ZE-01}} Histogram binning is a simple non-parametric post-processing method that maps miscalibrated scores to probabilities via discretization. The core idea is to divide the range of outputs into a finite number of bins and represent all outputs within the same bin by a single probability estimate. Concretely, we choose bin boundaries that partition the output range into $\binsNum$ disjoint intervals. For each bin, we assign a scalar parameter $\predictPara_\bin$ as it calibrated value. Given the boundaries, the optimal choice of $\predictPara_\bin$ is simply the average of the labels among samples whose outputs fall into that bin, i.e. the empirical event frequency within the bin. In the test set, calibration reduces to a lookup operation: determine which bin the input output belongs to and return the corresponding $\predictPara_\bin$.

\xhdr{Isotonic regression \citep{ZE-02}} Isotonic regression is a non-parametric recalibration method that learns a monotone mapping from miscalibrated scores to probabilities, motivated by the assumption that higher scores should correspond to higher event likelihoods. Specifically, it fits a stepwise-constant function that transforms uncalibrated outputs by minimizing a squared-loss objective under the monotonicity constraint. A standard approach is the pair-adjacent violators (PAV) algorithm, which efficiently computes the optimal monotone step function and returns a set of score intervals together with an estimated probability for each interval. In the test set, a prediction is calibrated by locating the interval it falls into and outputting the corresponding estimated probability. Conceptually, isotonic regression can be viewed as a flexible generalization of histogram binning, where the partition are determined in a data-driven manner subject to monotonicity.

\xhdr{Platt scaling \citep{Pla-99}}  Platt scaling is a parametric post-processing calibration method that converts a non-probabilistic output into a calibrated probability by fitting a logistic regression on a calibration set. In the standard form, it takes a real-valued score $\logit_i$ from a fixed classifier and outputs a sigmoid-transformed probability $\postProPred_i=\sigmoid{\plattCoef \logit_i+\plattInterc}$, where parameters $\plattCoef$ and $\plattInterc$ are learned in the calibration set by minimizing the negative log-likelihood (NLL). In our setting, the raw forecasts are already probabilities $\prediction_i\in(0,1)$. We therefore apply Platt scaling on the log-odds of the raw probabilities: $\logit_i=\log(\prediction_i/{1-\prediction_i})$, and fit $\postProPred_i=\sigmoid{\plattCoef \logit_i+\plattInterc}$, with $\plattCoef$ and $\plattInterc$ estimated by minimizing NLL on calibration set. This preserves a monotone, sigmoid-shaped mapping while operating on an unbounded input, consistent with the original motivation of Platt scaling as a sigmoid model in log-odds space. 

\xhdr{Scaling binning \citep{KLM-19}} Scaling binning is a hybrid recalibration method that combines a learned scaling function with a subsequent discretization step. It It proceeds in three stages. First, it fits a function $\scalingFunc\in\scalingFuncFamily$ on a calibration set to map the original outputs to intermediate scores by minimizing the squared loss $\sum(\outcome-g(\prediction))^2$. Second, it constructs $\binsNum$ bins using an equal-frequency rule based on the values $\{\scalingFunc(\prediction_i)\}$, so that each bin contains approximately the same number of calibration samples in the transformed space. Finally, it discretizes the learned mapping by replacing $\scalingFunc(\prediction)$ with a bin-wise constant: for any input $\prediction$, one computes $\scalingFunc(\prediction)$, then locates the bin that $\scalingFunc(\prediction)$ falls into, and outputs the average of $\scalingFunc(\prediction_i)$ over calibration samples whose transformed scores fall in the same bin. The resulting calibrator is therefore a piecewise-constant approximation of the fitted scaling function in $\scalingFuncFamily$, which retains the global trend learned by $\scalingFunc$ while reducing sensitivity to local noise through binning step. Then we use the binned mapping to recalibrate predictions in the test set.

\xhdr{Temperature scaling \citep{GPSW-17}} Temperature scaling is a parametric recalibration method that introduces a single positive parameter $\temperature>0$. Given a logit vector $\logit_i$, it outputs the recalibrated predictive distribution as $\softmax(\logit_i/\temperature)$. The temperature $\temperature$ controls the sharpness of the distribution: $\temperature=1$ recovers the original probabilities; $\temperature>1$ produces a softer distribution and approaches the uniform distribution as $\temperature\rightarrow\infty$; while $\temperature<1$ makes the distribution more peaked and collapses to a point as $\temperature\rightarrow 0$. The parameter $\temperature$ is learned from a calibration set by minimizing the NLL. Since rescaling logits by a positive constant does not change the predicted class, temperature scaling preserves the predicted label and therefore does not affect classification accuracy.

\subsection{Prompt Sample in Appendix~\ref{apx:experimental process}}
\label{apx:prompt}

\begin{promptbox}{Weather}
    \ttfamily
    You are a careful forecaster.
    \par\vspace{1\baselineskip}
    CRITICAL RULES:
    
    - USE ONLY the five provided numbers: temperature, humidity, wind\_speed, cloud\_cover, pressure.
    
    - Do NOT use internet, retrieval, climatology, location priors, memory, or any training-time facts.
    
    - Even if uncertain, provide your best probability estimate.
    \par\vspace{1\baselineskip}
    OUTPUT POLICY:
    
    - Return ONLY valid JSON. No extra text. No chain-of-thought.
    \par\vspace{1\baselineskip}
    JSON schema:
    
    \{"rain\_probability": <float in [0,1]>, "confidence\_note": "<=20 words>"\}
    \par\vspace{1\baselineskip}
    Below is a single weather snapshot. Using ONLY these five values, estimate the probability that it will rain (precipitation >= 0.1 mm) in the next \{hours\} hours.
    \par\vspace{1\baselineskip}
    Current readings (SI units):
    
    - temperature (°C): \{temperature\}
    
    - humidity (\%): \{humidity\}
    
    - wind\_speed (m/s): \{wind\_speed\}
    
    - cloud\_cover (\%): \{cloud\_cover\}
    
    - pressure (hPa): \{pressure\}
    \par\vspace{1\baselineskip}
    Constraints:
    
    - Base your judgment solely on the five numbers above; do NOT use any external knowledge or priors.
    
    - Return ONLY JSON with keys exactly:
    \{\{"rain\_probability": <float in [0,1]>, "confidence
    \_note": "<=20 words>"\}\}
\end{promptbox}

\begin{promptbox}{Bitcoin}
    \ttfamily
    The Bitcoin dataset records daily OHLCV data. We aim to predict whether the closing
    price on \{target\_date\} will exceed the previous day's close (\{prev\_date\}) by
    several relative and absolute thresholds.
    \par\vspace{1\baselineskip}
    Use ONLY information prior to \{target\_date\}. You may reference the last 30 days of OHLCV data (up to {prev\_date}) and any news summaries dated on or before \{prev\_date\}.
    \par\vspace{1\baselineskip}
    [Threshold events to evaluated]
    
    Relative (vs previous close):
    
    \{relative\_events\}
    \par\vspace{1\baselineskip}
    Absolute (USD):
    
    \{absolute\_events\}
    \par\vspace{1\baselineskip}
    Return a single JSON object:
    
    \{\{
    
      "relative": \{\{
      
        "rel\_0p5": probability close increases >=0.5\%,
        
        "rel\_1p5": probability close increases >=1.5\%,
        
        "rel\_2p0": probability close increases >=2.0\%,
        
        "rel\_4p0": probability close increases >=4.0\%
        
      \}\},
      
      "absolute": \{\{
      
        "abs\_100": probability close increases >= \$100,
        
        "abs\_500": probability close increases >= \$500,
        
        "abs\_1500": probability close increases >= \$1500,
        
        "abs\_3000": probability close increases >= \$3000
        
      \}\}
      
    \}\}
    
    All probabilities must be between 0 and 1. 
    \par\vspace{1\baselineskip}
    [BEGIN DATA]
    \par\vspace{1\baselineskip}
    [30-day OHLCV up to \{prev\_date\} (date, open, high, low, close, volume)]
    
    \{price\_table\}
    \par\vspace{1\baselineskip}
    [News summary (up to \{prev\_date\})]
    
    \{news\_summary\}
    \par\vspace{1\baselineskip}
    [END DATA]
\end{promptbox}
For weather forecasting tasks, the prompt prohibits the use of any external information, so the LLMs' predictive ability is evaluated based solely on the provided meteorological records. In contrast, for Bitcoin forecasting tasks, we explicitly allow the model to use the past 30 days of price data and a news summary which is restricted to information available up to the targeted day $\newday$, as cryptocurrency prices can be strongly influenced by market conditions and external events. 

\end{document}